\pdfoutput=1
\documentclass[twoside,11pt]{article}

\usepackage{jmlr2e}

\usepackage{amsfonts}
\usepackage{amsmath}
\usepackage{amssymb}
\usepackage{amsthm}
\usepackage{appendix}
\usepackage{bbold}
\usepackage{booktabs}  
\usepackage{color}
\usepackage{courier}
\usepackage{enumitem}
\usepackage{environ}
\usepackage{etoolbox}
\usepackage{float}
\usepackage{graphicx}
\usepackage{helvet}
\usepackage{mathtools}
\usepackage{mdwlist}  
\usepackage{microtype}
\usepackage{multirow}
\usepackage{nicefrac}
\usepackage{subfigure}
\usepackage{tabularx}
\usepackage{times}
\usepackage{url}
\usepackage{verbatim}
\usepackage{wrapfig}
\usepackage{xspace}

\usepackage{algorithm}
\usepackage{algorithmic}

\usepackage{hyperref}

\setlist{nosep}

\hypersetup{colorlinks = false, pdfborder = {0 0 0}}

\newif\ifdraft
\draftfalse
\newcommand{\TODO}[1]{
  \ifdraft
    \ifdefined\isaccepted \else
      \textcolor{red}{[TODO: #1]}\xspace
    \fi
  \fi
}
\newcommand{\NOTE}[1]{
  \ifdraft
    \ifdefined\isaccepted \else
      \textcolor{blue}{[NOTE: #1]}\xspace
    \fi
  \fi
}

\renewcommand{\paragraph}[1]{\textbf{#1}\;\xspace}

\renewcommand{\eqref}[1]{Equation~\ref{eq:#1}}
\newcommand{\secref}[1]{Section~\ref{sec:#1}}
\newcommand{\appref}[1]{Appendix~\ref{app:#1}}
\newcommand{\figref}[1]{Figure~\ref{fig:#1}}
\newcommand{\tabref}[1]{Table~\ref{tab:#1}}
\newcommand{\algref}[1]{Algorithm~\ref{alg:#1}}
\newcommand{\thmref}[1]{Theorem~\ref{thm:#1}}
\newcommand{\lemref}[1]{Lemma~\ref{lem:#1}}
\newcommand{\corref}[1]{Corollary~\ref{cor:#1}}
\newcommand{\defref}[1]{Definition~\ref{def:#1}}

\newcommand{\algrefs}[3][and]{Algorithms~\ref{alg:#2} #1~\ref{alg:#3}}

\makeatletter
\newtheorem*{rep@theorem}{\rep@title}
\newcommand{\newreptheorem}[2]{\newenvironment{rep#1}[1]{\def\rep@title{#2 \ref{##1}}\begin{rep@theorem}}{\end{rep@theorem}}}
\makeatother

\newtheorem{theorem}{Theorem}
\newtheorem{lemma}{Lemma}
\newtheorem{corollary}{Corollary}
\newtheorem{definition}{Definition}
\newreptheorem{theorem}{Theorem}
\newreptheorem{lemma}{Lemma}
\newreptheorem{corollary}{Corollary}

\newif\ifreptheorem
\reptheoremfalse
\newif\ifshowproofs
\showproofsfalse
\newcommand{\switchreptheorem}[2]{
  \ifreptheorem
    #1\xspace
  \else
    #2\xspace
  \fi
}

\newcommand{\replabel}[1]{
  \ifreptheorem
    \tag{\ref{#1}}
  \else
    \label{#1}
  \fi
}
\NewEnviron{thm}[1]{
  \makeatletter
  \ifcsname defined:thm:#1\endcsname
    \reptheoremtrue
    \begin{reptheorem}{thm:#1}\BODY\end{reptheorem}
    \reptheoremfalse
  \else
    \begin{theorem}\label{thm:#1}\BODY\end{theorem}
    \csgdef{defined:thm:#1}{}
  \fi
  \makeatother
}
\NewEnviron{lem}[1]{
  \makeatletter
  \ifcsname defined:lem:#1\endcsname
    \reptheoremtrue
    \begin{replemma}{lem:#1}\BODY\end{replemma}
    \reptheoremfalse
  \else
    \begin{lemma}\label{lem:#1}\BODY\end{lemma}
    \csgdef{defined:lem:#1}{}
  \fi
  \makeatother
}
\NewEnviron{cor}[1]{
  \makeatletter
  \ifcsname defined:cor:#1\endcsname
    \reptheoremtrue
    \begin{repcorollary}{cor:#1}\BODY\end{repcorollary}
    \reptheoremfalse
  \else
    \begin{corollary}\label{cor:#1}\BODY\end{corollary}
    \csgdef{defined:cor:#1}{}
  \fi
  \makeatother
}
\NewEnviron{prf}[1]{
  \ifshowproofs
    \begin{proof}\BODY\end{proof}\label{app:prf:#1}
  \else
    \begin{proof}In \appref{prf:#1}.\end{proof}
  \fi
}

\newenvironment{titled-paragraph}[1]{\textbf{#1:}}{}

\newcommand{\iid}{\emph{i.i.d.}\xspace}
\newcommand{\wrt}{w.r.t.\xspace}
\newcommand{\ie}{i.e.\xspace}
\newcommand{\eg}{e.g.\xspace}
\newcommand{\egcite}[1]{\citep[\eg][]{#1}}

\newcommand{\probability}{\ensuremath{\mathrm{Pr}}}
\newcommand{\expectation}{\ensuremath{\mathbb{E}}}

\newcommand{\defeq}{\ensuremath{\vcentcolon=}}

\newcommand{\N}{\ensuremath{\mathbb{N}}}

\newcommand{\R}{\ensuremath{\mathbb{R}}}

\newcommand{\suchthat}[1][]{\ensuremath{\underset{#1}{\mathrm{s.t.\;}}}}

\newcommand{\inner}[2]{\ensuremath{\left\langle {#1}, {#2} \right\rangle}\xspace}

\newcommand{\norm}[1]{\ensuremath{\left\lVert {#1} \right\rVert}\xspace}

\DeclareMathOperator*{\argmin}{argmin}

\newcommand{\indices}[1]{\ensuremath{\left[#1\right]}}

\newcommand{\codecomment}[1]{\`//~\textit{#1}}
\newcounter{lineno}
\newenvironment{pseudocode}{\setcounter{lineno}{0}\begin{tabbing}\textbf{mm}\=mm\=mm\=mm\=mm\=\kill}{\end{tabbing}}
\newcommand{\codename}{\>}

\newcommand{\codeline}{\>\stepcounter{lineno}\textbf{\arabic{lineno}}\'\>}

\DeclareMathOperator{\fix}{fix}

\newcommand{\elementwiseproduct}{\odot}

\newcommand{\elementwiseexp}{\operatorname{.exp}}

\newcommand{\grad}{\ensuremath{\nabla}\xspace}
\newcommand{\subgrad}{\ensuremath{\check{\grad}}\xspace}
\newcommand{\supgrad}{\ensuremath{\hat{\grad}}\xspace}
\newcommand{\stochasticgrad}{\ensuremath{\Delta}\xspace}
\newcommand{\stochasticsubgrad}{\ensuremath{\check{\stochasticgrad}}\xspace}
\newcommand{\stochasticsupgrad}{\ensuremath{\hat{\stochasticgrad}}\xspace}

\newcommand{\parameters}{\ensuremath{\theta}\xspace}
\newcommand{\Parameters}{\ensuremath{\Theta}\xspace}
\newcommand{\multipliers}{\ensuremath{\lambda}\xspace}
\newcommand{\Multipliers}{\ensuremath{\Lambda}\xspace}
\newcommand{\matrixmultipliers}{\ensuremath{M}\xspace}
\newcommand{\Matrixmultipliers}{\ensuremath{\mathcal{M}}\xspace}

\newcommand{\numconstraints}{\ensuremath{m}\xspace}
\newcommand{\matrixmultipliersize}{\ensuremath{\tilde{m}}\xspace}

\newcommand{\constraint}[1]{\ensuremath{g_{#1}}\xspace}
\newcommand{\objective}{\constraint{0}\xspace}
\newcommand{\proxyconstraint}[1]{\ensuremath{\tilde{g}_{#1}}\xspace}
\newcommand{\lagrangian}{\ensuremath{\mathcal{L}}\xspace}

\newcommand{\bregman}[1]{\ensuremath{D_{#1}}\xspace}
\newcommand{\bound}[1]{\ensuremath{B_{#1}}\xspace}

\newcommand{\Radius}{\ensuremath{R}\xspace}
\newcommand{\margin}{\ensuremath{\gamma}\xspace}
\newcommand{\deltamatrix}{\ensuremath{A}\xspace}

\newcommand{\approximation}{\ensuremath{\rho}\xspace}
\newcommand{\oracle}{\ensuremath{\mathcal{O}_{\approximation}}\xspace}

\jmlrheading{}{}{}{}{}{}

\ShortHeadings{Optimization with Non-Differentiable Constraints}{Cotter, Jiang, Wang, Narayan, Gupta, You, and Sridharan}

\begin{document}

\title{Optimization with Non-Differentiable Constraints with Applications to Fairness, Recall, Churn, and Other Goals}

\author{\name Andrew Cotter \email acotter@google.com\\
\name Heinrich Jiang \email heinrichj@google.com\\
\name Serena Wang \email serenawang@google.com \\
\name Taman Narayan \email tamann@google.com \\
\name Maya Gupta \email mayagupta@google.com \\
\addr Google AI\\
1600 Amphitheatre Pkwy\\
Mountain View, CA, USA \\
\name Seungil You \email seungil.you@gmail.com \\
\addr Kakao Mobility \\
Seongnam-si, Gyeonggi-do, South Korea \\
\name Karthik Sridharan \email sridharan@cs.cornell.edu \\
\addr Cornell University\\
Ithaca, NY, USA \\
}
\editor{}

\maketitle

\begin{abstract}
\TODO{Draft mode is on! Comment out the {\textbackslash}drafttrue statement to turn off TODO and NOTE comments before submitting}
We show that many machine learning goals, such as improved fairness metrics, can be expressed as constraints on the model's predictions, which we call rate constraints. We study the problem of training non-convex models subject to these rate constraints (or any non-convex and non-differentiable constraints). In the non-convex setting, the standard approach of Lagrange multipliers may fail. Furthermore, if the constraints are non-differentiable, then one cannot optimize the Lagrangian with gradient-based methods. To solve these issues, we introduce the proxy-Lagrangian formulation. This new formulation leads to an algorithm that produces a stochastic classifier by playing a two-player non-zero-sum game solving for what we call a semi-coarse correlated equilibrium, which in turn corresponds to an approximately optimal and feasible solution to the constrained optimization problem. We then give a procedure which shrinks the randomized solution down to one that is a mixture of at most $m+1$ deterministic solutions, given $m$ constraints. This culminates in algorithms that can solve non-convex constrained optimization problems with possibly non-differentiable and non-convex constraints with theoretical guarantees.  We provide extensive experimental results enforcing a wide range of policy goals including different fairness metrics, and other goals on accuracy, coverage, recall, and churn.
\end{abstract}


\section{Introduction}\label{intro}
We seek to provide better ways to control machine learning to meet societal, legal, and practical goals, and to take advantage of different kinds of side information and intuition that practitioners may have about their machine learning problem.  In this paper, we show that many real-world goals and side information can be expressed as constraints on the model's prediction rates on different datasets, which we refer to as \emph{rate constraints}, turning training into a constrained optimization problem. A simple example of a rate constraint is requiring that a binary classifier predict $80\%$ of examples to be the positive class, leading to the constrained optimization problem:
\begin{align}  
  \label{eq:example} \min_{\theta} & \frac{1}{N}  \sum_{j=1}^N  \ell(f(x_j;\theta), y_j)  + R\left(\theta\right) \\
  \notag & \textrm{s.t. }  \frac{1}{N} \sum_{j=1}^N I_{f(x_j; \theta)
  \geq 0} \geq 0.8,
\end{align}
where $\{(x_j, y_j)\}$ is a training set for $j=1, \ldots, N$, $\ell$ is the loss, $R$ is a regularizer, and $I$ is the usual indicator.

\subsection{The Broad Applicability of Rate Constraints}
One can express a surprisingly large set of real-world goals using rate constraints.  Here we preview some categories of goals, with more details in \secref{goals}. 

\textbf{Fairness:}
Many fairness goals can be expressed as rate constraints, including the popular fairness goal of \emph{statistical parity}. For example, one can constrain a classifier's
positive prediction rate for men to be within ten percent of its rate for women.  Other fairness goals that can be expressed as rate constraints are \emph{equal opportunity} and \emph{equal odds} \citep{Hardt:2016}, as well as fairness goals we encounter in real-world problems but have not previously seen in the machine learning literature, such as \emph{no worse off}. 

\textbf{Performance Measures:} Some metrics can be expressed as rate constraints, though approximations may be needed. For example, one can lower-bound the accuracy on particular slices of the data, or the recall.  Precision and win-loss ratio (WLR) compared to a baseline classifier can be expressed with rate constraints, however, there are some caveats about how satisfying constraints on these metrics will generalize to test samples (details below).  AUC can be approximated as a set of rate constraints, using the approximation proposed in \citet{Eban:2017}.

\textbf{Churn:}
Churn is the probability that each sample's classifier decision will change if the new classifier is deployed, compared to the decision of the classifier it is replacing. Reducing classifier \emph{churn} \citep{Cormier:2016,Goh:2016} is important in many practical machine learning systems to improve overall system stability and to make changes easier to measure and test.   Churn can be expressed with rate constraints, thus one can constrain the churn of a new classifier to some desired level. 

\textbf{Multiple Training Datasets:} Another application is when one has multiple training sets of varying quality. For example, one might have a small set of expertly rated training data, but also a large set of noisy training data one would also like to use. One can train using rate constraints to require the classifier to achieve a certain accuracy on the small expertly rated data, while training to minimize errors on large noisy data.

\textbf{Unlabeled Datasets:} Many of the rate constraints we discuss do not require labels, such as the fairness goal of statistical parity, coverage, or churn. This enables one to take advantage of large unlabeled datasets, which are cheaper to obtain than labeled data.

\subsection{Why Constrain? Why Not Penalize?}
For many of these problems, one could instead train the model with an additive penalty, rather than a constraint as in (\eqref{example}), to encourage the desired behavior. Given multiple goals, one could use a linear combination of multiple penalties.  However, that approach requires the practitioner to determine the right weight for each penalty; and those weights
may interact with each other in unexpected ways, particularly if the different goals are defined on different datasets. 

We have found that specifying a desired constraint for each goal is in practice cleaner and easier for practitioners. The key reason is that a constraint has absolute meaning, making it possible for a practitioner to specify their goal as a constraint without regard for the presence of other constraints. For example, the meaning of requiring a classifier to have $80\%$ recall in India does not depend on whether there is one or one hundred other such locale-specific constraints on the classifier.  We also found that using hard constraints leads to a more understandable machine learning model because it is clearer what the model was trained to do, and it is clearer to measure and verify whether the training sufficiently achieved the practitioner's intent for each individual goal.

\subsection{Training with Constraints: }
Training with rate constraints poses some hard challenges:
\begin{enumerate}
\item \textbf{Non-convex:} For nonlinear function classes, such as neural networks, the objective and constraint functions will be non-convex.
\item \textbf{Non-differentiable:} Rate constraints are 
  linear combinations of positive and negative classification rates, that is, they are made up of \emph{indicator} functions (0-1 losses) which are not even semi-differentiable (as in the example of \eqref{example}).
\item \textbf{Data-dependent:} The constraints are data-dependent, and for large datasets may be expensive to compute.
\end{enumerate}

While our motivating optimization problem is training with rate constraints, the analysis and algorithms we present will apply generally to constrained optimization problems of the form:
\begin{align} \label{eq:constrained-problem} 
\min_{\theta \in \Theta} & g_0(\theta)\\
\notag \suchthat & g_i(\theta) \leq 0 \textrm{ for } i = 1, \ldots, m,
\end{align}
where $g_0$ and $g_i$ may be non-convex, and $g_i$ may additionally be non-differentiable and expensive to evaluate.

\subsection{The Lagrangian Is Not Optimal for Non-Convex Problems}\label{sec:twoplayergame}



A popular approach to constrained optimization problems of the form (\eqref{constrained-problem}) is the method of Lagrange multipliers, using the Lagrangian defined as follows:
\begin{align} \label{eq:lagrangian}
\mathcal{L}(\theta, \lambda) \stackrel{\triangle}{=} g_0(\theta) + \sum_{i=1}^{m}
    \lambda_i g_i(\theta),
\end{align}
where $\lambda$ is an $(m+1)$-dimensional non-negative vector of Lagrange multipliers. The method of Lagrange multipliers can be viewed as a two-player zero-sum game where one player minimizes (\eqref{lagrangian}) with respect to the model parameters $\theta$, and the other player maximizes (\eqref{lagrangian}) with respect to the Lagrange multipliers $\lambda$. If the objective and constraints are all convex in $\theta$, this becomes a convex game, and under general conditions, has a pure Nash equilibrium~\citep{Neumann:1928} such that there exists a $\theta$ for the first player and a vector of Lagrange multipliers $\lambda$ for the second player such that neither player has the incentive to change their action given the other player's choice. In other words, there is \emph{no duality gap} between the primal (min-max) and dual (max-min) problems in the convex case and thus, solving for a pure Nash equilibrium (equivalently finding a saddle point in the Lagrangian) gives us an optimal and feasible solution to the original constrained optimization problem (\eqref{constrained-problem}).

While applying convex solvers to non-convex problems has proved reasonable for many methods such as stochastic gradient descent (SGD), constrained optimization with Lagrange multipliers may \emph{not even have a stationary point}, and thus may never converge, instead oscillating between different solutions. Moreover, because of the Lagrange duality gap which exists in non-convex problems, the saddle points found will not in general correspond to the desired solution. In game theory parlance, we show that for the non-convex setting of interest, a \emph{pure} Nash equilibrium does not exist in general (an example can be found in the Appendix). 

However, if we allow each player to choose a distribution of solutions (that is, a stochastic solution) instead of a single solution over their respective spaces $\Theta$ and $\Lambda$, and take the value of the Lagrangian to be the expected value over these distributions, then an equilibrium will exist. 

In this paper, we provide algorithms that approximately find such stochastic solution equilibria, and we show that these correspond to nearly-feasible and nearly-optimal {\it stochastic} solutions to the original constrained optimization problem (\eqref{constrained-problem}). The stochastic solution gives a randomized model: to classify an example $x$, we  sample a $\theta^*$ from the distribution over $\Theta$, and our guarantees will be in terms of expectations with respect to this randomized model and the samples. 

\subsection{The Lagrangian Is Impractical for Non-differentiable Constraints}\label{sec:twoplayergame}
Given non-differentiable constraints (such as rate constraints), a major shortcoming of the Lagrangian is that one cannot use gradient-based methods to optimize it. One approach is to use the Lagrangian but replace  non-differentiable constraints with relaxed versions which are differentiable~\egcite{Davenport:2010,Bottou:2011,Eban:2017}. However optimizing with the relaxed versions may lead to solutions which either over-constrain or fail to satisfy the original constraints. 

We introduce what we call the {\it proxy-Lagrangian}, 
where the key idea is to only relax the non-differentiable constraints when necessary. Solving the proxy-Lagrangian poses a considerable amount of technical challenges but leads to a number of interesting insights, and we provide algorithms which attain solutions with precise optimality and feasibility guarantees on the original non-differentiable constraints. 

Overall, our paper gives an end-to-end recipe to provably (with access to an optimization oracle) and efficiently solve non-convex optimization problems with possibly non-differentiable constraints, and whose final solution is a mixture of at most $m+1$ deterministic solutions. In practice, we use SGD in place of the oracle. To our knowledge, this is the first time such a procedure has been found to provably solve such non-convex problems with such irregular constraints and return a sparse solution. 

In addition, for those practical situations where a stochastic model is unappealing, we will also experimentally consider algorithms that do produce deterministic models, though they do not come with guarantees.

\subsection{Main Contributions and Organization}

The main contributions of this paper are: 
\begin{itemize}
\item We show many real-world goals can be addressed and prior knowledge captured by training with rate constraints.

\item We give a new \emph{proxy-Lagrangian} formulation for optimizing non-convex objectives with  non-differentiable constraints.

\item We provide an algorithm that outputs a $m+1$ \emph{sparse} stochastic classifier with theoretical guarantees, where $m$ is the number of constraints.

\item We show that our \emph{proxy-Lagrangian} formulation can also be used to produce a deterministic classifier that may be more practical for some applications, but without guarantees.

\item We experimentally demonstrate that the proposed optimization can be used to train classifiers with rate constraints, on both benchmark datasets and for real-world case studies.
\end{itemize}

Although our motivation and experimental focus is on the problem of training classifiers with rate constraints, our proposed proxy-Lagrangian formulation and theoretical results have broader application to other constrained optimization problems.

We next review related work. Then in \secref{goals} we detail many different goals that can be expressed with rate constraints. We then turn to the question of how to actually optimize with constraints, proposing new algorithms and theoretical results in \secref{ourOptSection}. \secref{experiments} presents a diverse set of experiments on benchmark and real datasets to illustrate the applicability of rate constraints and the proposed optimization.  We close with a discussion of conclusions in \secref{results} and open questions in \secref{conclusions}.

\section{Related Work} \label{sec:related}
We begin by reviewing our own prior work which this paper builds upon, then other work that considers specific rate constraints, and then related work in constrained optimization. 

\subsection{Our Prior Related Work}
In \citet{Goh:2016}, we showed that some fairness goals, low-churn re-training, and  recall lower bounds can all be optimized for by expressing these goals as constraints on the classifier's decisions on targeted datasets, and training the classifier to respect these constraints as part of the empirical risk minimization. We referred to these constraints as \emph{dataset constraints}, but we now use the more precise term \emph{rate constraints} to reflect our focus on the class of constraints that can be written in terms of the classifier's positive and negative decision rates. In this paper, we broaden the set of goals that can be expressed as rate constraints, and provide more insight into how to use rate constraints in practice. 

In \citet{Goh:2016}, the proposed constrained optimization algorithm was limited to linear classifiers. Our proposed method used a new cutting-plane algorithm to iteratively upper-bound the ramp loss with a convex loss, then solve the resulting inner-loop minimizations using an SVM solver. While amenable to theoretical analysis, this strategy is a bit slow and difficult to scale to more than a handful of constraints.  In contrast, in this paper we show we can effectively and efficiently train \emph{nonlinear} classifiers with rate constraints using the more popular and scalable approach of stochastic gradients.  Overall, compared to our prior work in \citet{Goh:2016}, we provide many new algorithmic, theoretical, and experimental contributions. 

Our strategy presented here for treating non-differentiable problems as a non-zero sum two-player game using a proxy Lagrangian formulation can be found on arXiv (\citet{Cotter:2018}), and will be separately submitted to the ALT 2019 conference. This journal paper differs by being more comprehensive and complete, with more discussion of how a broad set of goals can be expressed as rate constraints and much more comprehensive experiments.  

A key issue raised in this paper is how well satisfying rate constraints on a training set generalizes to new data drawn from the same distribution (\eg on a testing set, or at evaluation time). We build on the work presented in this paper to address this generalization question in a separate paper~\citep{Cotter:generalization}.

\subsection{Other Related Work in Training with Rate Constraints}
A number of special cases of rate constraints have been previously considered in the literature.

\citet{Mann:2007} and follow-on work \citep{Bellare:2009,Mann:2010} optimized probabilistic models with constraints, to incorporate side information about the prior priors on class labels (which we call \emph{coverage constraints}). They note their strategy could also be applied to any constraints that can be written as an expectation over a score on the random $(X,Y)$ samples. They incorporated this side information as an additive regularizer and penalized the relative entropy between the given priors and estimated multi-class logistic regression models \citep{Mann:2007}. They noted their approximation for the indicator could lead to degenerate solutions (see their paper for details and how they addressed the problem with additional regularization). 

Neyman-Pearson classification trains with a constraint on the false positive rate \citep{Scott:2005}, and a number of researchers have investigated this special case.
\citet{Eban:2017} optimized the model parameters and Lagrangian multiplier using stochastic gradients with a hinge approximation for the indicators in the empirical loss and constraints, and took the last training iterate as their solution. We compare to that optimization strategy in our experiments (listed as \emph{Hinge Last} in the result tables). Similarly,  \citet{Davenport:2010} optimized Neyman-Pearson support vector machines  with hinge loss relaxations,  using coordinate descent. \citet{Bottou:2011} relaxed the indicators to the ramp loss (both in the objective and constraints). 

\citet{Agarwal:2018} recently addressed training classifiers with fairness constraints. Like this work, their proposed algorithm is based on the two-player game perspective. Unlike this paper, they assume a \emph{zero-sum} game, which works because they also assume oracle solvers for the two players, side-stepping the practical issues of dealing with the non-differentible non-convex indicators in the constraints, which is the focus of our algorithmic and theoretical contributions. Similar to this work, they output a stochastic classifier, but do not provide the sparse $m+1$ solution that we present in this work. They also consider a deterministic solution, which they produce by searching over a grid of values for $\lambda$ for the best $\lambda$. They noted in their experimental section that the resulting deterministic solution was generally as good as their stochastic solutions on test data for those experiments they tried it on. As they note, a grid-search over $\lambda$ is less ideal as the number of constraints grows. 

Some other work in training fairer classifiers has used weaker constraints or relaxed them immediately to weaker constraints such as correlation, e.g. \citet{Zafar:2015}. Another set of work in fair classification only corrects a model post-training by optimizing additive group-specific bias parameters,  e.g. \citet{Hardt:2016} and \citet{Woodworth:2017}. \citet{donini2018empirical} studies optimization of fairness constraints for kernel methods by formulating the fairness constraints as orthogonality constraints.

\subsection{Other Types of Constraints On Machine Learned Models}
We focus on rate constraints in this paper, which differ from some other constrained machine learning tasks in that: \emph{(i)} because they depend on $f(x)$, they generally will depend on every parameter in the model, \emph{(ii)} thus are usually relatively expensive to compute, and \emph{(iii)} we do not generally expect to have a very large number of constraints (for practical problems generally 2-1000). 

These qualities are different than the popular constrained machine learning problem of \emph{shape constraints}, which may entail adding constraints to training to make the model monotonic (see e.g. \citep{Barlow:72,Groeneboom:2014,Gupta:2016,Canini:2016,luss2017bounded,You:2017,Lafferty:2018} 
and/or convex/concave~\egcite{Pya:2015,Chen:2016,DimReturns:2018}. Shape constraints generally entail adding many sparse cheap-to-evaluate constraints. For example, for isotonic regression on $N$ sample points, there are $O(N)$ constraints, and each is a function of only two model parameters~\citep{Barlow:72}. As another examples, for diminishing returns constraints on ensembles of lattices, there are $O(K)$ constraints where $K$ is the number of model parameters, but each constraint is a function of only three model parameters \citep{DimReturns:2018}. Problems with many cheap sparse constraints can be well-handled by stochastic sampling of the constraints, as in \citep{Cotter:2016}, unlike for rate constraints. 

Another type of constrained machine learning aims to constrain the model parameters to obey other properties, such as physical limits on the learned system (see e.g. \citet{Long:2018,Stewart:2017}). These constraints generally do not take the form of rate constraints, but such constrained machine learning models may also benefit form the presented algorithms and theory. 

\subsection{Related Work in Constrained Optimization as a Two Player Game}

Our constrained optimization algorithms and analyses build on the long history of treating constrained optimization as a two-player game: \citet{Arora:2012} surveys some such work, and there are several
more recent examples, e.g. \citep{Agarwal:2018,Kearns:2017,Narasimhan:2018}. 

In this paper we extend prior work in treating constrained optimization as a two-player game in three key ways. First, we introduce a shrinking procedure that significantly simplifies a  ``$T$-stochastic'' solution (\ie stochastic classifiers supported on all $T$ of the iterates) to a sparse ``$m$-stochastic'' solution (stochastic classifiers supported on only $m+1$ iterates, where $m$ is the number of constraints). Secondly, to handle non-differentiable constraints, we propose a new proxy-Lagrangian \emph{non-zero-sum} formulation, whereas prior work formulates the optimization as a zero-sum game. Third, we consider a broader set of problems than some of this prior work. For example, \citet{Agarwal:2018} propose a Lagrangian-based approach that is very similar we outline in \secref{lagrangian}, but only tackle fairness
constraints. Here we can handle any problem of the form of Equation 2. 

For example, our contributions also apply to robust optimization problems of the form:
\begin{equation*}
  \min_{\parameters \in \Parameters} \; \max_{i \in \indices{\numconstraints}}
  \constraint{i}\left(\parameters\right)
\end{equation*}

The recent work of \citet{Chen:2017} addresses non-convex robust optimization. Like both us and \citet{Agarwal:2018}, they (i) model such a problem as a
two-player game where one player chooses a mixture of objective functions, and
the other player minimizes the loss of the mixture, and (ii) they find a
\emph{distribution} over solutions rather than a pure equilibrium. These
similarities are unsurprising in light of the fact that robust optimization can
be reformulated as constrained optimization via the introduction of a slack
variable:
\begin{align}
  \label{eq:related-work:robust}
  \min_{\parameters \in \Parameters, \xi \in \Xi} \; & \xi \\
  \notag \suchthat & \forall i \in \indices{\numconstraints} . \xi \ge
  \constraint{i}\left(\parameters\right)
\end{align}
Correspondingly, one can transform a robust problem to a constrained one at the
cost of an extra bisection search~\egcite{Christiano:2011,Rakhlin:2013}. As
this relationship suggests, our main contributions can be adapted to the robust
optimization setting. In particular: (i) our proposed shrinking procedure can be
applied to \eqref{related-work:robust} to yield a distribution over only $m+1$
solutions, and (ii) one could perform robust optimization over
non-differentiable (even discontinuous) losses using ``proxy objectives,'' just
as we use proxy constraints.

\subsection{Other Strategies for Constrained Optimization}

There are other strategies for constrained optimization, each of which we argue is not well-suited to the problem of training classifiers with rate constraints. 

The complexity of the constraints makes it unattractive to use approaches that require
projections, such as projected SGD, or optimization of constrained subproblems, such as 
Frank-Wolfe~\citep{Hazan:2012,Jaggi:2013,Garber:2013}). 

Another strategy for constrained optimization is to penalize
violations~\egcite{Arora:2012,Rakhlin:2013,Mahdavi:2012,Cotter:2016,yang2017richer}, for
example by adding $\gamma \max_{i \in \indices{m}} \max\left\{0, 
\constraint{i}\left(\parameters\right)\right\}$ to the objective, where
$\gamma\in\R_+$ is a hyperparameter, and optimizing the resulting problem using
a first order method.  This strategy is a poor match to rate constraints for two reasons. First,  if the constraint functions are
non-(semi)differentiable, as in the indicators used in rate constraints. Second, each rate constraint is data-dependent, so evaluating $\constraint{i}$, or even determining whether it
is positive (as is necessary for such methods, due to the max with $0$), requires enumerating over the
entire constraint dataset, making this incompatible with the use with a 
computationally-cheap stochastic gradient optimizer. 


\section{How To Use Rate Constraints}\label{sec:goals}

In \citet{Goh:2016}, we showed that a number of useful machine learning goals can be expressed as constraints on the classifier's mean decisions on different datasets, including recall, coverage (the positive or negative classification proportion), churn, and some different fairness metrics. In this section, we first detail the mathematical formulation of rate constraints and the resulting constrained empirical risk minimization training. Then, we provide a list of metrics that can be expressed as rate constraints in \tabref{goals}, and detail in the following subsections how these rate constraints can be used to impose a broad set of policy goals and take advantage of side information. 

Given a model $f(x) \in \mathbb{R}$, a dataset $D$, and using $I$ to denote the usual indicator, define the classifier's positive and negative classification rates on $D$ as $p^{+}\left(D\right)$ and $p^{-}\left(D\right)$, where,
\begin{equation}\label{eq:rates}
  p^+(D) \stackrel{\triangle}{=} \frac{1}{| D |} \sum_{x \in D} I_{f(x; \theta)
  \geq 0}
  \;\;\; \: \textrm{and} \:  \;\;\;
  p^-(D) \stackrel{\triangle}{=} \frac{1}{| D |} \sum_{x \in D}  I_{f(x;
  \theta) < 0}.
\end{equation}

A constraint that can be expressed in terms of a non-negative linear combination of positive classification rates $p^+\left(D_k\right)$ and  negative classification rates  $p^-\left(D_k\right)$ over different datasets $\{D_k\}$ we call a \emph{rate constraint}:
\begin{equation}\label{eq:rateConstraint}
 \sum_{k=1}^K  \alpha_k p^+\left(D_k\right)
  + \beta_k p^- \left(D_k\right) \leq \kappa.
\end{equation}

\tabref{goals} shows how different choices of scalars $\alpha_k, \beta_k,  \kappa \in \mathbb{R}$ and datasets $\{D_k\}$ correspond to different standard performance metrics like accuracy and recall.
One can add $m$ rate constraints to the standard structural risk minimization to train a classifier with parameters $\theta$ on train dataset $D_0$, producing the constrained empirical risk minimization :
\begin{align}  \label{eq:loss}  
\min_{\theta} & \frac{1}{| D_0 |} \sum_{(x,y) \in D_0}  \ell(f(x;\theta), y)  + R\left(\theta\right) \\
   \textrm{s.t. }  &  \sum_{k=1}^{K_i}  \alpha_{ik} p^+\left(D_{ik}\right)
  + \beta_{ik} p^-\left(D_{ik}\right)  \leq \kappa_{i} \;\;
  \textrm{for } i=1, \ldots, m, \nonumber
\end{align}
where $\alpha_{ik}, \beta_{ik} \in \mathbb{R}$, $D_{ik}$ is the $k$th dataset for the $i$th constraint, $K_i$ is the number of datasets used to specify the $i$th constraint, and $\kappa_i \in \mathbb{R}$.

Throughout this work, we focus on inequality constraints, for lower-bounding or upper-bounding some rate. Equality constraints can be imposed by using both a lower-bound and upper-bound inequality constraint, though we suggest doing so with some margin between the lower and upper bound to make the optimization problem easier. 

For some applications it is notationally more convenient to drop the normalization: let $c^{+}\left(D\right)$ and $c^{-}\left(D\right)$ denote the count of the positive and negative classifications: 
\begin{equation}\label{eq:counts}
  c^+(D) \stackrel{\triangle}{=} \sum_{x \in D} I_{f(x; \theta)
  \geq 0}
  \;\;\; \: \textrm{and} \:  \;\;\;
  c^-(D) \stackrel{\triangle}{=}  \sum_{x \in D}  I_{f(x;  \theta) < 0}.
\end{equation}

In the rest of this section we show how different rate constraints can be used to impose various policy goals or capture side information. A key insight is that one can add constraints just on specific groups or subsets of the dataset by the choice of the constraint datasets $\{D_k\}$, which makes this approach particularly useful for fairness goals or other slice-specific metrics that are measured in terms of statistics on different datasets (see Table 2 and further details below).

\begin{table*}[t]
\caption{Examples of Metrics Expressed with Rates}
\label{tab:goals}
\begin{center}
\begin{tabular}{ll}
\toprule
$D$ & Set of examples \\
$D[*]$ & Subset of $D$ that satisfies expression *, e.g. $D[x \in \textrm{male}]$ is \\ & the subset of $D$ of male examples,  $D[y=1]$ is  the subset \\& of $D$ whose label is $1$, etc.  \\
$p^+(D) \in [0,1]$ & Proportion of $D$  classified positive \\
$p^-(D) \in [0,1]$ & Proportion of $D$  classified negative \\
$c^+(D) \in \mathbb{N}$ & Count of $D$  classified positive: $c^+(D) = |D| \; p^+(D) $  \\
$c^-(D) \in \mathbb{N}$ & Count of $D$  classified negative: $c^-(D) = |D| \; p^-(D) $  \\
Coverage & $p^+(D)$ \\
True Positives (TP) & $ c^+(D[y=1])$ \\
False Positives (FP) & $ c^+(D[y=-1])$ \\
True Negatives (TN) & $ c^-(D[y=-1])$ \\
False Negatives (FN) & $ c^-(D[y=1])$ \\
Recall  & $p^+(D[y=1])$  \\
Precision & $ c^+ (D[y=1])  / c^+(D)$  \\
Accuracy  & $(c^+(D[y = 1])  + c^-(D[y = -1]) / |D|$ \\
$h$ & A fixed classifier taken as given  \\
AUCROC & $\lim_{L, J \rightarrow \infty} \frac{1}{L} \sum_{\ell=1}^L \hspace{3mm} \max_{j \in [J] : p^{+}_{\alpha_j}(D[y=-1]) \le \frac{\ell}{L}} \hspace{3mm} p^{+}_{\alpha_j}(D[y=1])$ \\
Wins Compared to $h$ &  $c^+(D[h = -1, y = 1]) + c^-(D[h = 1, y = -1]$ \\ 
Losses Compared to $h$ &  $c^+(D[h = -1, y = -1]) + c^-(D[h = 1, y = 1]$ \\
Win Loss Ratio (WLR) & Wins Compared to $h$ / Losses Compared to $h$\\
Churn & $(c^+(D[h = -1]  + c^-(D[h = 1]))/ |D| $\\
Loss-only Churn & $(c^+(D[h = -1, y = -1]) + c^-(D[h = 1, y = 1]) / |D[h=y]| $ \\
\end{tabular}
\end{center}
\end{table*}

\subsection{Coverage Constraints}
Coverage is the proportion of classifications that are positive: $p^+(D)$ (a variant is \emph{negative coverage} $p^-(D)$.). For example, if a company wants to train a classifier to identify promising repeat customers, and knows it will use the classifier to positively predict $10\%$ of all customers to receive a printed catalog, then one could train the classifier with a $10\%$ coverage constraint.  

Coverage constraints can also be used to capture prior knowledge in the training. For example, if training a model to classify Americans' sex as male or female, one can regularize the classifier by incorporating the prior knowledge that $51\%$ of examples should be predicted to be women, by using a $51\%$ coverage constraint on the full dataset.

Using slice-specific coverage constraints can capture more side information.  For example, for the American male/female classifier, in addition to the overall coverage constraint of $51\%$, one could also add constraints capturing prior information about state sex distributions, such as constraining $51.5\%$ of
examples from New York to be classified as women, but constraining only $47.6\%$ of examples
from Alaska to be classified as women. 

A key advantage of coverage constraints is that they do not require labeled examples. This enables one to train on labeled training examples from a convenient distribution (such as actively-sampled examples), but add a coverage constraint to ensure the classifier is optimized to positively classify the desired proportion of positive classifications on a larger unlabeled dataset drawn \iid from the true underlying distribution. This usage of a coverage constraint forms a simple semi-supervised regularization of the classifier.

Another good use case for coverage constraints is to help make a controlled comparison of two model structures. For example, suppose one has a model type A (\eg a kernel SVM), and wonders if an alternative B (say, a DNN) is better, where A makes positive predictions on $40\%$ of test examples, while B appears to be more accurate, but only predicts the positive class for $35\%$ of test examples. If precision errors are worse than recall errors, we cannot be sure that B is better than A. We can try to quantify the misclassification costs of a false negative vs. a false positive, but that may be difficult to agree upon. It would be simpler to compare B to A at the same coverage as A, or at some other relevant coverage. Coverage-matching B to A can be done by tuning the decision threshold of B post-training, but including the coverage constraint in the training can help B learn to be a better classifier when tested at the desired coverage.

\subsection{Constraints on Accuracy, Recall, Precision, AUC}

As noted in \tabref{goals}, many standard performance metrics for classifiers, including accuracy, can be expressed in terms of rates, and thus constrained with rate constraints. 

Recall, defined as TP / (TP + FP),  can be constrained as $p^+(D[y=1]) > \kappa$ for the user's choice of $\kappa \in [0,1]$.  For example, one may wish to train a classifier that awards free lunches to poor students for overall accuracy, but constrain it to obtain at least $95\%$ recall on the most needy students.

Precision constraints can be used but are more subtle. Precision can be expressed in rates as $c^+ (D[y=1])  / c^+(D)$, and thus to get precision of at least $\kappa$, one can add a rate constraint:
\begin{equation}\label{eqn:precisionConstraint}
c^+ (D[y=1]) - \kappa  c^+(D)  \geq 0.
\end{equation}
If (\ref{eqn:precisionConstraint}) holds, then mathematically the precision is lower-bounded by $\kappa$ on the dataset $D$.  However, since the expectation of a ratio does not equal the ratio of the expected numerator and denominator, analyzing how well the empirical constraint holding generalizes to new \iid samples is not straightforward, and violating the constraint (\ref{eqn:precisionConstraint}) by some $\epsilon > 0$ does not translate directly into precision. 

The ROC AUC  (Area under the ROC curve)  can be approximated using a rate constraint, as in \citet{Eban:2017}. The ROC curve is obtained by plotting the true positive rate (TPR) vs. the false positive rate (FPR).  First, slice up the FPR-axis into $L$ slices (to approximate the required Riemann integral). Then for the $\ell$th slice, consider $J$ different decision thresholds and choose the threshold that maximizes TPR and satisfies the $\ell$th slice FPR bound $\ell/L$, and then the averaged maximum precision across the $L$ FPR slices is bounded: 

\begin{equation}\label{eqn:AUCROCConstraint}
\frac{1}{L} \sum_{\ell=1}^L \hspace{3mm} \max_{j \in [J] : p^{+}_{\alpha_j}(D[y=-1]) \le \frac{\ell}{L}} \hspace{3mm} p^{+}_{\alpha_j}(D[y=1]) \geq \kappa.
\end{equation}
where $p_{\alpha}^{+}(D)  \stackrel{\triangle}{=} \frac{1}{| D |} \sum_{x \in D} I_{f(x; \theta)
  \geq \alpha}$, $c_{\alpha}^{+}(D)  \stackrel{\triangle}{=} \sum_{x \in D} I_{f(x; \theta)
  \geq \alpha}$, and $\alpha_j := \frac{2j-1}{2J}$ for $j \in [J]$.
In particular, $p_{0}^{+} \equiv p^{+}$. Taking $L \rightarrow \infty, J \rightarrow \infty$ will have the expression on the LHS of (\ref{eqn:AUCROCConstraint}) converge to the exact ROC AUC.



\subsection{Churn and Win Loss Ratio Constraints} \label{sec:churn}
In practice, a classifier is often being trained to replace an existing model. In such cases, the new classifier $h$ may be evaluated compared to the existing one.

One common metric to compare two classifiers is the win-loss ratio (WLR), which for a given test set is the number of times the new classifier is right and the old classifier is wrong, divided by the number of times the new classifier is wrong and the old classifier is right.

WLR can be expressed in rates as given in Table \ref{tab:goals}, where we use $D[h=-1]$ to denote the subset of $D$ that is labeled negatively by the classifier $h$, and $D[h=-1,y=1]$ to denote the subset of $D$ of whose training label $y$ is $1$, so that $c^+(D[h=-1,y=1])$ is the number of wins of the new classifier over $h$, and so on. Re-arranging terms, one can constrain for WLR using a rate constraint:
\begin{align} 
&c^+(D[h = -1, y = 1]) + c^-(D[h = 1, y = -1])  \nonumber \\ 
&- \kappa (c^+(D[h = -1, y = -1]) + c^-(D[h = 1, y = 1])) \geq 0, \label{eqn:wlrConstraint}
\end{align}
where $\kappa \in \mathbb{R}^+$ is the lower-bound on the WLR. However, like precision, if the constraint holds on dataset $D$ then the WLR does meet the lowerbound on $D$, but generalization to \iid samples is not straightforward, since the expectation of a ratio does not equal the
ratio of expectations.

WLR constraints on different slices of the data can ensure that a new classifier's gains are not coming at the expense of an important subset of examples. (See also our discussion of \emph{no worse off} and \emph{no lost benefits} for related fairness constraints).

In practice, it is common to test a new classifier by drawing a fresh test set of examples whose classification decisions have changed between the new classifier and $h$, and then only label those changed examples. We refer to this as a  \emph{fresh test}. Fresh tests reduce overfitting to a fixed re-used test set, and reduce costs by only labeling examples whose decisions have changed. Under such test set-ups, even if the two classifiers have the same accuracies, a higher WLR makes it easier to statistically significantly confirm that the new classifier is better than $h$ \citep{Cormier:2016}.

The proportion of a dataset $D$ whose decisions are flipped is called \emph{churn} \citep{Cormier:2016,Goh:2016}. Note that when using a fresh test, the labeling costs scales linearly with the churn (and the size of the test set $D$).  High churn also causes more instability for follow-on systems, and can confuse users. \citet{Goh:2016} proposed
explicitly constraining the churn, which can be directly expressed as the rate constraint: 
\begin{equation}\label{eqn:churnConstraint}
c^+(D[h = -1])  + c^-(D[h = 1])   \leq \kappa |D|,
\end{equation}
where $\kappa \in [0,1]$ is the proportion of $D$ whose classification decision is allowed to flip.  

We note that constraining churn on different slices of the
data with tighter and looser constraints can be useful. For example, if the classifier is to be used worldwide, but evaluating classifier changes is more expensive in
Norway than in Vietnam, or if there is known to be less headroom to improve on
examples from Norway, then it can be beneficial to constrain the churn
tightly for examples from Norway, but loosely for examples from Vietnam.

Of course, constraining churn too tightly limits the potential accuracy gains. Thus we also propose considering loss-only churn constraints, which only penalizes changing decisions if they used to be correct:
\begin{equation}\label{eqn:lossonlychurnConstraint}
c^+(D[h = -1, y = -1])  + c^-(D[h = 1, y = 1])   \leq \kappa |D|,
\end{equation}
where $\kappa \in [0,1]$ is the proportion of $D$ whose classification decision is allowed to flip. 

One disadvantage of constraining loss-only churn is it requires labeled examples, whereas churn constraints can be conveniently employed on a dataset of unlabeled examples. 

\subsection{Group-Specific Goals and Fairness Goals}
An important use case for rate constraints is enforcing metrics for different groups or categories of examples. For example, ensuring that a classifier that identifies family-friendly videos works roughly equally well for different types of adult content. Rate constraints can be used to enforce a broad set of such group-specific goals, as detailed in \tabref{fairness}, where $k$ indexes the $K$ different groups of interest.

A special case of group-specific goals are those that can be interpreted as designed to improve some fairness metric, and in these cases the groups are usually defined as different categories of people, such as races or age brackets. \tabref{fairness} shows that many of the fairness goals already studied in the machine learning literature can be expressed with rate constraints.  However, fairness is a complex moral and policy problem, and depending on the context and application,  different concrete formulations may be appropriate to improve fairness measures, and some may not be group-based at all. 
Many of these fairness goals are designed for applications where positive classification endows a benefit, such as being awarded a loan, a job, or a free lunch. For example, the goal of \emph{statistical parity} reflects that a bank might be legally required to give loans at equal rates to different groups, that is, the classifier is required to provide equal positive rates of classification across groups (see e.g. \citep{Zafar:2015,fish:2016,Hardt:2016,Goh:2016}). Statistical parity is also known as \emph{demographic parity} \citep{Hardt:2016}, and \emph{equal coverage} \citep{Goh:2016}. Notice that a statistical parity constraint ignores the labels of the training data. We introduce the related goal of \emph{minimum coverage}, which enforces a pre-set minimal benefit rate for each group.

Similarly, one of the fairness goals we see in practice but have not noticed in the literature is \emph{no lost benefits}: which requires a model to classify examples positively from each group at least as often as the classifier $h$ that it is replacing. \emph{No lost benefits} is a type of churn goal (see Sec. \ref{sec:churn}) that is measured for the whole group (rather than for individual decisions).

The other fairness goals in the Table \ref{tab:fairness} depend on the training labels. For example, we add the goal \emph{accurate coverage}, which requires the classifier give free lunches to each group to match that group's positive training label rate.  This goal ignores whether the individual predictions are accurate, but tries to ensure that each group overall receives a rate of benefits that it is labeled as deserving.  

Returning to the scenario where the new classifier will replace a current classifier $h$, we propose the \emph{not worse off} fairness goal. For example, suppose we invent a new driving test that is more accurate than a current written driving test at diagnosing whether illiterate people are safe drivers, then \emph{not worse off} requires that the new driving test not reduce accuracy compared to the old test for other groups, e.g. senior citizens and teenagers.  \emph{Not worse off} is a label-dependent group-specific churn goal.

\emph{Equal opportunity} and \emph{equal odds}  \citep{Hardt:2016} also rely on the training labels, but disregard any previous classifier. For example, \emph{equal opportunity} requires that \emph{if} a classifier awards free lunches (positive classification) to half of the east-side children who are labeled as deserving free lunches, then it should also award free lunches to half of the west-side children who are labeled as deserving free lunches. Notice that \emph{equal opportunity} imposes no conditions whatsoever on the negatively-labeled examples (in this case, those shudents who are not labeled as deserving of free lunches). In contrast, the fairness goal of \emph{equal odds} requires both the true positive rate and the false positive rate to be the same for all groups. We add to this category the related goal of \emph{equal accuracy}, which aims to make the classifier equally accurate for the different groups.  

Another fairness goal we find useful in practice but have not previously seen in the literature is \emph{minimum accuracy}, which requires that every group experience some pre-set level of accuracy. \emph{Minimum accuracy} ensures that no group is left behind, but respects that for some problems some groups may be much easier to classify than other groups.  For such problems, cross-group constraints can lead to degenerate solutions, as the only way to make all groups have equal metrics may be to produce a degenerate classifier. 

Fairness goals that depend on the training labels are most compelling when the training examples and labels are believed to have been fairly sampled and labeled. These goals are less compelling when the training data is not entirely trusted, or thought to be misaligned with the policy goals, a situation referred to as \emph{negative legacy} \citep{kamishima:2012}.

\begin{table*}[t]
\caption{Group-Specific and Fairness Goals Expressed As Rate Constraints for Groups $k = 1, \ldots, K$}
\label{tab:fairness}
\begin{center}
\begin{tabular}{ll}
\toprule
Statistical Parity & $ p^+(D_k) = p^+(D) \: \forall k$ \\
Minimum Coverage & $ p^+(D_k) \geq \kappa \: \forall k$ and user-specified $\kappa \in [0,1]$ \\
No Lost Benefits & $p^+(D_k) \geq |(D_k[h=1])| / |D_k| \: \forall k$  \\
Accurate Coverage & $ p^+(D_k) = |D_k[y=1]| / |D_k| \: \forall k$ \\
Equal Opportunity  & $ p^+(D_k[y=1]) = p^+(D[y=1]) \: \forall k$ \\ 
Equal Odds & $ p^+(D_k[y=1]) = p^+(D[y=1]) \: \forall k$ \\
& and $p^+(D_k[y=-1]) = p^+(D[y=-1]) \: \forall k$  \\ 
Equal Accuracy  & $ (c^+(D_k[y=1]) + c^-(D_k[y=-1])) / |D_k| $ \\
& $= (c^+(D[y=1]) + c^-(D[y=-1])) / |D| \: \forall k$  \\
Minimum Accuracy & $(c^+(D_k[y=1]) + c^-(D_k[y=-1])) / |D_k| \geq \kappa \: \forall k$ \\
Accurate Coverage & $ p^+(D_k) = |D_k[y=1]| / |D_k| \: \forall k$ \\
Not Worse Off & $(c^+(D_k[y=1]) + c^-(D_k[y=-1])) / |D_k| \geq |(D_k[y=h])| / |D_k| \: \forall k$  \\

\end{tabular}
\end{center}
\end{table*}

\subsection{Egregious Examples and Steering Examples}\label{sec:steering}
Another use of rate constraints is to constrain the performance on auxiliary labeled datasets to control the classifier. 

For example, \citet{Goh:2016} proposed constraining the classifier for high accuracy on a small set of particularly \emph{egregious examples} that should definitely not be mislabeled.  Egregious examples act as an \emph{integrated unit test}: as the classifier trains it actively is testing to see if it satisfies the constraint on the egregious examples and is able to correct the training if not. 

Another practical example of using an auxiliary labeled dataset we term \emph{steering examples}, which we define as a set of labeled examples that are more accurately labeled than the training set. For example, one may have access to a large but noisy training set of clicks on news articles. However, a click on a news article might be because it was
relevant news, or because it had a catchy headline. We can try to steer the classifier to focus its fitting on the relevant news articles by providing a smaller but expertly-labeled curated set of examples that mark catchy headlines as negative, and then constrain the classifier to
achieve some reasonable accuracy on the steering examples (e.g. $70\%$). The classifier will be forced by the constraint to disregard the incorrectly-labeled training examples. A second example is a classifier whose goal is to determine if an online store should advertise to a given customer. Suppose there is a large dataset of
training examples with the positive label, ``customer clicked advertisement and
\emph{visited} website'', but a relatively small set of examples where the
positive label is, ``customer clicked advertisement and made a
\emph{purchase}.'' It may be better to train on the large set of
``visited'' examples due to its much larger size and coverage, but also to constrain at least some specified accuracy on
the smaller ``purchase'' examples in order to steer the classifier towards prioritizing clicks that lead to purchases.

\subsection{Decision Rule Priors}
Machine learning practitioners often have prior knowledge about a classification problem that they can communicate as a decision rule on a tiny set of features. For example, ``Don't recommend a book to a user if it is in a
language they haven't purchased before.''  We propose incorporating such decision rule priors into the structural risk minimization problem by
constraining a large set of unlabeled samples to satisfy the decision rule.

Such decision rule priors can act as regularizers against noisy and
poorly-sampled training examples, and can produce a classifier that is more
interpretable because it is known to obey the given decision rules
(like all rate constraints, this depends on whether one constrains with slack or not, and exactly how well the satisfied constraint
generalizes to a different draw of \iid samples or non-\iid samples depends on
the dataset used in the constraint, the classifier's function class, and how
hard the constraint is to satisfy).

This proposal is similar to Bayesian Rule Lists (BRL) \citep{Letham:2015} in
that a decision rule (or set of decision rules) is given a priori to training
the model. However, BRL training takes as input a large set of decision rules
and outputs a posterior over the rules, rather than incorporating a decision
rule into a structural risk minimization problem.

\subsection{How To Best Specify Rate Constraints} \label{sec:bestpractices}
For any rate constraint, one wants to allow some slack in order to find a feasible solution. For example, statistical parity could be written as a constraint with \emph{additive slack} like this:
\begin{equation*}
p^+(D)  - p^+(D_k)  \leq \kappa,
\end{equation*}
or with \emph{multiplicative slack} like this:
\begin{equation*}
p^+(D)  -  \kappa p^+(D_k) \geq 0.
\end{equation*}

Our experience is that additive slack tends to be more likely to produce reasonable solutions than multiplicative slack for many constraints. The danger to watch out for is whether the constraint is specified in a way that encourages the training to satisfy the constraint in a suboptimal way. For example, suppose one constrains the false positive rate of each groups to be no worse than $125\%$ of the overall false positive rate (multiplicative slack), then the training is incentized to increase the overall false positive rate because that loosens the constraint further (due to the slack being multiplicative). 

Constraints can also be expressed \emph{pairwise} between groups instead of against the global rate
\begin{equation*}
p^+(D_j)  - p^+(D_k)  \leq \kappa,
\end{equation*}
for all $j, k$ pairs.  Our experience is that constraints that involve larger datasets are generally preferable, as the smaller the dataset used in a constraint the greater the risk of degenerate solutions or overfitting. 

Equality constraints can be expressed by using both a lower-bound and upper-bound inequality constraint.  In practice, we suggest allowing some slack of wiggle-room between the lower and upper bounds in order to make the optimization more stable, as a larger feasible set will make the  stochastic gradient optimization more stable.


\section{Optimizing with Constraints} \label{sec:ourOptSection}

For nonlinear function classes, training a classifier with rate constraints as per (\eqref{loss}) is a non-convex optimization over a non-convex constraint set.  In this section we provide new theoretical insights and algorithms to optimize general non-convex problems with non-convex constraints, then demonstrate our algorithmic proposals work well in practice with multiple real-world constraints in \secref{experiments}. We first outline our two main contributions for this section below.

\textbf{A Minimal Stochastic Solution: }
Algorithms that solve non-convex constrained optimization problems based on regret minimization, which includes our approach as well as previous  work~\egcite{Chen:2017,Agarwal:2018} will output a distribution over $\theta$s which has discrete support over $T$ different $\theta$ (produced at different iterations of the algorithm), requiring us to store and sample from $T$ different models. In practice, large $T$ may be problematic to store and analyze. Surprisingly, we prove that there always exists an equilibrium that has \emph{sparse support} on at most $m+1$ choices of model parameters,  where $m$ is the number of constraints. We use this result to provide a new practical algorithm to shrink the approximated equilibrium down to a nearly-optimal and nearly-feasible solution supported on at most $m+1$ models, which is guaranteed to be at least as good as the original stochastic classifier supported on $T$ models. 

\textbf{Handling Non-Differentiable Constraints: }
A key issue for \eqref{loss} is the non-differentiability of the constraints due to the indicators in the rate constraints. To handle this, in \secref{proxy-lagrangian-equilibrium}, we introduce a new formulation we call the \emph{proxy-Lagrangian} that changes the standard two-player \emph{zero-sum} game to a two-player \emph{non-zero-sum} game, which presents new challenges to analysis. In fact, solving for such a Nash equilibrium is PPAD-complete in the non-zero-sum setting \cite{Chen:2006}. We prove that a particular game theory solution concept, which we call {\it semi-coarse correlated equilibrium}, results in a stochastic classifier that is feasible and optimal. This is surprising because the semi-coarse correlated equilibrium is a weaker notion of equilibrium than Nash equilibrium. We go on to provide a novel algorithm that converges to
such an equilibrium. To our knowledge, we give the first reduction to this particular solution concept and the first practical use for it, which may be of independent interest. Interestingly, the
$\parameters$-player needs to only minimize the usual external regret, but the
$\multipliers$-player must minimize the \emph{swap regret}~\citep{Blum:2007}, a stronger notion of regret. While the resulting distribution is supported on (a possibly large number of) $\left(\parameters,\multipliers\right)$ pairs, applying the same
``shrinking'' procedure as before yields a distribution over only
$\numconstraints+1$ $\parameters$s that is at least as good as the original.

In Section~\ref{sec:lagrangian}, we handle the optimization of the zero-sum Lagrangian game with an oracle-based algorithm and introduce  our proposed ``shrinking'' procedure. Then, in Section~\ref{sec:introduction:non-zero-sum} we introduce the concept of proxy constraints, describe how it is useful to handle non-differentiable constraints, and formulate the non-zero-sum modification of the Lagrangian, which we call the proxy-Lagrangian. Section~\ref{sec:proxy-lagrangian-equilibrium} describes the equilibrium required out of this non-zero-sum game so that it will correspond to an approximately feasible and optimal solution to the constrained optimization problem.  Section~\ref{sec:proxy-lagrangian-algorithm} gives an oracle-based procedure for solving for such an equilibrium. Section~\ref{sec:practical-proxy-lagrangian-algorithm} gives a more practical stochastic gradient-based optimizer along with improved guarantees in the convex setting. Finally, Section~\ref{sec:proxy-lagrangian:shrinking} shows that the "shrinking" procedure holds for the non-zero-sum solution as well.

\subsection{Lagrangian Optimization in the Non-convex Setting}\label{sec:lagrangian}

We start by assuming an approximate Bayesian optimization oracle (defined in \secref{oracle}), which enables us to use the Lagrangian formulation and not relax the non-convex and/or non-differentiable constraints. This setting is a slight generalization of that presented in \citet{Agarwal:2018}.  \algref{oracle-lagrangian} solves for a stochastic solution to the non-convex constrained optimization problem. It proceeds by playing the following for $T$ rounds: the model parameter player plays best-response (that is, the $\theta$ which  minimizes the Lagrangian given the last choice of Lagrange multipliers), and the Lagrange multiplier player plays a regret minimizing strategy (here we use projected SGD). 
 
Our first contribution of this section (in \secref{mixedNash}) is showing that the resulting stochastic classifier is provably approximately feasible and optimal in expectation. This extends the fair classification work of \citet{Agarwal:2018} to our slightly more general setting. Our second contribution comes in \secref{lagrangian:shrinking}:  we will show how the support of the stochastic solution can be efficiently ``shrunk'' to one that is \emph{at least as good}, but is supported on only $m+1$ solutions and is shown to also have a considerable gain empirically. 

%
\begin{algorithm*}[t]

\begin{pseudocode}
\codename $\mbox{OracleLagrangian}\left( \Radius \in \R_+, \lagrangian : \Parameters \times \Multipliers \rightarrow \R, \oracle : \left(\Parameters \rightarrow \R\right) \rightarrow \Parameters, T \in \N, \eta_{\multipliers} \in \R_+ \right)$: \\
\codeline Initialize $\multipliers^{(1)} = 0$ \\
\codeline For $t \in \indices{T}$: \\
\codeline \> Let $\parameters^{(t)} = \oracle\left( \lagrangian\left(\cdot,\multipliers^{(t)}\right) \right)$ \codecomment{Oracle optimization} \\
\codeline \> Let $\stochasticgrad^{(t)}_{\multipliers}$ be a gradient of $\lagrangian\left(\parameters^{(t)},\multipliers^{(t)}\right)$ \wrt $\multipliers$ \\
\codeline \> Update $\multipliers^{(t+1)} = \Pi_{\Multipliers}\left( \multipliers^{(t)} + \eta_{\multipliers} \stochasticgrad^{(t)}_{\multipliers} \right)$ \codecomment{Projected gradient update} \\
\codeline Return $\parameters^{(1)},\dots,\parameters^{(T)}$ and $\multipliers^{(1)},\dots,\multipliers^{(T)}$
\end{pseudocode}

\caption{
  Optimizes the Lagrangian formulation (\eqref{lagrangian}) in the non-convex
  setting via the use of an approximate Bayesian optimization oracle $\oracle$
  (\defref{oracle}) for the $\parameters$-player.
  The parameter $\Radius$ is the radius of the Lagrange multiplier space
  $\Multipliers \defeq \left\{ \multipliers \in \R_+^{\numconstraints} :
  \norm{\multipliers}_1 \le \Radius \right\}$, and the function
  $\Pi_{\Multipliers}$ projects its argument onto $\Multipliers$ \wrt the
  Euclidean norm.
}

\label{alg:oracle-lagrangian}

\end{algorithm*}

\subsubsection{Oracle for Unconstrained Non-convex Minimization (Additive Approximation)} \label{sec:oracle}
\algref{oracle-lagrangian}, like \citet{Chen:2017}'s
algorithm for robust optimization, requires an \emph{oracle} for performing
approximate non-convex minimization:
\begin{definition}
  \label{def:oracle}
  A $\approximation$-approximate Bayesian optimization oracle is a function
  $\oracle : \left(\Parameters \rightarrow \R\right) \rightarrow \Parameters$
  for which:
  \begin{equation*}
    f\left( \oracle\left(f\right) \right) \le \inf_{\parameters^* \in
    \Parameters} f\left(\parameters^*\right) + \approximation
  \end{equation*}
  for any $f : \Parameters \rightarrow \R$ that can be written as a nonnegative
    linear combination of the objective and constraint functions
    $\objective,\constraint{1},\dots,\constraint{\numconstraints}$.
\end{definition}
with the $\parameters$-player using this oracle, and the $\multipliers$-player
using projected gradient ascent. We note that this is a standard assumption in order to obtain theoretical guarantees. (e.g. see \citet{Chen:2017}, which uses a multiplicative instead of additive approximation).

\subsubsection{Approximate Mixed Nash Equilibrium} \label{sec:mixedNash}
We characterize the relationship between an
approximate Nash equilibrium of the Lagrangian game, and a nearly-optimal
nearly-feasible solution to the non-convex constrained problem
(\eqref{constrained-problem}) in our theorem below.  This theorem has a few differences from the more typical equivalence between Nash equilibria and optimal feasible solutions in the convex setting. First, it characterizes \emph{mixed} equilibria, in that uniformly sampling from the sequences $\theta^{(t)}$ and
$\lambda^{(t)}$ can be interpreted as defining distributions over
$\Parameters$ and $\Multipliers$. 
Second, we
require compact domains in order to prove convergence rates (below) so
$\Multipliers$ is taken to consist only of sets of Lagrange multipliers with
bounded $1$-norm~\footnote{In \appref{suboptimality}, this is generalized to
$p$-norms.}.

Finally, as a consequence of this second point, the feasibility guarantee of
\eqref{thm:lagrangian-suboptimality-and-feasibility:optimality-and-feasibility}
(right) only holds if the Lagrange multipliers are, on average, smaller than
the maximum $1$-norm radius $\Radius$. Thankfully, as is shown by the final
result of \thmref{lagrangian-suboptimality-and-feasibility}, if there exists a
point satisfying the constraints with some margin $\gamma >  0$, then there will
exist $R$s that are large enough to guarantee feasibility to within
$O(\epsilon)$.

\begin{theorem}
  \label{thm:lagrangian-suboptimality-and-feasibility}
  Define:
   \begin{equation}
  \Lambda \stackrel{\triangle}{=} \{\lambda \in \R_+^m: \| \lambda \|_1 \leq \Radius \}
  \end{equation}
  and let
  $\parameters^{(1)},\dots,\parameters^{(T)} \in \Parameters$ and
  $\multipliers^{(1)},\dots,\multipliers^{(T)} \in \Multipliers$ be sequences of
  parameter vectors and Lagrange multipliers that comprise an approximate mixed Nash equilibrium, \ie:
  \begin{equation*}
    \max_{\multipliers^* \in \Multipliers} \frac{1}{T} \sum_{t=1}^T
    \lagrangian\left( \parameters^{(t)}, \multipliers^* \right) -
    \inf_{\parameters^* \in \Parameters} \frac{1}{T} \sum_{t=1}^T
    \lagrangian\left( \parameters^*, \multipliers^{(t)} \right) \le \epsilon
  \end{equation*}
  Define $\bar{\parameters}$ as a random variable for which $\bar{\parameters}
  = \parameters^{(t)}$ with probability $1/T$, and let $\bar{\multipliers}
  \stackrel{\triangle}{=} \left(\sum_{t=1}^T \multipliers^{(t)}\right) / T$.
  Then $\bar{\parameters}$ is nearly-optimal and nearly-feasible in expectation:
  \begin{equation*}
    \label{eq:thm:lagrangian-suboptimality-and-feasibility:optimality-and-feasibility}
    \expectation_{\bar{\parameters}}\left[
    \objective\left(\bar{\parameters}\right) \right] \le \inf_{\parameters^*
    \in \Parameters : \forall i .  \constraint{i}\left(\parameters^*\right) \le
    0} \objective\left(\parameters^*\right) + \epsilon
    \;\;\;\; \mathrm{and} \;\;\;\;
    \max_{i \in \indices{\numconstraints}}
    \expectation_{\bar{\parameters}}\left[
    \constraint{i}\left(\bar{\parameters}\right) \right] \le
    \frac{\epsilon}{\Radius - \|\bar{\multipliers}\|_1}
  \end{equation*}
  Additionally, if there exists a $\parameters' \in \Parameters$ that satisfies
  all of the constraints with margin $\gamma$ (\ie
  $\constraint{i}\left(\parameters'\right) \le -\gamma$ for all
  $i\in\indices{\numconstraints}$), then:
  \begin{equation*}
    \|\bar{\multipliers}\|_1 \le \frac{\epsilon +
    \bound{\objective}}{\gamma}
  \end{equation*}
  where $\bound{\objective} \ge \sup_{\parameters \in \Parameters}
  \objective\left(\parameters\right) - \inf_{\parameters \in \Parameters}
  \objective\left(\parameters\right)$ is a bound on the range of the objective
  function $\objective$.
\end{theorem}
\begin{proof}
  This is a special case of \thmref{lagrangian-suboptimality} and
  \lemref{lagrangian-feasibility} in \appref{suboptimality}.
\end{proof}

\subsubsection{Convergence of \algref{oracle-lagrangian}}

\algref{oracle-lagrangian}'s
convergence rate is given by the following lemma:
\begin{lem}{oracle-lagrangian}
  \ifshowproofs
  \textbf{(\algref{oracle-lagrangian})}
  \fi
  Suppose that $\Multipliers$ and $\Radius$ are as in
  \thmref{lagrangian-suboptimality-and-feasibility}, and define
  $\bound{\stochasticgrad} \ge \max_{t \in \indices{T}}
  \norm{\stochasticgrad_{\multipliers}^{(t)}}_2$.
  If we run \algref{oracle-lagrangian} with the step size $\eta_{\multipliers}
  \defeq \Radius / \bound{\stochasticgrad} \sqrt{2T}$, then the result
  satisfies \thmref{lagrangian-suboptimality-and-feasibility}
  for:
  \begin{equation*}
    \epsilon = \approximation + \Radius \bound{\stochasticgrad}
    \sqrt{\frac{2}{T}}
  \end{equation*}
  where $\approximation$ is the error associated with the oracle $\oracle$.
\end{lem}
\begin{prf}{oracle-lagrangian}
  Applying \corref{sgd} to the optimization over $\multipliers$ gives:
  \begin{equation*}
    \frac{1}{T} \sum_{t=1}^T \lagrangian\left( \parameters^{(t)},
    \multipliers^* \right) - \frac{1}{T} \sum_{t=1}^T \lagrangian\left(
    \parameters^{(t)}, \multipliers^{(t)} \right) \le \bound{\Multipliers}
    \bound{\stochasticgrad} \sqrt{\frac{2}{T}}
  \end{equation*}
  By the definition of $\oracle$ (\defref{oracle}):
  \begin{equation*}
    \frac{1}{T} \sum_{t=1}^T \lagrangian\left( \parameters^{(t)},
    \multipliers^* \right) - \inf_{\parameters^* \in \Parameters} \frac{1}{T}
    \sum_{t=1}^T \lagrangian\left( \parameters^*, \multipliers^{(t)} \right)
    \le \approximation + \bound{\Multipliers} \bound{\stochasticgrad}
    \sqrt{\frac{2}{T}}
  \end{equation*}
  Using the linearity of $\lagrangian$ in $\multipliers$, the fact that
  $\bound{\Multipliers} = \Radius$, and the definitions of $\bar{\parameters}$
  and $\bar{\multipliers}$, yields the claimed result.
\end{prf}

Combined with \thmref{lagrangian-suboptimality-and-feasibility}, we therefore
have that if $\Radius$ is sufficiently large, then \algref{oracle-lagrangian}
will converge to a distribution over $\Parameters$ that is, in expectation,
$O(\approximation)$-far from being optimal and feasible at a $O(1/\sqrt{T})$
rate, where $\approximation$ is defined in \secref{oracle}.

\subsubsection{Shrinking the Stochastic Solution}\label{sec:lagrangian:shrinking}

A disadvantage of \algref{oracle-lagrangian} is that it results in a mixture of $T$ solutions, which may be large and thus undesirable in practice. However, we can show that much smaller Nash equilibria exist:
\begin{lemma}
  \label{lem:sparse-lagrangian}
  If $\Parameters$ is a compact Hausdorff space, $\Multipliers$ is compact, and
  the objective and constraint functions
  $\objective,\constraint{1},\dots,\constraint{\numconstraints}$ are
  continuous, then the Lagrangian game (\eqref{lagrangian}) has a mixed Nash
  equilibrium pair $\left(\parameters,\multipliers\right)$ where $\parameters$
  is a random variable supported on at most $\numconstraints+1$ elements of
  $\Parameters$, and $\multipliers$ is non-random.
\end{lemma}
\begin{proof}
  Follows from \thmref{sparse-equilibrium} in \appref{sparsity}.
\end{proof}
We do not content ourselves with merely having shown the existence of such an equilibrium. Thankfully, we can re-formulate the problem of finding the optimal
$\epsilon$-feasible mixture of the $\parameters^{(t)}$s as a linear program
(LP) that can be solved to \emph{shrink} the support set to $\numconstraints+1$ solutions. We must first evaluate
the objective and constraint functions for every $\parameters^{(t)}$, yielding
a $T$-dimensional vector of objective function values, and $\numconstraints$
such vectors of constraint function evaluations, which are then used to specify
the LP:
\begin{lem}{sparse-linear-program}
  Let $\parameters^{(1)},\parameters^{(2)},\dots,\parameters^{(T)} \in
  \Parameters$ be a sequence of $T$ ``candidate solutions'' of
  \eqref{constrained-problem}.
  Define $\vec{\objective},\vec{\constraint{i}} \in \R^T$ such that
  $\left(\vec{\objective}\right)_t = \objective\left(\parameters^{(t)}\right)$
  and $\left(\vec{\constraint{i}}\right)_t =
  \constraint{i}\left(\parameters^{(t)}\right)$ for
  $i\in\indices{\numconstraints}$, and consider the linear program:
  \switchreptheorem{
    \begin{align*}
      \min_{p \in \Delta^T} \; & \inner{p}{\vec{\objective}} \\
      \suchthat & \forall i \in \indices{\numconstraints} .
      \inner{p}{\vec{\constraint{i}}} \le \epsilon
    \end{align*}
  }{
    \begin{equation*}
      \min_{p \in \Delta^T} \inner{p}{\vec{\objective}} \;\;\;\;
      \suchthat \forall i \in \indices{\numconstraints} .
      \inner{p}{\vec{\constraint{i}}} \le \epsilon
    \end{equation*}
  }
  where $\Delta^T$ is the $T$-dimensional simplex. Then every vertex $p^*$ of
  the feasible region---in particular an optimal one---has at most
  $\numconstraints^* + 1 \le \numconstraints + 1$ nonzero elements, where
  $\numconstraints^*$ is the number of active
  $\inner{p^*}{\vec{\constraint{i}}} \le \epsilon$ constraints.
\end{lem}
\begin{prf}{sparse-linear-program}
  The linear program contains not only the $\numconstraints$ explicit
  linearized functional constraints, but also, since $p \in \Delta^T$, the $T$
  nonnegativity constraints $p_t \ge 0$, and the sum-to-one constraint
  $\sum_{t=1}^T p_t = 1$.

  Since $p$ is $T$-dimensional, every vertex $p^*$ of the feasible region must
  include $T$ active constraints. Letting $\numconstraints^* \le
  \numconstraints$ be the number of active linearized functional constraints,
  and accounting for the sum-to-one constraint, it follows that at least $T -
  \numconstraints^* - 1$ nonnegativity constraints are active, implying that
  $p^*$ contains at most $\numconstraints^* + 1$ nonzero elements.
\end{prf}

This lemma suggests a two-phase approach to actually finding the $\numconstraints+1$ stochastic solution. In the first phase, apply \algref{oracle-lagrangian}, yielding a sequence of iterates for which
the uniform distribution over the $\parameters^{(t)}$s is approximately
feasible and optimal. Then apply the procedure of \lemref{sparse-linear-program}
to find the \emph{best} distribution over these iterates, which in particular
can be no worse than the uniform distribution, and is supported on at most
$\numconstraints+1$ iterates. 

\subsection{Proxy Constraints and a Non-Zero Sum Game}\label{sec:introduction:non-zero-sum}

Most real-world machine learning implementations use first-order methods (even on non-convex problems, \eg DNNs). To use such a method, however, one must have gradients, and gradients are unavailable for 
rate constraints (as in \eqref{loss}): due to the indicators in the rate constraint expression (\eqref{rateConstraint}), the constraint functions are piecewise-constant, so their gradients are zero almost everywhere, and a gradient-based method cannot be expected to succeed.
In general, for constrained optimization problems in the form of \eqref{constrained-problem}, non-differentiable constraints arise naturally when one wishes to constrain counts or proportions.

The obvious solution is to use a surrogate. For example, we might consider replacing the indicators defining a rate with sigmoids, and then optimizing the Lagrangian. This solves the differentiability problem, but introduces a new one: a (mixed) Nash equilibrium would correspond to a solution satisfying the sigmoid-relaxed constraints, instead of the \emph{actual} constraints.
Interestingly, it turns out that we can seek to satisfy the original un-relaxed constraints, even while using a surrogate. Our proposal is motivated by the observation that, while differentiating the Lagrangian (\eqref{lagrangian}) \wrt $\parameters$ requires differentiating the constraint functions $\constraint{i}\left(\parameters\right)$, to differentiate it \wrt $\multipliers$ we only need to \emph{evaluate} them. Hence, a surrogate is only necessary for the $\parameters$-player; the $\multipliers$-player can continue to use the original constraint functions.

We refer to a surrogate that is used by only one of the two players as a ``proxy'', and introduce the notion of ``proxy constraints'' by taking $\proxyconstraint{i}\left(\parameters\right)$ to be a sufficiently-smooth upper bound on $\constraint{i}\left(\parameters\right)$ for $i\in\indices{\numconstraints}$, and formulating two functions that we call ``proxy-Lagrangians'':
\begin{align}
  \label{eq:proxy-lagrangian}
  \mathcal{L}_{\theta}(\theta, \lambda) & \stackrel{\triangle}{=}
  \lambda_1 g_0(\theta) + \sum_{i=1}^{m} \lambda_{i+1}
  \tilde{g}_I(\theta)\\
  \notag
  \mathcal{L}_{\lambda}(\theta, \lambda) & \stackrel{\triangle}{=}
  \sum_{i=1}^{m} \lambda_{i+1} g_i(\theta)
\end{align}
where we restrict $\Lambda$ to be the $(m+1)$-dimensional simplex $\Delta^{m+1}$. The $\parameters$-player seeks to minimize $\lagrangian_{\parameters}\left(\parameters, \multipliers\right)$, while the $\multipliers$-player seeks to maximize $\lagrangian_{\multipliers}\left(\parameters, \multipliers\right)$. Notice that the $\proxyconstraint{i}$s are \emph{only} used by the $\parameters$-player. Intuitively, the $\multipliers$-player chooses how much to weigh the proxy constraint functions, but---and this is the key to our proposal---does so in such a way as to satisfy the \emph{original} constraints.

Viewed as a two-player game, what we have changed is that now the $\theta$ and $\lambda$  players each have their own payoff functions $\mathcal{L}_{\theta}(\theta, \lambda)$ and $\mathcal{L}_{\lambda}(\theta, \lambda)$ respectively, making the game {\it non-zero sum}. 
Finding a Nash equilibrium of a non-zero-sum game is much more difficult than for a zero-sum game---in fact, it's PPAD-complete even in the finite setting~\citep{Chen:2006}. We will present a procedure which approximates a \emph{weaker} type of equilibrium: instead of converging to a Nash equilibrium, it converges to a new solution concept, which we call a {\it semi-coarse correlated equilibrium}. Despite being weaker than a Nash equlibrium, we show that it still corresponds to a nearly-optimal and nearly-feasible solution to constrained optimization in expectation.

The proxy-Lagrangian formulation leads to a tighter approximation than the popular approach of using a surrogate for \emph{both} players, as has been  proposed \eg for Neyman-Pearson classification~\citep{Davenport:2010,Bottou:2011}, and AUC optimization~\citep{Eban:2017}. These proposals optimize a simpler problem (a zero-sum game), but one that is a worse reflection of the true goal. In the experimental section, we will provide evidence that the proposed proxy-Lagrangian formulation can provide higher accuracy while still satisying the constraints. This is especially important when the rate constraints express real-world restrictions on how the learned model is permitted to behave. For example, if we require an 80\% threshold in terms of the \emph{number of positive predictions}, we would like that and not a relaxation of this.

\subsection{Proxy-Lagrangian Equilibrium}\label{sec:proxy-lagrangian-equilibrium}

For the proxy-Lagrangian game (\eqref{proxy-lagrangian}), we cannot expect to
find a Nash equilibrium, at least not efficiently, since it is non-zero-sum.
However, the analogous result to
\thmref{lagrangian-suboptimality-and-feasibility} requires a \emph{weaker} type
of equilibrium: a joint distribution over $\Theta$ and $\Lambda$ \wrt
which the $\theta$-player can only make a negligible improvement compared to
the best constant strategy, and the $\lambda$-player compared to the best
action-swapping strategy (this is a type of $\Phi$-correlated
equilibrium~\citep{Rakhlin:2011}). We call this semi-coarse-correlated equilibrium. We present our theorem showing the achievability of this type of equilibrium, then we present \algref{oracle-proxy-lagrangian} to satisfy the theorem. 
\begin{theorem}
  \label{thm:proxy-lagrangian-suboptimality-and-feasibility}
  Define
  %
  %
  $\Matrixmultipliers$
  as the set of all left-stochastic $\left(\numconstraints + 1\right) \times
  \left(\numconstraints + 1\right)$ matrices, $\Multipliers \stackrel{\triangle}{=}
  \Delta^{\numconstraints+1}$ as the
  $\left(\numconstraints+1\right)$-dimensional simplex, and assume that each
  $\proxyconstraint{i}$ upper bounds the corresponding $\constraint{i}$.
  Let $\parameters^{(1)},\dots,\parameters^{(T)} \in \Parameters$ and
  $\multipliers^{(1)},\dots,\multipliers^{(T)} \in \Multipliers$ be sequences
  %
  %
  satisfying:
  \begin{align*}
    \frac{1}{T} \sum_{t=1}^T \lagrangian_{\parameters}\left( \parameters^{(t)},
    \multipliers^{(t)} \right) - \inf_{\parameters^* \in \Parameters}
    \frac{1}{T} \sum_{t=1}^T \lagrangian_{\parameters}\left( \parameters^*,
    \multipliers^{(t)} \right) \le& \epsilon_{\parameters} \\
    \max_{\matrixmultipliers^* \in \Matrixmultipliers} \frac{1}{T} \sum_{t=1}^T
    \lagrangian_{\multipliers}\left( \parameters^{(t)}, \matrixmultipliers^*
    \multipliers^{(t)} \right) - \frac{1}{T} \sum_{t=1}^T
    \lagrangian_{\multipliers}\left( \parameters^{(t)}, \multipliers^{(t)}
    \right) \le& \epsilon_{\multipliers}
  \end{align*}
  Define $\bar{\parameters}$ as a random variable for which $\bar{\parameters}
  = \parameters^{(t)}$ with probability $\multipliers^{(t)}_1 / \sum_{s=1}^T
  \multipliers^{(s)}_1$, and let $\bar{\multipliers} \stackrel{\triangle}{=} \left(\sum_{t=1}^T
  \multipliers^{(t)}\right) / T$.
  Then $\bar{\parameters}$ is nearly-optimal and nearly-feasible in expectation:
  \begin{equation}
 \label{eq:thm:proxy-lagrangian-suboptimality-and-feasibility:suboptimality}
    \expectation_{\bar{\parameters}}\left[
    \objective\left(\bar{\parameters}\right) \right] \le \inf_{\parameters^*
    \in \Parameters : \forall i .
    \proxyconstraint{i}\left(\parameters^*\right) \le 0} \objective\left(
    \parameters^* \right) + \frac{\epsilon_{\parameters} +
    \epsilon_{\multipliers}}{\bar{\multipliers}_1}
  \end{equation}
  and, 
    \begin{equation}    \label{eq:thm:proxy-lagrangian-suboptimality-and-feasibility:feasibility}
    \max_{i \in \indices{\numconstraints}}
    \expectation_{\bar{\parameters}}\left[
    \constraint{i}\left(\bar{\parameters}\right) \right] \le
    \frac{\epsilon_{\multipliers}}{\bar{\multipliers}_1}
  \end{equation}
  Additionally, if there exists a $\parameters' \in \Parameters$ that satisfies
  all of the proxy constraints with margin $\gamma$ (\ie
  $\proxyconstraint{i}\left(\parameters'\right) \le -\gamma$ for all
  $i\in\indices{\numconstraints}$), then:
  \begin{equation*}
    \bar{\multipliers}_1 \ge \frac{\gamma - \epsilon_{\parameters} -
    \epsilon_{\multipliers}}{\gamma + \bound{\objective}}
  \end{equation*}
  where $\bound{\objective} \ge \sup_{\parameters \in \Parameters}
  \objective\left(\parameters\right) - \inf_{\parameters \in \Parameters}
  \objective\left(\parameters\right)$ is a bound on the range of the objective
  function $\objective$.
\end{theorem}
\begin{proof}
  This is a special case of \thmref{proxy-lagrangian-suboptimality} and
  \lemref{proxy-lagrangian-feasibility} in \appref{suboptimality}.
\end{proof}
Notice that
\eqref{thm:proxy-lagrangian-suboptimality-and-feasibility:feasibility} guarantees feasibility \wrt the original constraints,  while \eqref{thm:proxy-lagrangian-suboptimality-and-feasibility:suboptimality} shows that the solution minimizes the objective approximately  as well as the best solution that's feasible \wrt the \emph{proxy} constraints. Hence, the guarantee for minimizing the objective
is no better than what we would have obtained if we took  $\constraint{i}
\stackrel{\triangle}{=} \proxyconstraint{i}$ for all $i \in \indices{\numconstraints}$, and
optimized the Lagrangian as in \secref{lagrangian}.
However, because the
feasible region \wrt the original constraints is larger (perhaps significantly
so) than that \wrt the proxy constraints, the proxy-Lagrangian approach has
more ``room'' to find a better solution in practice (this is demonstrated in the experiments).

One key difference between this result and
\thmref{lagrangian-suboptimality-and-feasibility} is that the $\Radius$
parameter is absent. Instead, its role, and that of
$\norm{\bar{\multipliers}}_{1}$, is played by the first coordinate of
$\bar{\multipliers}$. Inspection of \eqref{proxy-lagrangian} reveals that, if
one or more of the constraints are violated, then the $\multipliers$-player
would prefer the corresponding entries in $\multipliers$ to be higher, which in turn causes $\multipliers_1$ to become closer to $0$ from our procedures. Likewise, if they are satisfied (with
some margin), then it would prefer the entries after the first in $\multipliers$ to be $0$ which causes $\multipliers_1$ to be one in our procedures. In other words,
the first coordinate of $\multipliers^{(t)}$ encodes the
$\multipliers$-player's belief about the feasibility of $\parameters^{(t)}$,
for which reason $\parameters^{(t)}$ is weighted by $\multipliers^{(t)}_1$ in
the density defining $\bar{\parameters}$.
%
\begin{algorithm*}[t]

\begin{pseudocode}
\codename $\mbox{OracleProxyLagrangian}\left( \lagrangian_{\parameters}, \lagrangian_{\multipliers} : \Parameters \times \Delta^{\numconstraints+1} \rightarrow \R, \oracle : \left(\Parameters \rightarrow \R\right) \rightarrow \Parameters, T \in \N, \eta_{\multipliers} \in \R_+ \right)$: \\
\codeline Initialize $\matrixmultipliers^{(1)} \in \R^{\left(\numconstraints + 1\right) \times \left(\numconstraints + 1\right)}$ with $\matrixmultipliers_{i,j} = 1 / \left(\numconstraints+1\right)$ \\
\codeline For $t \in \indices{T}$: \\
\codeline \> Let $\multipliers^{(t)} = \pi\left( \matrixmultipliers^{(t)}\right)$ \codecomment{Stationary distribution of $\matrixmultipliers^{(t)}$} \\
\codeline \> Let $\parameters^{(t)} = \oracle\left( \lagrangian_{\parameters}\left(\cdot,\multipliers^{(t)}\right) \right)$ \codecomment{Oracle optimization} \\
\codeline \> Let $\stochasticgrad^{(t)}_{\multipliers}$ be a gradient of $\lagrangian_{\multipliers}\left(\parameters^{(t)},\multipliers^{(t)}\right)$ \wrt $\multipliers$ \\
\codeline \> Update $\tilde{\matrixmultipliers}^{(t+1)} = \matrixmultipliers^{(t)} \elementwiseproduct \elementwiseexp\left( \eta_{\multipliers} \stochasticgrad^{(t)}_{\multipliers} \left( \multipliers^{(t)} \right)^T \right)$ \codecomment{$\elementwiseproduct$ and $\elementwiseexp$ are element-wise} \\
\codeline \> Project $\matrixmultipliers^{(t+1)}_{:,i} = \tilde{\matrixmultipliers}^{(t+1)}_{:,i} / \norm{\tilde{\matrixmultipliers}^{(t+1)}_{:,i}}_1$ for $i\in\indices{\numconstraints+1}$ \codecomment{Column-wise projection}  \\
\codeline Return $\parameters^{(1)},\dots,\parameters^{(T)}$ and $\multipliers^{(1)},\dots,\multipliers^{(T)}$
\end{pseudocode}

\caption{
  Optimizes the proxy-Lagrangian formulation (\eqref{proxy-lagrangian}) in
  the non-convex setting via the use of an approximate Bayesian optimization
  oracle $\oracle$ (\defref{oracle}, but with $\proxyconstraint{i}$s instead of
  $\constraint{i}$s in the linear combination defining $f$) for the
  $\parameters$-player, with the $\multipliers$-player minimizing swap regret.
  The $\pi(\matrixmultipliers)$ operation on line $3$ results in a stationary
  distribution of $\matrixmultipliers$ (\ie a $\multipliers \in \Multipliers$
  such that $\matrixmultipliers \multipliers = \multipliers$, which can be
  derived from the top eigenvector).
}

\label{alg:oracle-proxy-lagrangian}

\end{algorithm*}

\begin{algorithm*}[t]

\begin{pseudocode}
\codename $\mbox{StochasticProxyLagrangian}\left(\lagrangian_{\parameters}, \lagrangian_{\multipliers} : \Parameters \times \Delta^{\numconstraints+1} \rightarrow \R, T \in \N, \eta_{\parameters}, \eta_{\multipliers} \in \R_+ \right)$: \\
\codeline Initialize $\parameters^{(1)} = 0$ \codecomment{Assumes $0 \in \Parameters$} \\
\codeline Initialize $\matrixmultipliers^{(1)} \in \R^{\left(\numconstraints + 1\right) \times \left(\numconstraints + 1\right)}$ with $\matrixmultipliers_{i,j} = 1 / \left(\numconstraints+1\right)$ \\
\codeline For $t \in \indices{T}$: \\
\codeline \> Let $\multipliers^{(t)} = \pi\left( \matrixmultipliers^{(t)}\right)$ \codecomment{Stationary distribution of $\matrixmultipliers^{(t)}$} \\
\codeline \> Let $\stochasticsubgrad^{(t)}_{\parameters}$ be a stochastic subgradient of $\lagrangian_{\parameters}\left(\parameters^{(t)},\multipliers^{(t)}\right)$ \wrt $\parameters$ \\
\codeline \> Let $\stochasticgrad^{(t)}_{\multipliers}$ be a stochastic gradient of $\lagrangian_{\multipliers}\left(\parameters^{(t)},\multipliers^{(t)}\right)$ \wrt $\multipliers$ \\
\codeline \> Update $\parameters^{(t+1)} = \Pi_{\Parameters}\left( \parameters^{(t)} - \eta_{\parameters} \stochasticsubgrad^{(t)}_{\parameters} \right)$ \codecomment{Projected SGD update} \\
\codeline \> Update $\tilde{\matrixmultipliers}^{(t+1)} = \matrixmultipliers^{(t)} \elementwiseproduct \elementwiseexp\left( \eta_{\multipliers} \stochasticgrad^{(t)}_{\multipliers} \left( \multipliers^{(t)} \right)^T \right)$ \codecomment{$\elementwiseproduct$ and $\elementwiseexp$ are element-wise} \\
\codeline \> Project $\matrixmultipliers^{(t+1)}_{:,i} = \tilde{\matrixmultipliers}^{(t+1)}_{:,i} / \norm{\tilde{\matrixmultipliers}^{(t+1)}_{:,i}}_1$ for $i\in\indices{\numconstraints+1}$ \codecomment{Column-wise projection} \\
\codeline Return $\parameters^{(1)},\dots,\parameters^{(T)}$ and $\multipliers^{(1)},\dots,\multipliers^{(T)}$
\end{pseudocode}

\caption{
  Optimizes the proxy-Lagrangian formulation (\eqref{proxy-lagrangian}) in
  the convex setting, with the $\parameters$-player minimizing external regret,
  and the $\multipliers$-player minimizing swap regret.
  The $\pi(\matrixmultipliers)$ operation on line $4$ outputs the stationary
  distribution of $\matrixmultipliers$ (that is, a $\multipliers \in \Multipliers$
  such that $\matrixmultipliers \multipliers = \multipliers$) which can be
  derived from the top eigenvector.
  The function $\Pi_{\Parameters}$ projects its argument onto $\Parameters$
  \wrt the Euclidean norm.
}

\label{alg:stochastic-proxy-lagrangian}

\end{algorithm*}

\subsection{Proxy-Lagrangian Optimization Algorithm}\label{sec:proxy-lagrangian-algorithm}

To optimize the proxy-Lagrangian formulation, we present \algref{oracle-proxy-lagrangian}, which is motivated by the observation that,
while \thmref{proxy-lagrangian-suboptimality-and-feasibility} only requires
that the $\parameters^{(t)}$ sequence suffer low external regret \wrt
$\lagrangian_{\parameters}\left(\cdot, \multipliers^{(t)}\right)$, the
condition on the $\multipliers^{(t)}$ sequence is stronger, requiring it to
suffer low \emph{swap regret}~\citep{Blum:2007} \wrt $\lagrangian_{\multipliers}\left(\parameters^{(t)}, \cdot\right)$.

Hence, the $\parameters$-player uses the oracle to minimize external regret, while the
$\multipliers$-player uses a swap-regret minimization algorithm of the type
proposed by \citet{Gordon:2008}, yielding the convergence guarantee:
\begin{lem}{oracle-proxy-lagrangian}
  \ifshowproofs
  \textbf{(\algref{oracle-proxy-lagrangian})}
  \fi
  Suppose that $\Matrixmultipliers$ and $\Multipliers$ are as in
  \thmref{proxy-lagrangian-suboptimality-and-feasibility}, and define the upper
  bound $\bound{\stochasticgrad} \ge \max_{t \in \indices{T}}
  \norm{\stochasticgrad_{\multipliers}^{(t)}}_{\infty}$.

  If we run \algref{oracle-proxy-lagrangian} with the step size
  $\eta_{\multipliers} \defeq \sqrt{ \left(\numconstraints+1\right) \ln
  \left(\numconstraints+1\right) / T \bound{\stochasticgrad}^2 }$, then the
  result satisfies satisfies the conditions of
  \thmref{proxy-lagrangian-suboptimality-and-feasibility} for:
  \begin{align*}
    \epsilon_{\parameters} =& \approximation \\
    \epsilon_{\multipliers} =& 2 \bound{\stochasticgrad} \sqrt{ \frac{
    \left(\numconstraints+1\right) \ln \left(\numconstraints+1\right) }{T} }
  \end{align*}
  where $\approximation$ is the error associated with the oracle $\oracle$.
\end{lem}
\begin{prf}{oracle-proxy-lagrangian}
  Applying \lemref{internal-regret} to the optimization over $\multipliers$
  (with $\matrixmultipliersize \defeq \numconstraints + 1$) gives:
  \begin{equation*}
    \frac{1}{T} \sum_{t=1}^T \lagrangian_{\multipliers}\left(
    \parameters^{(t)}, \matrixmultipliers^* \multipliers^{(t)} \right) -
    \frac{1}{T} \sum_{t=1}^T \lagrangian_{\multipliers}\left(
    \parameters^{(t)}, \multipliers^{(t)} \right) \le 2
    \bound{\stochasticgrad} \sqrt{ \frac{ \left(\numconstraints+1\right) \ln
    \left(\numconstraints+1\right) }{T} }
  \end{equation*}
  By the definition of $\oracle$ (\defref{oracle}):
  \begin{equation*}
    \frac{1}{T} \sum_{t=1}^T \lagrangian_{\parameters}\left( \parameters^{(t)},
    \multipliers^{(t)} \right) - \inf_{\parameters^* \in \Parameters}
    \frac{1}{T} \sum_{t=1}^T \lagrangian_{\parameters}\left( \parameters^*,
    \multipliers^{(t)} \right) \le \approximation
  \end{equation*}
  Using the definitions of $\bar{\parameters}$ and $\bar{\multipliers}$ yields
  the claimed result.
\end{prf}

\subsection{Practical Stochastic Proxy-Lagrangian Algorithm}\label{sec:practical-proxy-lagrangian-algorithm}


\algref{stochastic-proxy-lagrangian} is designed for the convex setting (except
for the $\constraint{i}$s), thus we can safely use SGD for the $\parameters$-updates instead of the oracle and enjoy a more practical procedure. 
We stress that this is a considerable improvement over previous Lagrangian methods in the convex setting, as they require both the loss and constraints to be convex in order to attain optimality and feasibility guarantees. Here, while we assume convexity of the objective and proxy-constraints, the \emph{original} constraints do not need to be convex, but we are still able to prove similar guarantees.

\begin{lem}{stochastic-proxy-lagrangian}
  \ifshowproofs
  \textbf{(\algref{stochastic-proxy-lagrangian})}
  \fi
  Suppose that $\Parameters$ is a compact convex set, $\Matrixmultipliers$ and
  $\Multipliers$ are as in
  \thmref{proxy-lagrangian-suboptimality-and-feasibility}, and that the
  objective and proxy constraint functions
  $\objective,\proxyconstraint{1},\dots,\proxyconstraint{\numconstraints}$ are
  convex (but not $\constraint{1},\dots,\constraint{\numconstraints}$). Define
  the three upper bounds $\bound{\Parameters} \ge \max_{\parameters \in
  \Parameters} \norm{\parameters}_2$, $\bound{\stochasticsubgrad} \ge \max_{t
  \in \indices{T}} \norm{\stochasticsubgrad_{\parameters}^{(t)}}_2$, and
  $\bound{\stochasticgrad} \ge \max_{t \in \indices{T}}
  \norm{\stochasticgrad_{\multipliers}^{(t)}}_{\infty}$.

  If we run \algref{stochastic-proxy-lagrangian} with the step sizes
  $\eta_{\parameters} \defeq \bound{\Parameters} / \bound{\stochasticsubgrad}
  \sqrt{2T}$ and $\eta_{\multipliers} \defeq \sqrt{
  \left(\numconstraints+1\right) \ln \left(\numconstraints+1\right) / T
  \bound{\stochasticgrad}^2 }$, then the result satisfies the conditions of
  \thmref{proxy-lagrangian-suboptimality-and-feasibility} for:
  \switchreptheorem{
    \begin{align*}
      \epsilon_{\parameters} =& 2 \bound{\Parameters} \bound{\stochasticsubgrad}
      \sqrt{ \frac{ 1 + 16 \ln\frac{2}{\delta} }{T} } \\
      \epsilon_{\multipliers} =& 2 \bound{\stochasticgrad} \sqrt{ \frac{ 2
      \left(\numconstraints+1\right) \ln \left(\numconstraints+1\right) \left( 1
      + 16 \ln\frac{2}{\delta}\right) }{T} }
    \end{align*}
  }{
    \begin{equation*}
      \epsilon_{\parameters} = 2 \bound{\Parameters} \bound{\stochasticsubgrad}
      \sqrt{ \frac{ 1 + 16 \ln\frac{2}{\delta} }{T} }
      \;\;\;\; \mathrm{and} \;\;\;\;
      \epsilon_{\multipliers} = 2 \bound{\stochasticgrad} \sqrt{ \frac{ 2
      \left(\numconstraints+1\right) \ln \left(\numconstraints+1\right) \left( 1
      + 16 \ln\frac{2}{\delta}\right) }{T} }
    \end{equation*}
  }
  with probability $1-\delta$ over the draws of the stochastic
  (sub)gradients.
\end{lem}
\begin{prf}{stochastic-proxy-lagrangian}
  Applying \corref{stochastic-sgd} to the optimization over $\parameters$, and
  \lemref{stochastic-internal-regret} to that over $\multipliers$ (with
  $\matrixmultipliersize \defeq \numconstraints + 1$), gives that with
  probability $1-2\delta'$ over the draws of the stochastic
  (sub)gradients:
  \begin{align*}
    \frac{1}{T} \sum_{t=1}^T \lagrangian_{\parameters}\left( \parameters^{(t)},
    \multipliers^{(t)} \right) - \frac{1}{T} \sum_{t=1}^T
    \lagrangian_{\parameters}\left( \parameters^*, \multipliers^{(t)} \right)
    \le& 2 \bound{\Parameters} \bound{\stochasticsubgrad} \sqrt{ \frac{ 1 + 16
    \ln\frac{1}{\delta'} }{T} } \\
    \frac{1}{T} \sum_{t=1}^T \lagrangian_{\multipliers}\left(
    \parameters^{(t)}, \matrixmultipliers^* \multipliers^{(t)} \right) -
    \frac{1}{T} \sum_{t=1}^T \lagrangian_{\multipliers}\left(
    \parameters^{(t)}, \multipliers^{(t)} \right) \le& 2
    \bound{\stochasticgrad} \sqrt{ \frac{ 2 \left(\numconstraints+1\right)
    \ln \left(\numconstraints+1\right) \left( 1 + 16
    \ln\frac{1}{\delta'}\right) }{T} }
  \end{align*}
  Taking $\delta=2\delta'$, and using the definitions of $\bar{\parameters}$
  and $\bar{\multipliers}$, yields the claimed result.
\end{prf}

\subsection{Shrinking the Stochastic Proxy Lagrangian Solution}\label{sec:proxy-lagrangian:shrinking}

%
Like \algref{oracle-lagrangian}, \algrefs{oracle-proxy-lagrangian}{stochastic-proxy-lagrangian} return a stochastic solutions with support on $T$ discrete solutions. Again, we show that we can find just as good a stochastic solution with minimal support on $m+1$ discrete solutions. 

It turns out that the same existence result that we provided for the Lagrangian
game (\lemref{sparse-lagrangian})---of a \emph{Nash} equilibrium---holds for
the proxy-Lagrangian (this is \lemref{sparse-proxy-lagrangian} in
\appref{sparsity}).
Furthermore, the exact same linear programming procedure of
\lemref{sparse-linear-program} can be applied (with the $\vec{\constraint{i}}$s
being defined in terms of the \emph{original}---not proxy---constraints) to
yield a solution with support size $\numconstraints+1$, and works equally well.
This is easy to verify: since $\bar{\parameters}$, as defined in
\thmref{proxy-lagrangian-suboptimality-and-feasibility}, is a distribution over
the $\parameters^{(t)}$s, and is therefore feasible for the LP, the \emph{best}
distribution over the iterates will be at least as good.

\section{Experiments}\label{sec:experiments}

We illustrate the broad applicability of rate constraints and investigate how well different optimization strategies perform.   We use the experiments to investigate the following questions: 

\noindent\textbf{Do rate constraints help in practice?}
\begin{itemize}
    \item Can we effectively solve the rate-constrained optimization problem?
    \item Can we get good results \emph{at test time} by training with rate constraints?
    \item Do rate constraints interact well with other types of constraints (\eg data-independent shape constraints)?
\end{itemize}

\noindent\textbf{Does the proxy-Lagrangian better solve the constrained optimization problem?}
\begin{itemize}
    \item Does simply using a hinge surrogate for both players over-constrain?
    \item Does the proposed proxy-Lagrangian formulation result in better solutions?
    \item With the proxy-Lagrangian, is it necessary in practice for the $\lambda$-player to minimize the swap regret or does simply minimizing the external regret work just as well?
\end{itemize}

\noindent\textbf{Do we really need stochastic classifiers?}
\begin{itemize}
    \item Do the iterates oscillate due to non-existence of an equilibrium in the non-convex setting, causing the last iterate to sometimes be very bad?
    \item Does the proposed sparsely supported $m$-stochastic classifier work at least as well in practice as the $T$-stochastic classifier?
    \item Does the \emph{best} iterate perform as well as the stochastic classifiers?
\end{itemize}

To investigate these questions, we compared twelve optimization algorithms for each of seven datasets. \tabref{datasets} lists the three benchmark and four real-world datasets we used, each  randomly split into train, validation and test sets. We experimented with seven different rate constraints and monotonicity constraints \citep{Groeneboom:2014} as described in \tabref{expConstraints} and the following subsections. The last column of \tabref{expConstraints} states whether the classifier had access to information about the different datasets used in the constraints, for example, if there were ten constraints defined on ten different countries, was there a features in the feature vector $x$ that specified which country the example belonged to?

More details about each dataset and the chosen constraints are given in the following subsections.

As listed in \tabref{datasets}, we performed the experiments on linear models and two types of nonlinear models: standard two-layer ReLU neural nets (NN), and two-layer calibrated ensemble of random tiny lattices (RTLs) \citep{Canini:2016}. 

The rest of this section delves deeper into experimental details and result tables. Then \secref{results}  discusses the results and how they provide positive and negative evidence for the above research questions -- the reader may prefer to skip to \secref{results} and only consult the following experimental details as needed.

\subsection{TensorFlow Implementation}
Our experiments were all run using TensorFlow. We have already open-sourced our implementation of Lagrangian and proxy-Lagrangian optimization at
\url{https://github.com/tensorflow/tensorflow/tree/r1.10/tensorflow/contrib/constrained_optimization}.  (Note to Reviewers: by November 2018 we plan to also open-source a user-friendly API for this package, specifically for rate constraints.)

The lattice models were implemented with the open-source Tensor Flow Lattice package, and consist of learned one-dimensional piecewise linear feature transformations followed by an ensemble of interpolated look-up tables (aka lattices).  All model parameters are jointly trained. For more details on lattice models see \citet{Gupta:2016},\citet{Canini:2016},\citet{You:2017}.  Lattice models can be efficiently constrained for \emph{partial monotonicity} shape constraints, where the term \emph{partial} refers to the practitioner specifying which features can only have a positive (or negative) impact on $f(x)$. To produce the desired partial monotoncity, a large number of data-independent linear inequality constraints are needed, each constraining a pair of model parameters. In the Tensor Flow Lattice package, these monotonicity shape constraints are handled by a projection after each minibatch of stochastic gradients, see \citet{You:2017} for more details.

\subsection{Hyperparameter Optimization}
For each of the different datasets, we fix the number of loops and model architecture ahead of time to perform well for the unconstrained problem. For the presented results, we validate the ADAM learning rate. Then for each of the twelve compared optimizations, we validate the two ADAM learning rates, one for optimizing the model parameters $\theta$, and the other for optimizing the constraints parameters $\lambda$. All ADAM learning rates were varied by powers of 10 around the usual default of ADAM learning rate of $0.001$.

The usual strategy of choosing hyperparameters that score best on the validation set is not satisfying in the constrained optimization setting, because now there are two metrics of interest: accuracy and constraint violation, and the appropriate trade-off between them may be problem dependent. One solution researchers turn to is to side-step the issue of choosing one set of hyperparameters, and instead present the Pareto frontier of results over many hyperparameters on the test set. While certainly valuable in a research setting, we must be mindful that in practice one cannot see the Pareto frontier on the test set, and must make a choice for hyperparameters based only on the training and validation sets (as is standard). 

For our experiments, we investigate the practical setting in which one must choose one set of hyperparameters on which to evaluate the test set. For that, we need a heuristic to choose the best hyperparameters based only on the training and validation data. We analyzed a number of such heuristics that differently balance the validation accuracy and constraint violation, and were unable to find any heuristic that was perfect, but settled on the following strategy that has some nice properties. For any set $\beta$ of hyperparameter choices, let  $\textrm{LossRank}(\beta)$ be the rank ($1 \ldots, B$) of the validation loss using a model trained with hyperparameter set $\beta$, with $\textrm{LossRank}(\beta)= 1$ corresponding to the smallest loss, and let  $\textrm{WorstConstraintRank}(\beta)$ be the rank ($1 \ldots, B$) of the maximum constraint violation on the validation set. Then choose the hyperparameter set $\beta$ that satisfies:
\begin{align}\label{eq:heuristic}
\argmin_{\beta} \: \max \left\{ \textrm{LossRank}(\beta), \textrm{WorstConstraintRank}(\beta) \right\}, 
\end{align}
with ties broken by the minimizing the validation loss.

This strategy chooses the hyperparameter set that has both low loss and small constraint violations, and guarantees that no other hyperparameter set choice would have both better validation accuracy and smaller constraint violations.

\begin{table}
\centering
  \caption{Datasets and Model Types Used in Experiments}
  \label{tab:datasets}
  \begin{tabular}{lrrrrlrll}
 \toprule
Dataset & Features & Train & Valid & Test & Model Type & \# Model \\
 & &  &  &  &  & Parameters \\
\midrule
Bank Marketing & 60 & 31,647 & 4,521 & 9,042 & Linear & 61 \\
Adult  & 122  & 34,189 & 4,884 & 9,768 & Linear  & 123 \\
COMPAS & 31 & 4,320 & 612 & 1,225 & 2 Layer NN & 10 hidden units \\
Business Entity  & 37 & 11,560 & 3,856 & 3,856 & 2 Layer NN   & 16 hidden units \\
Thresholding & 7 & 70,874 & 10,125 & 20,250 & 2 Layer NN & 32 hidden units\\
Map Intent & 32 & 420,000 & 60,000 & 120,000 & RTL & 93,600\\
Filtering & 16 & 1,282,532 & 183,219 & 366,440 & RTL & 3,305 \\
  \bottomrule
\end{tabular}
\end{table}

\begin{table}
\centering
  \caption{Constraints Used in Experiments}
  \label{tab:expConstraints}
  \begin{tabular}{lll}
 \toprule
Dataset & Constraints (\# of constraints) & Constraint Group in $x$?\\
\midrule
Bank Marketing  & Demographic Parity (5) & Y  \\
Adult  &  Equal Opportunity (4) &  Y  \\
COMPAS  & Equal Opportunity (4) &  Y   \\
Business Entity Res. & Minimum Recall (18) and Equal Accuracy (1) & Y \\
Thresholding & Steering Examples Min Acc. (1) & N  \\
Map Intent & No Worse Off (10), Monotonicity (148,800) & Y \\ 
Filtering & Loss-only Churn (11), Monotonicity (9,740) & Y \\
  \bottomrule
\end{tabular}
\end{table}

\subsection{Algorithms Tested}
We experiment with four groups of algorithms:
\begin{enumerate}
\item {\bf Unconstrained:} the model is trained without any constraints.
\item {\bf Hinge:} We use the common approach of using a hinge relaxation of the constraints in place of the actual constraints in the Lagrangian. This approach refers to that of \algref{stochastic-lagrangian}.
\item {\bf 0-1 swap:} This refers to \algref{stochastic-proxy-lagrangian}, which directly uses the 0-1 constraint in the proxy-Lagrangian, the $\lambda$-player minimizes swap-regret and the $\theta$-player minimizes external regret. 
\item {\bf 0-1 ext:} This refers to \algref{proxy-additive-ext} training the non-zero-sum game where $\theta$ player minimizes the original Lagrangian but the $\lambda$-player minimizes external-regret on the Lagrangian with the original constraints replaced by the proxy constraints. This is the ``obvious'' non-zero-sum analogue of the Lagrangian, but does not enjoy the theoretical guarantees of the proxy-Lagrangian. This is used as a comparison to 0-1 swap to see whether minimizing external regret (instead of the more complex swap regret) suffices in practice.
\end{enumerate}
\begin{algorithm*}[t]

\begin{pseudocode}

\codename $\mbox{ProxyAdditiveExternalLagrangian}\left( \Radius \in \R_+, g_0 : \Theta \rightarrow \mathbb{R}, \{g_i\}_{i=1}^m, \{\widetilde{g_i}\}_{i=1}^m, T \in \N, \eta_{\parameters}, \eta_{\multipliers} \in \R_+ \right)$: \\
\codeline Initialize $\parameters^{(1)} = 0$, $\multipliers^{(1)} = 0$ \codecomment{Assumes $0 \in \Parameters$} \\
\codeline For $t \in \indices{T}$: \\
\codeline \> Let $\stochasticsubgrad^{(t)}_{\parameters}$ be a stochastic subgradient of $g_0(\parameters^{(t)}) + \sum_{i=1}^m \multipliers^{(t)}_i g_i(\theta)$ \wrt $\parameters$ \\
\codeline \> Let $\stochasticgrad^{(t)}_{\multipliers}$ be a stochastic gradient of  $g_0(\parameters^{(t)}) + \sum_{i=1}^m \multipliers^{(t)}_i \widetilde{g}_i(\theta)$ \wrt $\multipliers$ \\
\codeline \> Update $\parameters^{(t+1)} = \Pi_{\Parameters}\left( \parameters^{(t)} - \eta_{\parameters} \stochasticsubgrad^{(t)}_{\parameters} \right)$ \codecomment{Projected SGD updates \dots} \\
\codeline \> Update $\multipliers^{(t+1)} = \Pi_{\Multipliers}\left( \multipliers^{(t)} + \eta_{\multipliers} \stochasticgrad^{(t)}_{\multipliers} \right)$ \codecomment{\;\;\;\;\dots} \\
\codeline Return $\parameters^{(1)},\dots,\parameters^{(T)}$ and $\multipliers^{(1)},\dots,\multipliers^{(T)}$
\end{pseudocode}

\caption{
  Optimizes the Lagrangian formulation \emph{with} proxy constraints. Like the proxy-Lagrangian, this is a non-zero-sum game, but unlike the proxy-Lagrangian, we have no theoretical justification for it. That said, it makes intuitive sense, and works well in practice.
  The $\lambda$-player optimizes based on the proxy-constraints and the $\theta$-player optimizes based on the original constraints.
  The parameter $\Radius$ is the radius of the Lagrange multiplier space
  $\Multipliers \defeq \left\{ \multipliers \in \R_+^{\numconstraints} :
  \norm{\multipliers}_1 \le \Radius \right\}$, and the functions
  $\Pi_{\Parameters}$ and $\Pi_{\Multipliers}$ project their arguments onto
  $\Parameters$ and $\Multipliers$ (respectively) \wrt the Euclidean norm.
  $\{g_i\}_{i=1}^m, \{\widetilde{g_i}\}_{i=1}^m$ are respectively the original constraints and proxy-constraints.
}

\label{alg:proxy-additive-ext}

\end{algorithm*}

Then, for each constrained optimization technique, we show the results for the following four solution types:
\begin{enumerate}
    \item {\bf T-stoch}: the stochastic solution that is the uniform distribution over the $T$ iterates $\theta^{(1)},...,\theta^{(T)}$.
    \item {\bf m-stoch}: the stochastic solution obtained by applying the``shrinking" technique to the  $T$-stoch solution on the training set, which will have support on at most $m+1$ deterministic solutions.
    \item {\bf Last}: the last iterate (i.e. $\theta^{(T)}$).
    \item {\bf Best}: the ``best" iterate out of all $T$ iterates $\theta^{(1)},...,\theta^{(T)}$, where ``best" is chosen by the heuristic given in \eqref{heuristic}  applied on the training set. 
\end{enumerate}

We note that in the non-convex proxy-Lagrangian setting, the 0-1 swap algorithm's $T$-stoch or $m$-stoch solutions come with theoretical guarantees if we replace the SGD with the approximate optimization oracle. In contrast, the 0-1 ext algorithm has no such guarantees, but is simpler. Similarly, in the non-convex setting, the deterministic solutions will not have any guarantees, but are even simpler. 

\subsection{Bank Marketing}


\begin{table*}
  \caption{Bank Marketing Experiment Results}
  \label{tab:bank_marketing}
  \begin{tabular}{lrrrrrr}
 \toprule
Algorithm & Train Err. & Valid Err. & Test Err. & Train Vio. & Valid Vio. &  Test Vio. \\
\midrule
Unconstrained  & 0.0948  & 0.0935 & 0.0937 & 0.0202 &0.0220 & 0.0152 \\
\midrule
Hinge $m$-stoch.& 0.0955 & 0.0954 & 0.0949 & 0 & -0.0008 & -0.0030\\
Hinge $T$-stoch. & 0.1109 & 0.1114 & 0.1121 & -0.0177 & -0.0181 & -0.0179\\
Hinge Best & 0.0964 & 0.0969 & 0.0955 & -0.0032 & -0.0045 & -0.0047\\
Hinge Last  & 0.1122 & 0.1129 & 0.114 & -0.02 & -0.02 & -0.02\\
\midrule
0-1 swap. $m$-stoch. & 0.0939 & 0.0943 & 0.0951 & 0 & -0.0005 & 0.0019\\
0-1 swap. $T$-stoch. & 0.0963 & 0.0955 & 0.0947 & 0.0004 & -0.0003 & -0.0031\\
0-1 swap. Best  & 0.0936 & 0.0935 & 0.0932 & -0.0004 & -0.0009 & -0.0041\\
0-1 swap. Last  & 0.0963 & 0.0957 & 0.0954 & -0.0007 & -0.001 & -0.0035\\
\midrule
0-1 ext. $m$-stoch.   & 0.0946 & 0.0952 & 0.0946 & 0 & -0.001 & -0.0024\\
0-1 ext. $T$-stoch. & 0.1083 & 0.1087 & 0.1085 & -0.0135 & -0.0146 & -0.0139\\
0-1 ext. Best  & 0.0963 & 0.0964 & 0.0953 & -0.0021 & -0.0016 & -0.0056\\
0-1 ext. Last & 0.1029 & 0.1032 & 0.101 & -0.0046 & -0.0072 & -0.0056\\
  \bottomrule
\end{tabular}
\end{table*}

The Bank Marketing UCI benchmark dataset \citep{Lichman:2013} classifier predicts whether someone will sign up for the bank product being marketed.  This dataset was used to test improving statistical parity for a linear model in \citet{Zafar:2015} but with only one protected group based on age. We similarly use a linear model and age as a protected feature, but create $5$ protected groups based on the five training set quantiles of age.  We add a statistical parity rate constraint for each of the five age quantiles with an additive slack of $p = 2\%$:
\begin{equation*}
    p^+(D_k) \leq  p^+(D)) - .02\%,
\end{equation*}
where $D_k$ are the training examples from the $k$th protected group for $k = 1, 2, \ldots, 5$, and $D$ are all the training examples. 

The results can  be found in \tabref{bank_marketing}. We note that the Hinge Last solution is a \emph{degenerate} solution in that it always predicts the a priori more probable class. 


\subsection{Adult}



\begin{table}
  \caption{Adult Experiment Results}
  \label{tab:adult}
  \begin{tabular}{lrrrrrr}
 \toprule
Algorithm & Train Err. & Valid Err. & Test Err. & Train Vio. & Valid Vio. &  Test Vio. \\
\midrule
Unconstrained  & 0.1421  & 0.1348 & 0.1428 & 0.0803 & 0.0604 & 0.0555\\
\midrule
Hinge $m$-stoch.& 0.1431 & 0.1348 & 0.1442 & 0 & -0.0088 & 0.0025\\
Hinge $T$-stoch. & 0.1462 & 0.1394 & 0.1481 & -0.0409 & -0.0372 & -0.0436\\
Hinge Best & 0.1424 & 0.1333 & 0.1447 & -0.028 & -0.0154 & -0.0317\\
Hinge Last & 0.1532 & 0.1490 & 0.1551 & -0.0174 & -0.0217 & -0.0254\\
\midrule
0-1 swap. $m$-stoch. & 0.1431 & 0.1349 & 0.1432 & 0.0176 & 0.0023 & 0.0559\\
0-1 swap. $T$-stoch.  & 0.1428 & 0.1365 & 0.1436 & 0.0054 & 0.0354 & 0.0285\\
0-1 swap. Best & 0.1426 & 0.1354 & 0.1440 & -0.0016 & 0.0140 & 0.0154\\
0-1 swap. Last  & 0.1436 & 0.1358 & 0.1443 & 0.0069 & 0.0248 & 0.0221\\
\midrule
0-1 ext. $m$-stoch. & 0.1418 & 0.1348 & 0.1432 & 0 & -0.0019 & 0.0059\\
0-1 ext. $T$-stoch. & 0.1441 & 0.1369 & 0.1447 & 0.0034 & 0.022 & 0.0174\\
0-1 ext. Best & 0.1420 & 0.1348 & 0.1432 & -0.0374 & -0.0333 & -0.0015\\
0-1 ext. Last & 0.1436 & 0.1358 & 0.1448 & -0.0116 & 0.0078 & 0.0028\\
  \bottomrule

\end{tabular}
\end{table}

We used the benchmark  Adult income UCI dataset~\citep{Lichman:2013}. The goal is to predict whether someone makes more than $50$k per year, and also do well at the \emph{equal opportunity} fairness metric. We used four protected groups: two race-based (Black or White) and two gender-based (Male or Female). We preprocessed the dataset consistent with \citet{Zafar:2015} and \citet{Goh:2016}. \citet{Goh:2016} showed that by explicitly constraining the difference in coverage and using a \emph{linear} model, they could achieve higher $p$ fairness and better accuracy than earlier work using correlation constraints of \citet{Zafar:2015} by up to $0.5\%$ on this dataset.

For these experiments, we added four rate constraints to the training to impose \emph{equal opportunity} at $95\%$, that is for each of the protected groups  (Black, White, Female and Male) the constraints force the
classifier's coverage (the proportion classified positive) on the positively labeled examples for each protected group to be at least $95 \%$ of the overall coverage on the positively labeled examples:
\begin{equation}\label{eq:adultSlack}
    p^+(D_k[y=1]) \geq 0.95\cdot  p^+(D[y=1]),
\end{equation}
where $D_k$ are the training examples from the $k$th protected group for $k = 1, 2, \ldots, 4$, and $D$ are all the training examples. 

We use a linear model. The results can be found in \tabref{adult}.

\subsection{COMPAS}


\begin{table}
  \caption{COMPAS Experiment Results}
  \label{tab:compas}
  \begin{tabular}{lrrrrrr}
 \toprule
Algorithm & Train Err. & Valid Err. & Test Err. & Train Vio. & Valid Vio. &  Test Vio. \\
\midrule
Unconstrained  & 0.3056  &  0.3160 & 0.3109 & 0.1151 & 0.2143 & 0.1082\\
\midrule
Hinge $m$-stoch.  & 0.3711 & 0.3744 & 0.3676 & 0 & 0.0395 & 0.0284\\
Hinge $T$-stoch.  & 0.2880 & 0.3387 & 0.3198 & 0.1093 & 0.1779 & 0.0917\\
Hinge Best & 0.2840 & 0.3322 & 0.3223 & 0.0803 & 0.1262 & 0.0800\\
Hinge Last  & 0.2882 & 0.3322 & 0.3231 & 0.1275 & 0.1968 & 0.0996\\
\midrule
0-1 swap. $m$-stoch. & 0.3132 & 0.3015 & 0.3174 & 0.0004 & 0.0851 & 0.0111\\
0-1 swap. $T$-stoch.  & 0.2968 & 0.3208 & 0.3219 & 0.0257 & 0.1286 & 0.0547\\
0-1 swap. Best   & 0.3009 & 0.3096 & 0.3125 & 0.0281 & 0.1084 & 0.0356\\
0-1 swap. Last   & 0.3023 & 0.3096 & 0.3158 & 0.0412 & 0.1153 & 0.0480\\
\midrule
0-1 ext. $m$-stoch. & 0.3145 & 0.3080 & 0.3146 & 0 & 0.0813 & 0.0147\\
0-1 ext. $T$-stoch.& 0.2990 & 0.3128 & 0.3086 & 0.0323 & 0.1154 & 0.0321\\
0-1 ext. Best & 0.3106 & 0.3160 & 0.3101 & -0.0069 & 0.0797 & -0.0085\\
0-1 ext. Last & 0.2935 & 0.3160 & 0.3125 & 0.0330 & 0.1231 & 0.0325\\
  \bottomrule
\end{tabular}
\end{table}

The positive label in the ProPublica’s COMPAS recidivism data is a prediction the person will re-offend. The goal is to predict recidivism with fairness constraints and we preprocess this dataset in a similar manner as in the Adult dataset and the protected groups are also similar: two race-based (Black and White) and two gender-based (Male and Female). The classifier we use is a $2$ layer neural network with $10$ hidden units.

In this experiment, the goals are quite similar to that of the Adult experiment. Our protected groups are again two races (Black and White) and two genders (Male and Female) and the goal is to constrain equal opportunity such that no group is unfairly getting targeted. However, instead of expressing the constraint with multiplicative slack as in the Adult experiments, we expressed it as an additive slack of $5\%$:
\begin{equation*}
    p^+(D_k[y=1]) \leq  p^+(D[y=1]) + .05\%,
\end{equation*}
where $D_k$ are the training examples from the $k$th protected group for $k = 1, 2, \ldots, 4$, and $D$ are all the training examples.   That is, the positive prediction rate of the positively labeled examples for each protected class can exceed that of the overall dataset by at most $5\%$.

The results are shown in \tabref{compas}.

\subsection{Business Entity Resolution}
In this entity resolution problem from Google, the task is to classify whether a pair of business descriptions describe the same real business. Features include measures of similarity of the two business titles, phone numbers, and so on. We add two types of constraints to the training.  First, the data is global, and for each of the 16 most frequent countries, we imposed a minimum recall rate constraint:
\begin{equation*}
    p^+(D_k[y=1]) \geq 95\%,
\end{equation*}
where $D_k$ are the training examples from the $k$th country for $k = 1, 2, \ldots, 16$. It is also known whether each example is a chain business or not. We impose the same minimum recall rate constraint on chain business examples and non-chain business examples. Additionally, we add an equal accuracy constraint that the accuracy on not-chain businesses should not be worse than the accuracy on chain businesses by more than ten percent, as a proxy to making sure large and small businesses receive similar performance from the model:
\begin{align*}
\frac{c^+(D_{\textrm{notCh}}[y=1]) + c^-(D_{\textrm{notCh}}[y=-1])}{| D_{\textrm{notCh}} |} 
 \geq \frac{ c^+(D_{\textrm{ch}}[y=1]) + c^-(D_{\textrm{ch}}[y=-1])}{|D_{\textrm{ch}}|}  - 0.1,
\end{align*}
where ch is an abbreviation for chain.

We ran this experiment a two-layer neural network, the results are shown in \tabref{bizNN}). In the top row, one sees that the unconstrained model has a very high maximum constraint violations, because it is very difficult to achieve $95\%$ recall for all regions. 

\begin{table*}
  \caption{Business Entity Resolution Experiment Results: 2 Layer NN}
  \label{tab:bizNN}
  \begin{tabular}{lrrrrrr}
 \toprule
Algorithm & Train Err. & Valid Err. & Test Err. & Train Vio. & Valid Vio. &  Test Vio. \\
\midrule
Unconstrained  & 0.1223  & 0.1505 & 0.1520 & 0.1727 & 0.2172 & 0.2357\\
\midrule
Hinge $m$-stoch. & 0.2405 & 0.2509 & 0.2535 & 0 & 0.0341 & 0.0282\\
Hinge $T$-stoch.& 0.3308 & 0.3351 & 0.3446 & -0.0258 & 0.0196 & -0.0082\\
Hinge Best  & 0.2657 & 0.2720 & 0.2786 & -0.0083 & 0.0437 & 0.0026\\
Hinge Last   & 0.2483 & 0.2624 & 0.2617 & -0.0175 & 0.0125 & 0.0421\\
\midrule
0-1 swap. $m$-stoch.   & 0.1751 & 0.1953 & 0.1983 & 0 & 0.0745 & 0.0898\\
0-1 swap. $T$-stoch.  & 0.1506 & 0.1749 & 0.1760 & 0.0950 & 0.1427 & 0.1933\\
0-1 swap. Best  & 0.1407 & 0.1687 & 0.1696 & 0.0681 & 0.1224 & 0.1706\\
0-1 swap. Last   & 0.1699 & 0.1910 & 0.1927 & 0.0252 & 0.0864 & 0.0846\\
\midrule
0-1 ext. $m$-stoch. & 0.1891 & 0.2060 & 0.2063 & 0 & 0.0741 & 0.0752\\
0-1 ext. $T$-stoch.  & 0.1934 & 0.2082 & 0.2092 & 0.0011 & 0.0652 & 0.0770\\
0-1 ext. Best  & 0.1889 & 0.2053 & 0.2049 & 0.0026 & 0.0750 & 0.0750\\
0-1 ext. Last  & 0.1968 & 0.2118 & 0.2130 & 0.0008 & 0.0594 & 0.0750\\
  \bottomrule
\end{tabular}
\end{table*}


\subsection{Thresholding}
For this Google problem, a ranked list of tens or hundreds of business results is given for a specific query (e.g. [coffee near me]), and the task is to threshold the list to return only the results worth showing a user. To do this efficiently in the production setting, a binary classifier decides if the 2nd result is worth keeping, and if its decision is positive, continues down the ranked list, and once a result is classified as not worth keeping all lower-ranked results are discarded. For simplicity, all examples are treated as independent even if they originally came from the same ranked list. We use a $2$ layer neural network with $32$ hidden units as the classifier.

A medium-size labeled set is available with labels that are known to be noisy, and the label noise is not zero-mean and not homogenous across the feature space. That set is broken uniformly and randomly into train/validation/test sets.  

We also have an auxiliary independent set of $1,814$ \emph{steering examples} (see \secref{steering}) which were more carefully labeled by expert labelers, and were actively sampled to pinpoint key types of problems. If one only uses the steering examples (ignoring the noisy labeled data), previous experiments have shown that one can stably achieve a $33\%$ cross-validation error rate on the steering examples. The goal is to have a model that gets that $33\%$ error on the steering examples, but also works as well as possible on the larger noisy data. 

The top row of \tabref{thresholding} shows simply training on the noisy data produces a error rate of $35\%$ on the noisy test data, but violates our goal of $33\%$ error on the steering examples by $3\%$ (that is, it has error rate $36\%$ on the steering examples).

The other extreme of training only on the steering examples is also unsatisfying: as reported in the second row of \tabref{thresholding} that overfits the steering examples and performs poorly on the large (noisy) test set with an error rate of $39\%$, because the steering example set is too small and does not cover the entire feature space. 

For the rest of the rows in \tabref{thresholding}, we train on the noisy data with a minimum accuracy rate constraint for $67\%$ accuracy on the steering examples:  
\begin{align*}
\frac{c^+(D_{\textrm{steering}}[y=1]) + c^-(D_{\textrm{steering}}[y=-1])}{ | D_{\textrm{steering}} | } \geq 0.67.
\end{align*}

All of the  different optimizations find essentially feasible solutions, with many able to achieve the same or better test set performance as the unconstrained training (top row).  


\begin{table*}
  \caption{Thresholding Experiment Results}
  \label{tab:thresholding}
  \begin{tabular}{lrrrr}
 \toprule
Algorithm & Train Err. & Valid Err. & Test Err. & Steering Violation \\
\midrule
Unconstrained  & 0.3595  &  0.3491 & 0.3538 & 0.0316 \\
Unconstrained Trained on Steering & 0.3909 & 0.3924 & 0.3930 & -0.0456 \\
\midrule
Hinge $m$-stoch. & 0.3601 & 0.3512 & 0.3582 & 0  \\
Hinge $T$-stoch. & 0.3635 & 0.3558 & 0.3594 & -0.0037  \\
Hinge Best  & 0.3606 & 0.3509 & 0.3560 & -0.0031 \\
Hinge Last   & 0.3621 & 0.3542 & 0.3594 & -0.0003 \\
\midrule
0-1 swap. $m$-stoch. & 0.3574 & 0.3500 & 0.3557 & -0.0025 \\
0-1 swap. $T$-stoch.  & 0.3593 & 0.3513 & 0.3551 & 0.0010  \\
0-1 swap. Best  & 0.3561 & 0.3484 & 0.3532 & -0.0020  \\
0-1 swap. Last  & 0.3584 & 0.3497 & 0.3543 & -0.0020 \\
\midrule
0-1 ext. $m$-stoch. & 0.3605 & 0.3504 & 0.3568 & -0.0009 \\
0-1 ext. $T$-stoch. & 0.3602 & 0.352 & 0.3553 & 0.0010\\
0-1 ext. Best  & 0.3569 & 0.3486 & 0.3515 & -0.0014 \\
0-1 ext. Last  & 0.3579 & 0.3500 & 0.3539 & -0.0009 \\
  \bottomrule
\end{tabular}
\end{table*}

\subsection{Map Intent}
For this Google problem, the task is to classify whether a query is seeking a result on a map. For example, the query [coffee near me] would be labeled positive, while [coffee health benefits] would be labeled negative. We add ten rate constraints for ten regions that constrain the new model training to be at least as accurate as the production classifier is for each of those ten regions. 

\begin{align*}
\frac{c^+(D_{\textrm{region}}[y=1]) + c^-(D_{\textrm{region}}[y=-1])}{ | D_{\textrm{region}} | }  \geq \kappa_{\textrm{region}},
\end{align*}
where $\kappa_{\textrm{region}}$ is the accuracy of the production classifier for that region. The feature vector $x$ includes ten Bool features that indicate if $x$ belongs to one of these ten regions (some examples do not belong to any of the ten regions).

Thirty-two dense and categorical features are available. We train an RTL model that is an ensemble of 300 lattices, where each lattice acts on 8 of the 32 features, with shared calibrators, and the lattices are interpolated using multilinear interpolation, all implemented using the TensorFlow Lattice package. We enforce monotonicity constraints on 28 of the 32 features, resulting in an additional 148,800 constraints (each one is a linear inequality constraint on a pair of model parameters) applied during training; see \citet{You:2017}, \citet{Canini:2016}, \citet{Gupta:2016} for more technical details.

\begin{table*}
  \caption{Map Intent Experiment Results}
  \label{tab:trigger}
  \begin{tabular}{lrrrrrr}
 \toprule
Algorithm & Train Err. & Valid Err. & Test Err. & Train Vio. & Valid Vio. &  Test Vio. \\
\midrule
Unconstrained  & 0.3093  & 0.3122 & 0.3104 & 0.0187 & 0.0162 & 0.0319 \\
\midrule
Hinge $m$-stoch. & 0.3130 & 0.3129 & 0.3124 & 0.0182 & 0.0176 & 0.0313\\
Hinge $T$-stoch. & 0.3096 & 0.3136 & 0.3106 & 0.0194 & 0.0197 & 0.0210\\
Hinge Best  & 0.3056 & 0.3131 & 0.3104 & 0.0172 & 0.0194 & 0.0247\\
Hinge Last   & 0.3058 & 0.3130 & 0.3099 & 0.0177 & 0.0189 & 0.0220\\
\midrule
0-1 swap. $m$-stoch.  & 0.2949 & 0.3002 & 0.2997 & -0.0003 & 0.0025 & 0.0176\\
0-1 swap. $T$-stoch.   & 0.3004 & 0.3022 & 0.3024 & 0.0022 & 0.0061 & 0.0204\\
0-1 swap. Best  & 0.2949 & 0.3002 & 0.2997 & -0.0003 & 0.0025 & 0.0176\\
0-1 swap. Last  & 0.2953 & 0.3004 & 0.3002 & 0.0013 & 0.0034 & 0.0192\\
\midrule
0-1 ext. $m$-stoch.  & 0.3069 & 0.3115 & 0.3101 & 0.0094 & 0.0144 & 0.0231\\
0-1 ext. $T$-stoch.   & 0.3101 & 0.3121 & 0.3107 & 0.0132 & 0.0157 & 0.0243\\
0-1 ext. Best  & 0.3069 & 0.3115 & 0.3101 & 0.0094 & 0.0144 & 0.0231\\
0-1 ext. Last  & 0.3071 & 0.3111 & 0.3103 & 0.0096 & 0.0140 & 0.0242\\
  \bottomrule
\end{tabular}
\end{table*}

\subsection{Filtering}
For this Google problem, the task is to classify whether a candidate result for a query should be immediately discarded as too irrelevant to be worth showing to a user. For this problem we take as given a base classifier $h$, and the goal is to maximize accuracy with minimal loss-only churn (see \secref{churn} for details). The baseclassifier $h$ was trained as a regression model to minimize mean squared error with respect to a real-valued label on $[-1,1]$, but then used as a classifier by thresholding the model's estimates at 0.0. Here we use the same training data, but we pre-threshold the real-valued training examples to form binary classification labels, then train the new classifier to minimize the classification error rate. We add ten loss-only churn rate constraints to individually restrict the loss-only churn with respect to the production model for each of ten mutually-exclusive geographic regions to less than $5\%$:

\begin{align*}
\frac{c^+(D_{\textrm{region}}[y=-1, h=-1]) + c^-(D_{\textrm{region}}[y=1, h=1])}{ | D_{\textrm{region}}[h=y] | }  \leq 0.05.
\end{align*}
That is, we ask that no more than five percent of the base classifier's wins are lost for each of the ten regions. The feature vector $x$ includes ten Bool features that indicate if $x$ belongs to one of these ten regions (some examples do not belong to any of the ten regions).

Both the production regression model $h$ and the new classifier $f(x)$ use the same model architecture: both are RTL models that are an ensemble of 50 lattices, where each lattice acts on 6 of 16 continuous-valued features, each feature is calibrated by a monotonic piecewise linear transform that is shared across the lattices, the lattices are interpolated using multilinear interpolation, all model parameters trained jointly using the TensorFlow Lattice package.  We enforce monotonicity constraints on 14 of the 16 features, resulting in an additional 9,740 constraints applied during training (each of these is simply a linear inequality constraint on a pair of model parameters); see \citet{You:2017}, \citet{Canini:2016}, \citet{Gupta:2016} for more technical details. 

The production classifier $h$ had a test error rate of $39.72\%$. As hoped, by training specifically for this classification task, the new classifier $f(x)$ achieves lower test error rates: as low as $27.61\%$ for the unconstrained training. However, the high test constraint violation of $32.27\%$ (measured as the maximum violation over the ten regions) shows that the new unconstrained classifier loses a large number of the wins the base classifier had for  at least one of the ten countries considered.

\begin{table*}
  \caption{Filtering Experiment Results}
  \label{tab:filtering}
  \begin{tabular}{lrrrrrr}
 \toprule
Algorithm & Train Err. & Valid Err. & Test Err. & Train Vio. & Valid Vio. &  Test Vio. \\
\midrule
Unconstrained  & 0.2747  & 0.2723 & 0.2761 & 0.3164 & 0.3107 & 0.3227 \\
\midrule
Hinge $m$-stoch. & 0.3363 & 0.3362 & 0.3369 & 0 & -0.0023 & -0.0012\\
Hinge $T$-stoch. & 0.3658 & 0.3656 & 0.3665 & -0.0297 & -0.0262 & -0.0243\\
Hinge Best  & 0.3404 & 0.3403 & 0.3409 & -0.0075 & -0.0080 & -0.0068\\
Hinge Last   & 0.3622 & 0.3618 & 0.3630 & -0.0239 & -0.0239 & -0.0242\\
\midrule
0-1 swap. $m$-stoch.  & 0.3230 & 0.3231 & 0.3239 & 0 & 0.0071 & 0.0130\\
0-1 swap. $T$-stoch.   & 0.3205 & 0.3208 & 0.3217 & 0.0096 & 0.0192 & 0.0227\\
0-1 swap. Best & 0.3175 & 0.3178 & 0.3186 & 0.0081 & 0.0116 & 0.0156\\
0-1 swap. Last  & 0.3185 & 0.3189 & 0.3195 & 0.0112 & 0.0146 & 0.0118\\
\midrule
0-1 ext. $m$-stoch.  & 0.3231 & 0.3234 & 0.3243 & 0 & 0.0048 & 0.0065\\
0-1 ext. $T$-stoch.   & 0.3300 & 0.3302 & 0.3309 & 0.0004 & 0.0008 & 0.0014\\
0-1 ext. Best  & 0.3180 & 0.3179 & 0.3190 & 0.0079 & 0.0116 & 0.0138\\
0-1 ext. Last  & 0.3268 & 0.3272 & 0.3278 & 0.0021 & 0.0055 & 0.0087\\
  \bottomrule
\end{tabular}
\end{table*}

\section{Discussion of Experimental Results}\label{sec:results}
Now that we have presented the experimental results, we return to discuss the experimental and theoretical evidence for and against the hypotheses and questions posed at the beginning of \secref{experiments}.

\subsection{Do Rate Constraints Help in Practice?}
Yes, overall the experiments show rate constraints are are a useful machine learning tool. Let us consider some more specific questions.

\subsubsection{Can we effectively solve the rate-constrained optimization problem?} \label{sec:LastIsBad}
Yes, but the optimization algorithm does matter. Note here we are asking whether the optimization problem is well-solved, and thus we focus on the \emph{training} error and the \emph{training} violation.

The good news is that compared to unconstrained (top row in result tables) the 0-1 swap regret $m$-stochastic optimization (row 6 in result tables) consistently across all experiments did produce lower training constraint violations while still achieving reasonable training error compared with the unconstrained. (Recall that each $m$-stochastic solves a linear program that sparsifies the corresponding $T$-stochastic such that the constraints are exactly satisfied if the $T$-stochastic solution is feasible, so it is by design that the $m$-stochastic solution train constraint violation is exactly $0.0$ for many of the experiments).  For Adult (see \tabref{adult}), the train error is only $.001$ worse, but the train violation drops from $.0803$ to $.0176$. For Bank Marketing (see \tabref{bank_marketing}), the train error is slightly better for 0-1 swap $m$-stochastic, and the train violation drops from $.0202$ to $0.0$. Similarly for COMPAS (see \tabref{compas}), the 0-1 swap $m$-stochastic has slightly higher training error but drops the train constraint violation from 0.1151 to almost zero. For Business Entity Resolution (\tabref{bizNN}), the training error does increase with 0-1 swap $m$-stochastic, but it is a reasonable price to pay in training accuracy for the huge reduction of the worst case equal-accuracy or min-recall constraint violation from 0.1727 to 0.0.   For the Thresholding problem (\tabref{thresholding}), the 0-1 swap $m$-stochastic is again slightly better on training error and effectively reduces the constraint violation to 0.0, and similarly for the Map Intent experiment (\tabref{trigger}), the training error is lower and the training constraint violation is lower. For Filtering (\tabref{filtering}), the training error for 0-1 swap regret $m$-stochastic did go up significantly from $0.2747$ to $0.3230$, but the unconstrained training violation was horrendous at $0.3164$ whereas the $m$-stochastic found a feasible solution.  In conclusion for all experiments run, we found the 0-1 swap regret $m$-stochastic did a good or reasonable job at the optimization problem of minimizing training error and satisfying the constraints on the training set. 

In contrast, one can see that using the baseline strategy of approximating all indicators with the hinge throughout the optimization can provide poor or even worse results than the unconstrained. For example, on the Map Intent experiment (see \tabref{trigger}), the hinge $T$-stochastic solution manages to have slightly both worse training error and worse training constraint violation than the unconstrained.  The other hinge optimizations are also un-compelling in this experiment. In contrast, the swap regret optimizations consistently find good solutions with lower training error and roughly zero training constraint violations. This is a challenging optimization problem because there are ten rate constraints on ten regions of differing sizes. 

The baseline strategy of simply taking the last iterate often does a good job at solving the constrained problem, but sometimes is worse at optimizing the constrained problem than the unconstrained solver!  For example, on COMPAS (see \tabref{compas}) the Hinge Last training violation is actually bigger than the unconstrained training violation. While Hinge Last does achieve slightly better training error, it hasn't achieve better validation error (or test error), so we don't believe this was simply an unlucky validation of hyperparameter choice.  For more details on why last iterate can perform badly, see \secref{LastOscillates}. 

While theory dictates a stochastic solution is necessary for guarantees, in practice the $T$-stochastic solutions can be quite poor, for example the $T$-stochastic solution on Map Intent (\tabref{tab:trigger}) the Hinge $T$-stochastic solution is worse than unconstrained on both training error and training constraint violation. This may be due to bad early iterates, which would be diluted with a longer run time. Compared to the $T$-stochastic solutions, the $m$-stochastic solutions are always better on training error and never more violating, as designed. 

The best iterate is by definition always at least as good as the last iterate on the training error and/or training violation. For all three optimization strategies (hinge, 0-1 swap regret, 0-1 external regret), it manages to consistently produce solutions that are better than the unconstrained in terms of training violations and have reasonable or good training errors. 

\subsubsection{Can we get good test results by training with rate constraints?}
Yes, mostly. The $m$-stochastic and best iterate solutions do result in lower test violations and reasonable test errors for six of the seven experiments. However, for Adult (\tabref{adult}), the 0-1 swap $m$-stochastic failed to produce lower test violation nor lower test error than the unconstrained, despite having much lower training and validation violations. Sadly, the good training and validation performance simply did not generalize to the test set.  This case is hard in part because the Black constraint in the Adult dataset is based on a relatively small sample:  only 345 positive training examples, 42 positive validation examples, and 179 positive test examples.

Overall, small constraint datasets can lead to poor generalization that can significantly hurt the overall metrics. The worst generalization happened with the Business Entity Resolution, where training violations for the proxy-Lagrangian methods ranged from $[0-.095]$, but the test violations ranged from $[0.075-0.19]$. For that experiment, the hinge solutions generalized better, but at the cost of much higher test errors.  Business Entity is a particularly hard problem because there are 16 constraints on different regions, some of which have very small datasets, and just like  training a model, there is a greater risk of poor generalization if the datasets used in the constraints are small.

For the larger datasets (Map Intent in Table \ref{tab:trigger} and Filtering in Table \ref{tab:filtering}), the classifier performance was much more similar on training and test sets. 

For a further discussion of generalization for rate-constraints, with some theoretical results and practical strategies, see \citet{Cotter:generalization}.

\subsubsection{Do rate constraints interact well with other types of constraints?}\label{sec:shapeConstraints}

We did not see any problems from combining rate constraints with shape constraints. For Map Intent (\tabref{trigger}), both 0-1 optimization strategies worked reasonably, with the 0-1 swap regret producing attractive solutions that both notably lowered training error and satisfied the constraints on the traning set.  This shows that the addition of the 148,800 pairwise monotonicity constraints did not cause a problem in optimizing the rate constrained problem. Similarly, for the Filtering (\tabref{filtering}), the addition of the 9,740 sparse linear inequality constraints to enforce monotonicity did not keep the optimizers from satisfying the rate constraints.

\subsection{Does the Proxy-Lagrangian Better Solve the Constrained Optimization Problem?}
We break this question into a few specific questions.

\subsubsection{Does simply using hinge surrogate for both players overconstrain?} 
We hypothesized that using the hinge loss as a convex relaxation to the 0-1 indicators in the rate constraints would cause the constrained optimization to find overly-constrained solutions at the cost of more training accuracy than needed to satisfy the constraints. This was not as large an effect as we expected. However, it can be seen in the Business Entity Resolution (\tabref{bizNN}) experiment where the hinge training violations are negative and the training errors are relatively high, whereas the 0-1 $m$-stochastic solutions crisply achieve the constraint with much lower training errors.

\subsubsection{Does the proxy-Lagrangian formulation result in better solutions?} 
In most experiments, there were trade-offs between test constraint violation and test accuracy which make it difficult to compare the hinge solutions to the proxy-Lagrangian solutions (denoted 0-1 in the tables) on the test metrics. 

On the training metrics, there is stronger evidence the $0-1$ $m$-stochastic optimization is in fact doing a better job solving the optimization problem than the hinge $m$-stochastic. For 7 of the 7 experiments, the $0-1$ ext. $m$-stochastic produced both lower train error and lower train violation than the Hinge $m$-stochastic solution. This was also true for 5 out of the 7 experiments for the $0-1$ swap $m$-stochastic, and the solutions were close for the remaining 2 experiments.

In the case of Map Intent (\tabref{trigger}), we clearly see that the Hinge solutions perform worse in both accuracy and fairness constraints than the 0-1 proxy-Lagrangian procedures on both training and testing. In the case of Thresholding, we see that the Hinge procedures seem to do slightly worse in final accuracy at the cost of over-constraining. We see that in Business Entity Resolution (\tabref{bizNN}), the Hinge procedures attain significantly higher errors than the other methods but do attain better constraint satisfaction on testing.
Thus, even though proxy-Lagrangian formulation may seem better on a few of the datasets, this effect was not seen consistently across the remaining datasets and thus, the question of whether the proxy-Lagrangian attains better solutions in practice remains inconclusive.

\subsubsection{Is minimizing swap regret necessary, or does external regret suffice?}
Our theoretical results show that in the proxy-Lagrangian setting, the appropriate type of equilibrium (i.e. semi-coarse correlated equilibrium) has optimality and feasibility guarantees for the original constrained optimization problem.
In order to attain such an equilibrium, we needed the $\lambda$-player to minimize swap-regret (while the $\theta$-player minimizes the classic external regret). However, minimizing swap-regret involves a more complicated procedure. We used the strategy of \citet{Gordon:2008}, who showed that any external regret minimizing procedure can be turned into one that minimizes swap regret by a meta-algorithm which runs $m$ copies of the procedure. We questioned whether it would be just as good in practice to use the simpler external-regret minimizing procedure, which still leads to a coarse-correlated equilibrium (which is a {\it weaker} notion than semi-coarse correlated equilibrium). 

Comparing the swap regret to the external regret for the same solution type ($m$-stochastic/$T$-stochastic/best/last), the external regret usually ends up with a solution with slightly lower test violations but slightly higher test error. The only exception was the Map Intent experiment in which  the swap-regret solutions were both considerably more accurate and better at satisfying the constraints. 
In conclusion, we have not seen experimental evidence that the extra complexity of swap regret is warranted in practice, though it may sometimes produce notably worse results (as in Map Intent).

\subsection{Do We Really Need Stochastic Classifiers?}
Next, we investigate some specific questions regarding the necessity of stochastic solutions over a deterministic classifier.

\begin{figure}[t]
\includegraphics[width=0.45\textwidth]{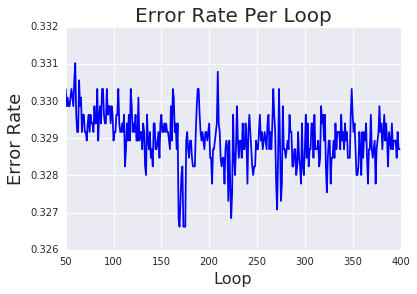}
\includegraphics[width=0.45\textwidth]{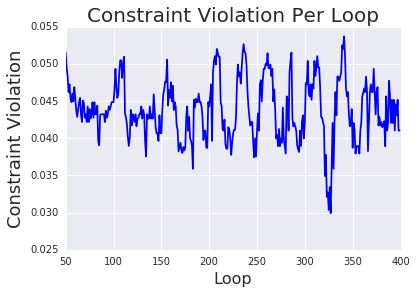}
\centering
\caption{\label{fig:oscillation} The plots for the errors and constraint violations for each iteration during training on the COMPAS dataset. The oscillation due to the conflicting goals of accuracy and constraints suggest that there may be no pure equilibrium to converge to in the non-convex setting.}
\end{figure}

\subsubsection{Do the iterates oscillate in the non-convex setting?} \label{sec:LastOscillates}
As noted in \secref{LastIsBad}, simply taking the last iterate can produce worse constraint violations to the optimization problem then solving the unconstrained problem. \figref{oscillation} plots the error and constraints for each of the iterates on the COMPAS dataset which shows such oscillation. This suggests that, as we showed in \secref{twoplayergame}, the phenomenon of the non-convex Lagrangian having no pure Nash equilibrium to which it can converge, may occur in practice.

\subsubsection{Does $m$-stochastic beat $T$-stochastic?}
Our theoretical results guarantee that the $m$-stochastic solution (which is obtained  through solving a simple LP on the $T$-stochastic iterates) will be no worse than the $T$-stochastic solution by forcing the $m$-stochastic solution to be at least as feasible as the $T$ stochastic solution, while having no worse error (at least on the training set). Our hope is therefore that our ``shrinking'' procedure will find better solutions on \emph{test} data.

We see consistently across datasets as well as optimization techniques that the $m$-stochastic is  indeed better than the $T$-stochastic in  terms of both error and constraint violation on training. Part of this effect may be due to the fact that many of the iterates of the $T$-stochastic perform poorly, for example the early iterates before our procedures are able to get to reasonable solutions. Or during phase-transitions if there is oscillation between satisfying constraints and satisfying error. Fortunately, the shrinking procedure seems to be able to choose a good re-weighting of the $T$-stochastic solution in order to attain well-performing final results. 

We also see that in the vast majority of situations, the test performance for the $m$-stochastic either surpasses that of the $T$-stochastic, or there is an accuracy-fairness trade-off between the two (and hence, not straightforward to compare the two).

\subsubsection{Does the best iterate perform as well as the stochastic classifiers?}

We have already established that a stochastic solution is needed in theory (\secref{twoplayergame}). However, stochastic solutions are unappealing in practice: they take more memory, are harder to test and debug due to their inherent randomness, and a randomized decision may feel less \emph{fair} in certain contexts (even if the outcomes statistically improve the desired fairness metric).   Here, we ask if a stochastic solution is needed \emph{in practice}, based on test metrics. 

First, we compare the 0-1 swap regret $m$-stochastic solution, which is our theoretically preferred stochastic solution, to the 0-1 swap regret best iterate. The 0-1 swap  best iterate is never a strictly worse choice than the 0-1 swap  $m$-stochastic.  In some cases the $m$-stochastic solution puts all or most of its weight on the best iterate---for example, for the Map Intent problem (\tabref{trigger}) the two solutions are identical. In other experiments the solutions differ but both achieve reasonable different trade-offs of test error and test violation, for example on the Thresholding problem (\tabref{thresholding}) and COMPAS (see \tabref{compas}), the best iterate has a lower test error, but a higher test constraint violation. 

Comparing the $m$-stochastic solution and best iterate solution for the hinge optimization and the 0-1 external regret optimization similarly suggests that much of the time the best iterate works just as well in practice.

\subsubsection{Does best iterate perform better in practice than last iterate?}

We have established that using the best iterate works well in practice. Now we discuss how much better best is than simply taking the last iterate. In fact, the last iterate is strictly worse at test metrics than the best iterate for 4 of the 7 experiments: Bank, Thresholding, Adult, and Compass; and the two solutions are similar for the other three experiments. 

If there were truly oscillations seen in practice, then the last iterate could be highly unstable and could produce undesirable solutions.  In practice, the strongest evidence for last being a risky choice is Hinge Last on COMPAS, where test error went up from 0.3109 to 0.3231, and training violation only went down from 0.1082 to 0.0996.

Overall, the experimental results suggest that the best iterate should be preferable to the last iterate.


\section{Conclusions, Advice to Practitioners, and Open Questions} \label{sec:conclusions}
In this paper, we provide the most comprehensive study to-date of training classifiers with a broad array of rate constraints, with new theoretical, algorithmic, and experimental results as well as practical insights and guidance for using rate constraints to solve real-world problems.   Next, we provide some conclusions, specifically draw out our best advice to practitioners, and note some open questions. 

\subsection{Advice to Practitioners: How To Train Classifiers with Rate Constraints}

Based on our experiments, our advice to practitioners is to optimize the rate-constrained training using either our proposed proxy-Lagrangian formulation (0-1 swap regret $m$-stochastic), \emph{or} the easier alternative of optimizing a non-zero-sum variant of the normal Lagrangian formulation (0-1 external regret best iterate). 

The 0-1 external regret best iterate optimization procedure is simple: when optimizing the model parameters $\theta$ use stochastic gradient descent as usual with a hinge relaxation of the indicators in the constraints, and when optimizing the Lagrange multipliers $\lambda$ use stochastic gradient descent, but do not relax the indicators in the rate constraints. One downside though is the \emph{best iterate} requires storing all the candidate iterates on the Pareto Frontier during training, in order to rank them by the training objective and training error at the end, and in the worst case that could be all candidate iterates. However one can control the number of candidate iterates. Simply taking the last iterate may also yield reasonable results, but we saw a number of situations where the last iterate performed strictly worse under all metrics compared to the best iterate. 

We caution against relaxing the indicators for both the $\theta$-player and $\lambda$-player. It is hardly simpler than the 0-1 external regret best iterate optimization, and experimentally generally (but not always) produced worse test results, sometimes notably worse. 

\subsection{More Experimental Conclusions}
The clearest experimental finding is that treating the optimization as a non-zero-sum two-player game where the $\lambda$-player does not relax the indicators in the rate constraints (notated as 0-1 in the experimental tables) does generally help, both in finding a better solution to the optimization problem (i.e. train metrics), and in practice (i.e. test metrics).  Another fairly clear experimental finding is that the $T$-stochastic solution can effectively be sparsified to an $m$-stochastic solution, generally with improved metrics. 

While the $T$-stochastic solution has better theoretical guarantees than any of our deterministic solutions, especially for large $T$, in practice we found the deterministic best iterate generally worked better than the $T$-stochastic solution.  Other comparisons were more cloudy, see Section \ref{sec:results} for details. 

\subsection{Advice to Practitioners: Plan To Overfit the Constraints}

A key issue with using rate constraints is \emph{generalization}: satisfying
the constraints on the training examples does not necessarily mean
that they will be satisfied on new test sets, and the generalization may be worse if the test examples are drawn from a different distribution. In expressing the rate constraints, one should add in some slack to account for generalization issues, especially if the constraints are optimized on small datasets.

\subsection{Open Questions on Generalization}

In some specific applications, the post-training correction approach using a separate \emph{validation} dataset of \citet{Woodworth:2017} can improve generalization performance.
More generally, in \citet{Cotter:generalization}, we extend the ideas of this paper with a focus on generalization. We show that providing different \emph{datasets} to the two players, instead of (or in addition to) different constraint functions, can theoretically and practically improve generalization. 

\subsection{Advice to Practitioners: How the Constraints Are Specified Matters}

Though not explored in this paper's experiments, we have learned that in practice, how one specifies the datasets and slack in a rate constraint are very important -  see Section \ref{sec:bestpractices} for details.

\subsection{Open Questions on Nonlinear Rate Constraints}

We have limited our focus to rate constraints that can be written as in \eqref{rateConstraint}, that is linear non-negative combinations of the positive and negative classification rates on datasets. It remains an open question how well the presented techniques would work for nonlinear rate constraints both theoretically and experimentally, and whether other strategies would be needed. We touched on these issues in Section \ref{sec:goals} in our discussion of win-loss ratio and precision, but did not present experimental results with nonlinear rate constraints.  

\subsection{Some Open Theoretical Questions}

One open question is how tight our optimality and feasibility guarantees are for our procedures in the following aspects:
\begin{itemize}
    \item The dependence on the number of iterations $T$ for our guarantees is $O\left(\sqrt{\frac{1}{T}}\right)$. This rate is an artifact of our usage of regret-minimization procedures, but it could be improved through a number of possible techniques, such as variance reduction~\egcite{Johnson:2013}, or by making stronger assumptions (\eg strong convexity and/or smoothness).
    \item The dependence on $m$, the number of constraints, is  $O(\sqrt{m\log m})$, which also comes from the regret-minimization procedures. This is because the $\lambda$-player essentially chooses a distribution over $m+1$ actions and this dependence on the number of arms is tight in the context of regret-minimization, but the question remains of whether there are situations where this could be improved upon for constrained optimization for either feasibility or optimality.
    \item Our results also have a dependence on the model complexity in both feasibility and optimality guarantees. This may be undesirable in models with a large number of parameters, such as modern neural networks. We explored the question of whether we can improve upon this dependence further in follow-up work of \cite{Cotter:generalization}, which improves the feasibility guarantee. However, further investigation is required to either establish matching lower bounds and/or obtaining tighter results.
\end{itemize}


\bibliographystyle{plainnat}
\bibliography{main}

\newpage
\clearpage
\appendix
\showproofstrue

\section{Proofs of Sub\{optimality,feasibility\} Guarantees}\label{app:suboptimality}

\begin{thm}{lagrangian-suboptimality}
  \textbf{(Lagrangian Sub\{optimality,feasibility\})}
  Define $\Multipliers = \left\{ \multipliers \in \R_+^{\numconstraints} :
  \norm{\multipliers}_p \le \Radius \right\}$, and consider the Lagrangian of
  \eqref{constrained-problem} (\eqref{lagrangian}).
  Suppose that $\parameters \in \Parameters$ and $\multipliers \in
  \Multipliers$ are random variables such that:
  \begin{equation}
    \replabel{eq:thm:lagrangian-suboptimality:epsilon}
    \max_{\multipliers^* \in \Multipliers} \expectation_{\parameters}\left[
    \lagrangian\left( \parameters, \multipliers^* \right) \right] -
    \inf_{\parameters^* \in \Parameters} \expectation_{\multipliers}\left[
    \lagrangian\left( \parameters^*, \multipliers \right) \right] \le \epsilon
  \end{equation}
  \ie $\parameters,\multipliers$ is an $\epsilon$-approximate Nash equilibrium.
  Then $\parameters$ is $\epsilon$-suboptimal:
  \begin{equation*}
    \expectation_{\parameters}\left[ \objective\left(\parameters\right) \right]
    \le \inf_{\parameters^* \in \Parameters : \forall i \in
    \indices{\numconstraints} . \constraint{i}\left(\parameters^*\right) \le 0}
    \objective\left(\parameters^*\right) + \epsilon
  \end{equation*}
  Furthermore, if $\multipliers$ is in the interior of $\Multipliers$, in the
  sense that $\norm{\bar{\multipliers}}_p < \Radius$ where $\bar{\multipliers}
  \defeq \expectation_{\multipliers}\left[ \multipliers \right]$, then
  $\parameters$ is $\epsilon / \left(\Radius -
  \norm{\bar{\multipliers}}_p\right)$-feasible:
  \begin{equation*}
    \norm{\left( \expectation_{\parameters} \left[
    \constraint{:}\left(\parameters\right) \right] \right)_+}_q \le
    \frac{\epsilon}{\Radius - \norm{\bar{\multipliers}}_p}
  \end{equation*}
  where $\constraint{:}\left(\parameters\right)$ is the
  $\numconstraints$-dimensional vector of constraint evaluations, and
  $\left(\cdot\right)_+$ takes the positive part of its argument, so that
  $\norm{\left( \expectation_{\parameters} \left[
  \constraint{:}\left(\parameters\right) \right] \right)_+}_q$ is the $q$-norm
  of the vector of expected constraint violations.
\end{thm}
\begin{prf}{lagrangian-suboptimality}
  First notice that $\lagrangian$ is linear in $\multipliers$, so:
  \begin{equation}
    \label{eq:thm:lagrangian-suboptimality:equilibrium}
    \max_{\multipliers^* \in \Multipliers} \expectation_{\parameters}\left[
    \lagrangian\left( \parameters, \multipliers^* \right) \right] -
    \inf_{\parameters^* \in \Parameters} \lagrangian\left( \parameters^*,
    \bar{\multipliers} \right) \le \epsilon
  \end{equation}

  \begin{titled-paragraph}{Optimality}
    Choose $\parameters^*$ to be the optimal \emph{feasible} solution in
    \eqref{thm:lagrangian-suboptimality:equilibrium}, so that
    $\constraint{i}\left(\parameters^*\right) \le 0$ for all
    $i\in\indices{\numconstraints}$, and also choose $\multipliers^* = 0$, which
    combined with the definition of $\lagrangian$
    (\eqref{lagrangian}) gives that:
    \begin{equation*}
      \expectation_{\parameters}\left[ \objective\left(\parameters\right) \right]
      - \objective\left(\parameters^*\right) \le \epsilon
    \end{equation*}
    which is the optimality claim.
  \end{titled-paragraph}

  \begin{titled-paragraph}{Feasibility}
    Choose $\parameters^* = \parameters$ in
    \eqref{thm:lagrangian-suboptimality:equilibrium}. By the definition of $\lagrangian$
    (\eqref{lagrangian}):
    \begin{equation*}
      \max_{\multipliers^* \in \Multipliers} \sum_{i=1}^{\numconstraints}
      \multipliers_i^* \expectation_{\parameters} \left[
      \constraint{i}\left(\parameters\right) \right] -
      \sum_{i=1}^{\numconstraints} \bar{\multipliers}_i
      \expectation_{\parameters} \left[ \constraint{i}\left(\parameters\right)
      \right] \le \epsilon
    \end{equation*}
    Then by the definition of a dual norm, H\"older's inequality, and the
    assumption that $\norm{\bar{\multipliers}}_p < \Radius$:
    \begin{equation*}
      \Radius \norm{\left( \expectation_{\parameters} \left[
      \constraint{:}\left(\parameters\right) \right] \right)_+}_q - \norm{\bar{\multipliers}}_p
      \norm{ \left( \expectation_{\parameters} \left[
      \constraint{:}\left(\parameters\right) \right] \right)_+ }_q \le \epsilon
    \end{equation*}
    Rearranging terms gives the feasibility claim.
  \end{titled-paragraph}
\end{prf}

\begin{lem}{lagrangian-feasibility}
  In the context of \thmref{lagrangian-suboptimality}, suppose that there
  exists a $\parameters' \in \Parameters$ that satisfies all of the
  constraints, and does so with $q$-norm margin $\margin$, \ie
  $\constraint{i}\left(\parameters'\right) \le 0$ for all $i \in
  \indices{\numconstraints}$ and
  $\norm{\constraint{:}\left(\parameters'\right)}_q \ge \margin$. Then:
  \begin{equation*}
    \norm{ \bar{\multipliers} }_p \le \frac{\epsilon +
    \bound{\objective}}{\gamma}
  \end{equation*}
  where $\bound{\objective} \ge \sup_{\parameters \in \Parameters}
  \objective\left(\parameters\right) - \inf_{\parameters \in \Parameters}
  \objective\left(\parameters\right)$ is a bound on the range of the objective
  function $\objective$.
\end{lem}
\begin{prf}{lagrangian-feasibility}
  Starting from \eqref{thm:lagrangian-suboptimality:epsilon} (in
  \thmref{lagrangian-suboptimality}), and choosing $\parameters^* =
  \parameters'$ and $\multipliers^* = 0$:
  \begin{align*}
    \epsilon \ge& \expectation_{\parameters}\left[
    \objective\left(\parameters\right) \right] -
    \expectation_{\multipliers}\left[ \objective\left(\parameters'\right) +
    \sum_{i=1}^{\numconstraints} \multipliers_i
    \constraint{i}\left(\parameters'\right) \right] \\
    \epsilon \ge& \expectation_{\parameters}\left[
    \objective\left(\parameters\right) - \inf_{\parameters' \in \Parameters}
    \objective\left(\parameters'\right) \right] - \left(
    \objective\left(\parameters'\right) - \inf_{\parameters' \in \Parameters}
    \objective\left(\parameters'\right) \right) + \gamma \norm{
    \bar{\multipliers} }_p \\
    \epsilon \ge& - \bound{\objective} + \gamma \norm{ \bar{\multipliers} }_p
  \end{align*}
  Solving for $\norm{ \bar{\multipliers} }_p$ yields the claim.
\end{prf}

\begin{thm}{proxy-lagrangian-suboptimality}
  \textbf{(Proxy-Lagrangian Sub\{optimality,feasibility\})}
  Let 
  $$\Matrixmultipliers \defeq \left\{ \matrixmultipliers \in
  \R^{\left(m+1\right)\times\left(m+1\right)} : \forall i \in
  \indices{\numconstraints + 1} . \matrixmultipliers_{:, i} \in
  \Delta^{\numconstraints + 1} \right\}$$ be the set of all left-stochastic
  $\left(\numconstraints + 1\right) \times \left(\numconstraints + 1\right)$
  matrices, and consider the ``proxy-Lagrangians'' of
  \eqref{constrained-problem} (\eqref{proxy-lagrangian}).
  Suppose that $\parameters \in \Parameters$ and $\multipliers \in
  \Multipliers$ are jointly distributed random variables such that:
  \begin{align}
    \replabel{eq:thm:proxy-lagrangian-suboptimality:equilibrium}
    \expectation_{\parameters,\multipliers}\left[
    \lagrangian_{\parameters}\left( \parameters, \multipliers \right) \right] -
    \inf_{\parameters^* \in \Parameters} \expectation_{\multipliers}\left[
    \lagrangian_{\parameters}\left( \parameters^*, \multipliers \right) \right]
    \le& \epsilon_{\parameters} \\
    \notag \max_{\matrixmultipliers^* \in \Matrixmultipliers}
    \expectation_{\parameters,\multipliers}\left[
    \lagrangian_{\multipliers}\left( \parameters, \matrixmultipliers^*
    \multipliers \right) \right] -
    \expectation_{\parameters,\multipliers}\left[
    \lagrangian_{\multipliers}\left( \parameters, \multipliers \right) \right]
    \le& \epsilon_{\multipliers}
  \end{align}
  \TODO{this is cumbersome}
  Define $\bar{\multipliers} \defeq \expectation_{\multipliers}\left[
  \multipliers \right]$, let $\left( \Omega, \mathcal{F}, P \right)$ be the
  probability space, and define a random variable $\bar{\parameters}$ such
  that:
  \begin{equation*}
    \probability\left\{ \bar{\parameters} \in S \right\} = \frac{
    \int_{\parameters^{-1}\left(S\right)} \multipliers_1\left(x\right)
    dP\left(x\right) }{ \int_{\Omega} \multipliers_1\left(x\right)
    dP\left(x\right) }
  \end{equation*}
  In words, $\bar{\parameters}$ is a version of $\parameters$ that has been
  resampled with $\multipliers_1$ being treated as an importance weight. In
  particular $\expectation_{\bar{\parameters}}\left[
  f\left(\bar{\parameters}\right) \right] =
  \expectation_{\parameters,\multipliers}\left[ \multipliers_1
  f\left(\parameters\right) \right] / \bar{\multipliers}_1$ for any
  $f:\Parameters \rightarrow \R$.
  Then $\bar{\parameters}$ is nearly-optimal:
  \begin{equation*}
    \expectation_{\bar{\parameters}}\left[
    \objective\left(\bar{\parameters}\right) \right] \le \inf_{\parameters^*
    \in \Parameters : \forall i \in \indices{\numconstraints} .
    \proxyconstraint{i}\left(\parameters^*\right) \le 0} \objective\left(
    \parameters^* \right) + \frac{\epsilon_{\parameters} +
    \epsilon_{\multipliers}}{\bar{\multipliers}_1}
  \end{equation*}
  and nearly-feasible:
  \begin{equation*}
    \norm{ \left( \expectation_{\bar{\parameters}}\left[
    \constraint{:}\left(\bar{\parameters}\right) \right] \right)_+ }_{\infty}
    \le \frac{\epsilon_{\multipliers}}{\bar{\multipliers}_1}
  \end{equation*}
  %
  %
  Notice the optimality inequality is weaker than it may appear, since the
  comparator in this equation is \emph{not} the optimal solution \wrt the
  constraints $\constraint{i}$, but rather \wrt the \emph{proxy} constraints
  $\proxyconstraint{i}$.
\end{thm}
\begin{prf}{proxy-lagrangian-suboptimality}
  \begin{titled-paragraph}{Optimality}
    If we choose $M^*$ to be the matrix with its first row being all-one, and
    all other rows being all-zero, then $\lagrangian_{\multipliers}\left(
    \parameters, \matrixmultipliers^* \multipliers \right) = 0$, which shows
    that the first term in the LHS of the second line of
    \eqref{thm:proxy-lagrangian-suboptimality:equilibrium} is nonnegative.
    Hence, $- \expectation_{\parameters,\multipliers}\left[
    \lagrangian_{\multipliers}\left( \parameters, \multipliers \right) \right]
    \le \epsilon_{\multipliers}$, so by the definition of
    $\lagrangian_{\multipliers}$ (\eqref{proxy-lagrangian}), and the fact
    that $\proxyconstraint{i} \ge \constraint{i}$:
    \begin{equation*}
      \expectation_{\parameters,\multipliers}\left[
      \sum_{i=1}^{\numconstraints} \multipliers_{i+1}
      \proxyconstraint{i}\left(\parameters\right) \right] \ge
      -\epsilon_{\multipliers}
    \end{equation*}
    Notice that $\lagrangian_{\parameters}$ is linear in $\multipliers$, so the
    first line of \eqref{thm:proxy-lagrangian-suboptimality:equilibrium},
    combined with the above result and the definition of
    $\lagrangian_{\parameters}$ (\eqref{proxy-lagrangian}) becomes:
    \begin{equation}
      \label{eq:thm:proxy-lagrangian-suboptimality:both-epsilons}
      \expectation_{\parameters,\multipliers}\left[ \multipliers_1
      \objective\left(\parameters\right) \right] - \inf_{\parameters^* \in
      \Parameters} \left( \bar{\multipliers}_1
      \objective\left(\parameters^*\right) + \sum_{i=1}^{\numconstraints}
      \bar{\multipliers}_{i+1} \proxyconstraint{i}\left(\parameters^*\right)
      \right) \le \epsilon_{\parameters} + \epsilon_{\multipliers}
    \end{equation}
    Choose $\parameters^*$ to be the optimal solution that satisfies the
    \emph{proxy} constraints $\tilde{g}$, so that
    $\proxyconstraint{i}\left(\parameters^*\right) \le 0$ for all
    $i\in\indices{\numconstraints}$. Hence:
    \begin{equation*}
      \expectation_{\parameters,\multipliers}\left[ \multipliers_1
      \objective\left(\parameters\right) \right] - \bar{\multipliers}_1
      \objective\left( \parameters^* \right) \le \epsilon_{\parameters} +
      \epsilon_{\multipliers}
    \end{equation*}
    which is the optimality claim.
  \end{titled-paragraph}

  \begin{titled-paragraph}{Feasibility}
    We'll simplify our notation by defining $\ell_1\left(\parameters\right)
    \defeq 0$ and $\ell_{i+1}\left(\parameters\right) \defeq
    \constraint{i}\left(\parameters\right)$ for
    $i\in\indices{\numconstraints}$, so that
    $\lagrangian_{\multipliers}\left(\parameters, \multipliers\right) =
    \inner{\multipliers}{\ell_{:}\left(\parameters\right)}$.
    Consider the first term in the LHS of the second line of 
    \eqref{thm:proxy-lagrangian-suboptimality:equilibrium}:
    \begin{align*}
      \max_{\matrixmultipliers^* \in \Matrixmultipliers}
      \expectation_{\parameters,\multipliers}\left[
      \lagrangian_{\multipliers}\left( \parameters, \matrixmultipliers^*
      \multipliers \right) \right] =&
      \max_{\matrixmultipliers^* \in \Matrixmultipliers}
      \expectation_{\parameters,\multipliers}\left[
      \inner{\matrixmultipliers^* \multipliers}{\ell_{:}\left(\parameters\right)}
      \right] \\
      =& \max_{\matrixmultipliers^* \in \Matrixmultipliers}
      \expectation_{\parameters,\multipliers}\left[
      \sum_{i=1}^{\numconstraints+1}
      \sum_{j=1}^{\numconstraints+1}
      \matrixmultipliers^*_{j,i} \multipliers_i \ell_j\left(\parameters\right)
      \right] \\
      =&
      \sum_{i=1}^{\numconstraints+1}
      \max_{\matrixmultipliers^*_{:,i} \in \Delta^{\numconstraints+1}}
      \sum_{j=1}^{\numconstraints+1}
      \expectation_{\parameters,\multipliers}\left[
      \matrixmultipliers^*_{j,i} \multipliers_i \ell_j\left(\parameters\right)
      \right] \\
      =&
      \sum_{i=1}^{\numconstraints+1}
      \max_{j \in \indices{\numconstraints+1}}
      \expectation_{\parameters,\multipliers}\left[
      \multipliers_i \ell_j\left(\parameters\right)
      \right] \\
    \end{align*}
    where we used the fact that, since $\matrixmultipliers^*$ is
    left-stochastic, each of its columns is a
    $\left(\numconstraints+1\right)$-dimensional multinoulli distribution.
    For the second term in the LHS of the second line of
    \eqref{thm:proxy-lagrangian-suboptimality:equilibrium}, we can use the fact
    that $\ell_1\left(\parameters\right) = 0$:
    \begin{equation*}
      \expectation_{\parameters,\multipliers}\left[
      \sum_{i=2}^{\numconstraints+1} \multipliers_i
      \ell_i\left(\parameters\right) \right] \le \sum_{i=2}^{\numconstraints+1}
      \max_{j \in \indices{\numconstraints+1}}
      \expectation_{\parameters,\multipliers}\left[ \multipliers_i
      \ell_j\left(\parameters\right) \right]
    \end{equation*}
    Plugging these two results into the second line of
    \eqref{thm:proxy-lagrangian-suboptimality:equilibrium}, the two sums
    collapse, leaving:
    \begin{equation*}
      \max_{i \in \indices{\numconstraints+1}}
      \expectation_{\parameters,\multipliers}\left[ \multipliers_1
      \ell_i\left(\parameters\right) \right] \le \epsilon_{\multipliers}
    \end{equation*}
    By the definition of $\ell_i$, and the fact that $\ell_1 = 0$:
    \begin{equation*}
      \norm{ \left( \expectation_{\parameters,\multipliers}\left[
      \multipliers_1 \constraint{:}\left(\parameters\right) \right] \right)_+
      }_{\infty} \le \epsilon_{\multipliers}
    \end{equation*}
    which is the feasibility claim.
  \end{titled-paragraph}
\end{prf}

\begin{lem}{proxy-lagrangian-feasibility}
  In the context of \thmref{proxy-lagrangian-suboptimality}, suppose that there
  exists a $\parameters' \in \Parameters$ that satisfies all of the
  \emph{proxy} constraints with margin $\margin$, \ie
  $\proxyconstraint{i}\left(\parameters'\right) \le -\margin$ for all $i \in
  \indices{\numconstraints}$. Then:
  \begin{equation*}
    \bar{\multipliers}_1 \ge \frac{\gamma - \epsilon_{\parameters} -
    \epsilon_{\multipliers}}{\gamma + \bound{\objective}}
  \end{equation*}
  where $\bound{\objective} \ge \sup_{\parameters \in \Parameters}
  \objective\left(\parameters\right) - \inf_{\parameters \in \Parameters}
  \objective\left(\parameters\right)$ is a bound on the range of the objective
  function $\objective$.
\end{lem}
\begin{prf}{proxy-lagrangian-feasibility}
  Starting from \eqref{thm:proxy-lagrangian-suboptimality:both-epsilons} (in
  the proof of \thmref{proxy-lagrangian-suboptimality}), and choosing
  $\parameters^* = \parameters'$:
  \begin{equation*}
    \expectation_{\parameters,\multipliers}\left[ \multipliers_1
    \objective\left(\parameters\right) \right] - \left( \bar{\multipliers}_1
    \objective\left(\parameters'\right) + \sum_{i=1}^{\numconstraints}
    \bar{\multipliers}_{i+1} \proxyconstraint{i}\left(\parameters'\right)
    \right) \le \epsilon_{\parameters} + \epsilon_{\multipliers}
  \end{equation*}
  Since $\proxyconstraint{i}\left(\parameters'\right) \le -\margin$ for all $i
  \in \indices{\numconstraints}$:
  \begin{align*}
    \epsilon_{\parameters} + \epsilon_{\multipliers} \ge &
    \expectation_{\parameters,\multipliers}\left[ \multipliers_1
    \objective\left(\parameters\right) \right] - \bar{\multipliers}_1
    \objective\left(\parameters'\right) + \left(1 - \bar{\multipliers}_1\right)
    \gamma \\
    \ge & \expectation_{\parameters,\multipliers}\left[ \multipliers_1 \left(
    \objective\left(\parameters\right) - \inf_{\parameters' \in \Parameters}
    \objective\left(\parameters'\right) \right) \right] - \bar{\multipliers}_1
    \left( \objective\left(\parameters'\right) - \inf_{\parameters' \in
    \Parameters} \objective\left(\parameters'\right) \right) + \left(1 -
    \bar{\multipliers}_1\right) \gamma \\
    \ge & -\bar{\multipliers}_1 \bound{\objective} + \left( 1 -
    \bar{\multipliers}_1 \right) \gamma
  \end{align*}
  Solving for $\bar{\lambda}_1$ yields the claim.
\end{prf}

\section{Proofs of Existence of Sparse Equilibria}\label{app:sparsity}

\begin{thm}{sparse-equilibrium}
  Consider a two player game, played on the compact Hausdorff spaces
  $\Parameters$ and $\Multipliers \subseteq \R^{\numconstraints}$.
  Imagine that the $\parameters$-player wishes to minimize
  $\lagrangian_{\parameters} : \Parameters \times \Multipliers \rightarrow \R$,
  and the $\multipliers$-player wishes to maximize $\lagrangian_{\multipliers}
  : \Parameters \times \Multipliers \rightarrow \R$, with both of these
  functions being continuous in $\parameters$ and linear in $\multipliers$.
  Then there exists a Nash equilibrium $\parameters$, $\multipliers$:
  \begin{align*}
    \expectation_{\parameters}\left[
    \lagrangian_{\parameters}\left(\parameters, \multipliers\right) \right] =&
    \min_{\parameters^* \in \Parameters}
    \lagrangian_{\parameters}\left(\parameters^*, \multipliers\right) \\
    \expectation_{\parameters}\left[
    \lagrangian_{\multipliers}\left(\parameters, \multipliers\right) \right] =&
    \max_{\multipliers^* \in \Multipliers} \expectation_{\parameters}\left[
    \lagrangian_{\multipliers}\left(\parameters, \multipliers^*\right) \right]
  \end{align*}
  where $\parameters$ is a random variable placing nonzero probability mass on
  \emph{at most} $\numconstraints+1$ elements of $\Parameters$, and
  $\multipliers \in \Multipliers$ is non-random.
\end{thm}
\begin{prf}{sparse-equilibrium}
  There are some extremely similar (and in some ways more general) results than
  this in the game theory
  literature~\egcite{Bohnenblust:1950,Parthasarathy:1975}, but for our
  particular (Lagrangian and proxy-Lagrangian) setting it's possible to provide
  a fairly straightforward proof.

  To begin with, \citet{Glicksburg:1952} gives that there exists a mixed
  strategy in the form of two random variables $\tilde{\parameters}$ and
  $\tilde{\multipliers}$:
  \begin{align*}
    \expectation_{\tilde{\parameters},\tilde{\multipliers}}\left[
    \lagrangian_{\parameters}\left(\tilde{\parameters},
    \tilde{\multipliers}\right) \right] =& \min_{\parameters^* \in \Parameters}
    \expectation_{\tilde{\multipliers}}\left[
    \lagrangian_{\parameters}\left(\parameters^*, \tilde{\multipliers}\right)
    \right] \\
    \expectation_{\tilde{\parameters},\tilde{\multipliers}}\left[
    \lagrangian_{\multipliers}\left(\tilde{\parameters},
    \tilde{\multipliers}\right) \right] =& \max_{\multipliers^* \in
    \Multipliers} \expectation_{\tilde{\parameters}}\left[
    \lagrangian_{\multipliers}\left(\tilde{\parameters}, \multipliers^*\right)
    \right]
  \end{align*}
  Since both functions are linear in $\tilde{\multipliers}$, we can define
  $\multipliers \defeq
  \expectation_{\tilde{\multipliers}}\left[\tilde{\multipliers}\right]$, and
  these conditions become:
  \begin{align*}
    \expectation_{\tilde{\parameters}}\left[
    \lagrangian_{\parameters}\left(\tilde{\parameters}, \multipliers\right)
    \right] =& \min_{\parameters^* \in \Parameters}
    \lagrangian_{\parameters}\left(\parameters^*, \multipliers\right) \defeq
    \ell_{\min} \\
    \expectation_{\tilde{\parameters}}\left[
    \lagrangian_{\multipliers}\left(\tilde{\parameters}, \multipliers\right)
    \right] =& \max_{\multipliers^* \in \Multipliers}
    \expectation_{\tilde{\parameters}}\left[
    \lagrangian_{\multipliers}\left(\tilde{\parameters}, \multipliers^*\right)
    \right]
  \end{align*}
  Let's focus on the first condition. Let $p_{\epsilon} \defeq
  \probability\left\{ \lagrangian_{\parameters}\left(\tilde{\parameters},
  \multipliers\right) \ge \ell_{\min} + \epsilon \right\}$, and notice that
  $p_{1/n}$ must equal zero for any $n \in \left\{1,2,\dots\right\}$ (otherwise
  we would contradict the above), implying by the countable additivity of
  measures that $\probability\left\{
    \lagrangian_{\parameters}\left(\tilde{\parameters}, \multipliers\right) =
    \ell_{\min} \right\} = 1$.
  We therefore assume henceforth, without loss of generality, that the support
  of $\tilde{\parameters}$ consists entirely of minimizers of
  $\lagrangian_{\parameters}\left(\cdot,\multipliers\right)$. Let $S \subseteq
  \Parameters$ be this support set.

  Define $G \defeq \left\{ \grad_{\tilde{\multipliers}}
  \lagrangian_{\multipliers}\left( \parameters', \multipliers \right) :
  \parameters' \in S \right\}$, and take $\bar{G}$ to be the closure of the
  convex hull of $G$.
  Since $\expectation_{\tilde{\parameters}}\left[ \grad_{\tilde{\multipliers}}
  \lagrangian_{\multipliers}\left(\tilde{\parameters},\multipliers\right)
  \right] \in \bar{G} \subseteq \R^{\numconstraints}$, we can write it as
  a convex combination of at most $\numconstraints+1$ extreme points of
  $\bar{G}$, or equivalently of $\numconstraints+1$ elements of $G$.
  Hence, we can take $\parameters$ to be a discrete random variable that places
  nonzero mass on at most $\numconstraints+1$ elements of $S$, and:
  \begin{equation*}
    \expectation_{\parameters}\left[ \grad_{\tilde{\multipliers}}
    \lagrangian_{\multipliers}\left(\parameters,\multipliers\right) \right] =
    \expectation_{\tilde{\parameters}}\left[ \grad_{\tilde{\multipliers}}
    \lagrangian_{\multipliers}\left(\tilde{\parameters},\multipliers\right)
    \right]
  \end{equation*}
  Linearity in $\lambda$ then implies that $\expectation_{\parameters}\left[
  \lagrangian_{\multipliers}\left(\parameters,\cdot\right) \right]$ and
  $\expectation_{\tilde{\parameters}}\left[
  \lagrangian_{\multipliers}\left(\tilde{\parameters},\cdot\right) \right]$ are
  the \emph{same function} (up to a constant), and therefore have the same
  maximizer(s). Correspondingly, $\parameters$ is supported on $S$, which
  contains only minimizers of
  $\lagrangian_{\parameters}\left(\cdot,\multipliers\right)$ by construction.
\end{prf}

\begin{lem}{sparse-proxy-lagrangian}
  \NOTE{It would be nice to avoid the assumption of continuity of the proxy
  constraint functions}
  If $\Parameters$ is a compact Hausdorff space and the objective, constraint
  and proxy constraint functions
  $\objective,\constraint{1},\dots,\constraint{\numconstraints},\proxyconstraint{1},\dots,\proxyconstraint{\numconstraints}$
  are continuous, then the proxy-Lagrangian game (\eqref{proxy-lagrangian})
  has a mixed Nash equilibrium pair $\left(\parameters,\multipliers\right)$
  where $\parameters$ is a random variable supported on at most
  $\numconstraints+1$ elements of $\Parameters$, and $\multipliers$ is
  non-random.
\end{lem}
\begin{prf}{sparse-proxy-lagrangian}
  Applying \thmref{sparse-equilibrium} directly would result in a support size
  of $\numconstraints+2$, rather than the desired $\numconstraints+1$, since
  $\Multipliers$ is $\left(\numconstraints+1\right)$-dimensional.
  Instead, we define $\tilde{\Multipliers} = \left\{ \tilde{\multipliers} \in
  \R_+^{\numconstraints} : \norm{\tilde{\multipliers}}_1 \le 1 \right\}$ as the
  space containing the last $\numconstraints$ coordinates of $\Multipliers$.
  Then we can rewrite the proxy-Lagrangian functions
  $\tilde{\lagrangian}_{\parameters},\tilde{\lagrangian}_{\multipliers} :
  \Parameters \times \tilde{\Multipliers} \rightarrow \R$ as:
  \begin{align*}
    \tilde{\lagrangian}_{\parameters}\left(\parameters,
    \tilde{\multipliers}\right) =& \left(1 -
    \norm{\tilde{\multipliers}}_1\right) \objective\left(\parameters\right) +
    \sum_{i=1}^{\numconstraints} \tilde{\multipliers}_i
    \proxyconstraint{i}\left(\parameters\right) \\
    \tilde{\lagrangian}_{\multipliers}\left(\parameters,
    \tilde{\multipliers}\right) =& \sum_{i=1}^{\numconstraints}
    \tilde{\multipliers}_i \constraint{i}\left(\parameters\right)
  \end{align*}
  These functions are linear in $\tilde{\multipliers}$, which is a
  $\numconstraints$-dimensional space, so the conditions of
  \thmref{sparse-equilibrium} apply, yielding the claimed result.
\end{prf}

\section{Proofs of Convergence Rates}\label{app:convergence}

\subsection{Non-Stochastic One-Player Convergence Rates}

\begin{thm}{mirror}
  \textbf{(Mirror Descent)}
  Let $f_1,f_2,\ldots : \Parameters \rightarrow \R$ be a sequence of convex
  functions that we wish to minimize on a compact convex set $\Parameters$.
  Suppose that the ``distance generating function'' $\Psi : \Parameters
  \rightarrow \R_+$ is nonnegative and $1$-strongly convex \wrt a norm
  $\norm{\cdot}$ with dual norm $\norm{\cdot}_*$.

  Define the step size $\eta = \sqrt{ \bound{\Psi} / T \bound{\subgrad}^2 }$,
  where $\bound{\Psi} \ge \max_{\parameters \in \Parameters}
  \Psi\left(\parameters\right)$ is a uniform upper bound on $\Psi$, and
  $\bound{\subgrad} \ge \norm{\subgrad f_t \left(\parameters^{(t)}\right)}_*$
  is a uniform upper bound on the norms of the subgradients.
  Suppose that we perform $T$ iterations of the following update, starting from
  $\parameters^{(1)} = \argmin_{\parameters \in \Parameters}
  \Psi\left(\parameters\right)$:
  \begin{align*}
    \tilde{\parameters}^{(t+1)} =& \grad \Psi^*\left( \grad \Psi\left(
    \parameters^{(t)} \right) - \eta \subgrad f_t\left( \parameters^{(t)}
    \right) \right) \\
    \parameters^{(t+1)} =& \argmin_{\parameters \in \Parameters}
    \bregman{\Psi}\left( \parameters \mid \tilde{\parameters}^{(t+1)} \right)
  \end{align*}
  where $\subgrad f_t\left(\parameters\right) \in \partial
  f_t(\parameters^{(t)})$ is a subgradient of $f_t$ at $\parameters$, and
  $\bregman{\Psi}\left(\parameters \mid \parameters'\right) \defeq
  \Psi\left(\parameters\right) - \Psi\left(\parameters'\right) - \inner{\grad
  \Psi\left(\parameters'\right)}{\parameters - \parameters'}$ is the Bregman
  divergence associated with $\Psi$. Then:
  \begin{equation*}
    \frac{1}{T} \sum_{t=1}^T f_t\left( \parameters^{(t)} \right) - \frac{1}{T}
    \sum_{t=1}^T f_t\left( \parameters^* \right) \le 2 \bound{\subgrad} \sqrt{
    \frac{\bound{\Psi}}{T} }
  \end{equation*}
  where $\parameters^* \in \Parameters$ is an arbitrary reference vector.
\end{thm}
\begin{prf}{mirror}
  Mirror descent~\citep{Nemirovski:1983,Beck:2003} dates back to 1983, but this
  particular statement is taken from Lemma 2 of \citet{Srebro:2011}.
\end{prf}

\begin{cor}{sgd}
  \textbf{(Gradient Descent)}
  Let $f_1,f_2,\ldots : \Parameters \rightarrow \R$ be a sequence of convex
  functions that we wish to minimize on a compact convex set $\Parameters$.

  Define the step size $\eta = \bound{\Parameters} / \bound{\subgrad} \sqrt{2
  T}$, where $\bound{\Parameters} \ge \max_{\parameters \in \Parameters}
  \norm{\parameters}_2$, and $\bound{\subgrad} \ge \norm{\subgrad f_t
  \left(\parameters^{(t)}\right)}_2$ is a uniform upper bound on the norms of
  the subgradients.
  Suppose that we perform $T$ iterations of the following update, starting from
  $\parameters^{(1)} = \argmin_{\parameters \in \Parameters}
  \norm{\parameters}_2$:
  \begin{equation*}
    \parameters^{(t+1)} = \Pi_{\Parameters}\left( \parameters^{(t)} - \eta
    \subgrad f_t\left( \parameters^{(t)} \right) \right)
  \end{equation*}
  where $\subgrad f_t\left(\parameters\right) \in \partial
  f_t(\parameters^{(t)})$ is a subgradient of $f_t$ at $\parameters$, and
  $\Pi_{\Parameters}$ projects its argument onto $\Parameters$ \wrt the
  Euclidean norm. Then:
  \begin{equation*}
    \frac{1}{T} \sum_{t=1}^T f_t\left( \parameters^{(t)} \right) - \frac{1}{T}
    \sum_{t=1}^T f_t\left( \parameters^* \right) \le \bound{\Parameters}
    \bound{\subgrad} \sqrt{\frac{2}{T}}
  \end{equation*}
  where $\parameters^* \in \Parameters$ is an arbitrary reference vector.
\end{cor}
\begin{prf}{sgd}
  Follows from taking $\Psi\left(\parameters\right) = \norm{\parameters}_2^2 /
  2$ in \thmref{mirror}.
\end{prf}

\begin{cor}{matrix-multiplicative}
  Let $\Matrixmultipliers \defeq \left\{ \matrixmultipliers \in
  \R^{\matrixmultipliersize \times \matrixmultipliersize} : \forall i \in
  \indices{\matrixmultipliersize} . \matrixmultipliers_{:, i} \in
  \Delta^{\matrixmultipliersize} \right\}$ be the set of all left-stochastic
  $\matrixmultipliersize \times \matrixmultipliersize$ matrices, and let
  $f_1,f_2,\ldots : \Matrixmultipliers \rightarrow \R$ be a sequence of concave
  functions that we wish to maximize.

  Define the step size $\eta = \sqrt{ \matrixmultipliersize \ln
  \matrixmultipliersize / T \bound{\supgrad}^2 }$, where
  $\bound{\supgrad} \ge \norm{\supgrad f_t\left( \matrixmultipliers^{(t)} \right)}_{\infty, 2}$ is a
  uniform upper bound on the norms of the supergradients, and
  $\norm{\cdot}_{\infty, 2} \defeq \sqrt{ \sum_{i=1}^{\matrixmultipliersize}
  \norm{\matrixmultipliers_{:,i}}_{\infty}^2 }$ is the $L_{\infty,2}$ matrix
  norm.
  Suppose that we perform $T$ iterations of the following update
  starting from the matrix $\matrixmultipliers^{(1)}$ with all elements equal
  to $1/\matrixmultipliersize$:
  \begin{align*}
    \tilde{\matrixmultipliers}^{(t+1)} =& \matrixmultipliers^{(t)}
    \elementwiseproduct \elementwiseexp\left( \eta \supgrad f_t\left(
    \matrixmultipliers^{(t)} \right) \right) \\
    \matrixmultipliers_{:,i}^{(t+1)} =&
    \tilde{\matrixmultipliers}_{:,i}^{(t+1)} /
    \norm{\tilde{\matrixmultipliers}_{:,i}^{(t+1)}}_1
  \end{align*}
  where $-\supgrad f_t\left(\matrixmultipliers^{(t)}\right) \in \partial
  \left(-f_t(\matrixmultipliers^{(t)})\right)$, \ie $\supgrad
  f_t\left(\matrixmultipliers^{(t)}\right)$ is a supergradient of $f_t$ at
  $\matrixmultipliers^{(t)}$, and the multiplication and exponentiation in the
  first step are performed element-wise. Then:
  \begin{equation*}
    \frac{1}{T} \sum_{t=1}^T f_t\left( \matrixmultipliers^* \right) -
    \frac{1}{T} \sum_{t=1}^T f_t\left( \matrixmultipliers^{(t)} \right) \le 2
    \bound{\supgrad} \sqrt{ \frac{ \matrixmultipliersize \ln
    \matrixmultipliersize }{T} }
  \end{equation*}
  where $\matrixmultipliers^* \in \Matrixmultipliers$ is an arbitrary reference
  matrix.
\end{cor}
\begin{prf}{matrix-multiplicative}
  Define $\Psi:\Matrixmultipliers \rightarrow \R \defeq \matrixmultipliersize
  \ln \matrixmultipliersize + \sum_{i,j \in \indices{\matrixmultipliersize}}
  \matrixmultipliers_{i,j} \ln \matrixmultipliers_{i,j}$ as
  $\matrixmultipliersize \ln \matrixmultipliersize$ plus the negative Shannon
  entropy, applied to its (matrix) argument element-wise
  ($\matrixmultipliersize \ln \matrixmultipliersize$ is added to make $\Psi$
  nonnegative on $\Matrixmultipliers$). As in the vector setting, the resulting
  mirror descent update will be (element-wise) multiplicative.

  The Bregman divergence satisfies:
  \begin{align}
    \notag \bregman{\Psi}\left(\matrixmultipliers \vert
    \matrixmultipliers'\right) =& \Psi\left(\matrixmultipliers\right) -
    \Psi\left(\matrixmultipliers'\right) - \inner{\nabla
    \Psi\left(\matrixmultipliers'\right)}{\matrixmultipliers -
    \matrixmultipliers'} \\
    \label{eq:cor:matrix-multiplicative:bregman} =&
    \norm{\matrixmultipliers'}_{1,1} - \norm{\matrixmultipliers}_{1,1} +
    \sum_{i=1}^{\matrixmultipliersize} D_{KL}\left( \matrixmultipliers_{:,i}
    \Vert \matrixmultipliers_{:,i}' \right)
  \end{align}
  where $\norm{\matrixmultipliers}_{1,1} = \sum_{i=1}^{\matrixmultipliersize}
  \norm{\matrixmultipliers_{:,i}}_1$ is the $L_{1,1}$ matrix norm.
  This incidentally shows that one projects onto $\Matrixmultipliers$ \wrt
  $\bregman{\Psi}$ by projecting each column \wrt the KL divergence, \ie by
  normalizing the columns.

  By Pinsker's inequality (applied to each column of an $\matrixmultipliers \in
  \Matrixmultipliers$):
  \begin{equation*}
    \norm{\matrixmultipliers - \matrixmultipliers'}_{1,2}^2 \le 2
    \sum_{i=1}^{\matrixmultipliersize} D_{KL}\left( \matrixmultipliers_{:,i}
    \Vert \matrixmultipliers_{:,i}' \right)
  \end{equation*}
  where $\norm{\matrixmultipliers}_{1,2} =
  \sqrt{\sum_{i=1}^{\matrixmultipliersize}
  \norm{\matrixmultipliers_{:,i}}_1^2}$ is the $L_{1,2}$ matrix norm.
  Substituting this into \eqref{cor:matrix-multiplicative:bregman}, and using
  the fact that $\norm{\matrixmultipliers}_{1,1} = \matrixmultipliersize$ for
  all $\matrixmultipliers \in \Matrixmultipliers$, we have that for all
  $\matrixmultipliers,\matrixmultipliers' \in \Matrixmultipliers$:
  \begin{equation*}
    \bregman{\Psi}\left(\matrixmultipliers \vert \matrixmultipliers'\right) \ge
    \frac{1}{2} \norm{\matrixmultipliers - \matrixmultipliers'}_{1,2}^2
  \end{equation*}
  which shows that $\Psi$ is $1$-strongly convex \wrt the $L_{1,2}$ matrix
  norm.
  The dual norm of the $L_{1,2}$ matrix norm is the $L_{\infty,2}$ norm
  \TODO{check this}, which is the last piece needed to apply
  \thmref{mirror}, yielding the claimed result.
\end{prf}

\begin{lem}{internal-regret}
  Let $\Multipliers \defeq \Delta^{\matrixmultipliersize}$ be the
  $\matrixmultipliersize$-dimensional simplex, define $$\Matrixmultipliers
  \defeq \left\{ \matrixmultipliers \in \R^{\matrixmultipliersize \times
  \matrixmultipliersize} : \forall i \in \indices{\matrixmultipliersize} .
  \matrixmultipliers_{:, i} \in \Delta^{\matrixmultipliersize} \right\}$$ as the
  set of all left-stochastic $\matrixmultipliersize \times
  \matrixmultipliersize$ matrices, and take $f_1,f_2,\ldots : \Multipliers
  \rightarrow \R$ to be a sequence of concave functions that we wish to
  maximize.

  Define the step size $\eta = \sqrt{ \matrixmultipliersize \ln
  \matrixmultipliersize / T \bound{\supgrad}^2 }$, where
  $\bound{\supgrad} \ge \norm{\supgrad f_t\left(\multipliers^{(t)}\right) }_{\infty}$ is a
  uniform upper bound on the $\infty$-norms of the supergradients.
  Suppose that we perform $T$ iterations of the following update, starting from
  the matrix $\matrixmultipliers^{(1)}$ with all elements equal to
  $1/\matrixmultipliersize$:
  \begin{align*}
    \multipliers^{(t)} =& \fix \matrixmultipliers^{(t)} \\
    \deltamatrix^{(t)} =& \left( \supgrad f_t\left(\multipliers^{(t)}\right)
    \right) \left( \multipliers^{(t)} \right) ^T \\
    \tilde{\matrixmultipliers}^{(t+1)} =& \matrixmultipliers^{(t)}
    \elementwiseproduct \elementwiseexp\left( \eta \deltamatrix^{(t)} \right)
    \\
    \matrixmultipliers_{:,i}^{(t+1)} =&
    \tilde{\matrixmultipliers}_{:,i}^{(t+1)} /
    \norm{\tilde{\matrixmultipliers}_{:,i}^{(t+1)}}_1
  \end{align*}
  where $\fix \matrixmultipliers$ is a stationary distribution of
  $\matrixmultipliers$ (\ie a $\multipliers \in \Multipliers$ such that
  $\matrixmultipliers \multipliers = \multipliers$---such always exists, since
  $\matrixmultipliers$ is left-stochastic), $-\supgrad
  f_t\left(\multipliers^{(t)}\right) \in \partial
  \left(-f_t(\multipliers^{(t)})\right)$, \ie $\supgrad
  f_t\left(\multipliers^{(t)}\right)$ is a supergradient of $f_t$ at
  $\multipliers^{(t)}$, and the multiplication and exponentiation of the third
  step are performed element-wise. Then:
  \begin{equation*}
    \frac{1}{T} \sum_{t=1}^T f_t\left( \matrixmultipliers^* \multipliers^{(t)}
    \right) - \frac{1}{T} \sum_{t=1}^T f_t\left( \multipliers^{(t)} \right) \le
    2 \bound{\supgrad} \sqrt{ \frac{ \matrixmultipliersize \ln
    \matrixmultipliersize }{T} }
  \end{equation*}
  where $\matrixmultipliers^* \in \Matrixmultipliers$ is an arbitrary
  left-stochastic reference matrix.
\end{lem}
\begin{prf}{internal-regret}
  This algorithm is an instance of that contained in Figure 1 of
  \citet{Gordon:2008}.

  Define $\tilde{f}_t\left(\matrixmultipliers\right) \defeq
  f_t\left(\matrixmultipliers^{(t)} \multipliers^{(t)}\right)$. Observe that
  since $\supgrad f_t\left(\multipliers^{(t)}\right)$ is a supergradient of
  $f_t$ at $\multipliers^{(t)}$, and $\matrixmultipliers^{(t)}
  \multipliers^{(t)} = \multipliers^{(t)}$:
  \begin{align*}
    f_t\left(\tilde{\matrixmultipliers} \multipliers^{(t)}\right) \le &
    f_t\left(\matrixmultipliers^{(t)} \multipliers^{(t)}\right) +
    \inner{\supgrad
    f_t\left(\multipliers^{(t)}\right)}{\tilde{\matrixmultipliers}
    \multipliers^{(t)} - \matrixmultipliers^{(t)} \multipliers^{(t)}} \\
    \le & f_t\left(\matrixmultipliers^{(t)} \multipliers^{(t)}\right) +
    \deltamatrix^{(t)} \cdot \left( \tilde{\matrixmultipliers} -
    \matrixmultipliers^{(t)} \right)
  \end{align*}
  where the matrix product on the last line is performed element-wise.  This
  shows that $\deltamatrix^{(t)}$ is a supergradient of $\tilde{f}_t$ at
  $\matrixmultipliers^{(t)}$, from which we conclude that the final two steps
  of the update are performing the algorithm of \corref{matrix-multiplicative},
  so:
  \begin{equation*}
    \frac{1}{T} \sum_{t=1}^T \tilde{f}_t\left( \matrixmultipliers^* \right) -
    \frac{1}{T} \sum_{t=1}^T \tilde{f}_t\left( \matrixmultipliers^{(t)} \right)
    \le 2 \bound{\supgrad} \sqrt{ \frac{ \matrixmultipliersize \ln
    \matrixmultipliersize }{T} }
  \end{equation*}
  where the $\bound{\supgrad}$ of \corref{matrix-multiplicative} is a uniform
  upper bound on the $L_{\infty,2}$ matrix norms of the $\deltamatrix^{(t)}$s.
  However, by the definition of $\deltamatrix^{(t)}$ and the fact that
  $\multipliers^{(t)} \in \Delta^{\matrixmultipliersize}$, we can instead take
  $\bound{\supgrad}$ to be a uniform upper bound on
  $\norm{\supgrad^{(t)}}_\infty$.
  Substituting the definition of $\tilde{f}_t$ and again using the fact that
  $\matrixmultipliers^{(t)} \multipliers^{(t)} = \multipliers^{(t)}$ then
  yields the claimed result.
\end{prf}

\subsection{Stochastic One-Player Convergence Rates}

\begin{thm}{stochastic-mirror}
  \textbf{(Stochastic Mirror Descent)}
  Let $\Psi$, $\norm{\cdot}$, $\bregman{\Psi}$ and $\bound{\Psi}$ be as in
  \thmref{mirror}, and let $f_1,f_2,\ldots : \Parameters \rightarrow \R$ be a
  sequence of convex functions that we wish to minimize on a compact convex set
  $\Parameters$.

  Define the step size $\eta = \sqrt{ \bound{\Psi} / T
  \bound{\stochasticsubgrad}^2 }$, where $\bound{\stochasticsubgrad} \ge
  \norm{\stochasticsubgrad^{(t)}}_*$ is a uniform upper bound on the norms of the
  stochastic subgradients.
  Suppose that we perform $T$ iterations of the following \emph{stochastic}
  update, starting from $\parameters^{(1)} = \argmin_{\parameters \in
  \Parameters} \Psi\left(\parameters\right)$:
  \begin{align*}
    \tilde{\parameters}^{(t+1)} &= \grad \Psi^* \left( \grad \Psi \left(
    \parameters^{(t)} \right) - \eta \stochasticsubgrad^{(t)} \right) \\
    \parameters^{(t+1)} &= \argmin_{\parameters \in \Parameters}
    \bregman{\Psi}\left(\parameters \vert \tilde{\parameters}^{(t+1)}\right)
  \end{align*}
  where $\expectation\left[ \stochasticsubgrad^{(t)} \mid \parameters^{(t)} \right]
  \in \partial f_t(\parameters^{(t)})$, \ie $\stochasticsubgrad^{(t)}$ is a
  stochastic subgradient of $f_t$ at $\parameters^{(t)}$. Then, with
  probability $1-\delta$ over the draws of the stochastic subgradients:
  \begin{equation*}
    \frac{1}{T} \sum_{t=1}^T f_t\left( \parameters^{(t)} \right) - \frac{1}{T}
    \sum_{t=1}^T f_t\left( \parameters^* \right) \le 2 \bound{\subgrad} \sqrt{ \frac{2
    \bound{\Psi} \left( 1 + 16 \ln\frac{1}{\delta}
    \right)}{T} }
  \end{equation*}
  where $\parameters^* \in \Parameters$ is an arbitrary reference vector.
\end{thm}
\begin{prf}{stochastic-mirror}
  This is nothing more than the usual transformation of a uniform regret
  guarantee into a stochastic one via the Hoeffding-Azuma inequality---we
  include a proof for completeness. \TODO{can we find a citation for this
  instead of proving it?}

  Define the sequence:
  \begin{equation*}
    \tilde{f}_t\left(\parameters\right) = f_t\left(\parameters^{(t)}\right) +
    \inner{\stochasticsubgrad^{(t)}}{\parameters - \parameters^{(t)}}
  \end{equation*}
  Then applying non-stochastic mirror descent to the sequence $\tilde{f}_t$
  will result in exactly the same sequence of iterates $\parameters^{(t)}$ as
  applying stochastic mirror descent (above) to $f_t$. Hence, by
  \thmref{mirror} and the definition of $\tilde{f}_t$ (notice that we can take
  $\bound{\subgrad} = \bound{\stochasticsubgrad}$):
  \begin{align}
    \notag \frac{1}{T} \sum_{t=1}^T \tilde{f}_t\left( \parameters^{(t)} \right)
    - \frac{1}{T} \sum_{t=1}^T \tilde{f}_t\left( \parameters^* \right) \le& 2 \bound{\subgrad}
    \sqrt{ \frac{\bound{\Psi}}{T} } \\
    \notag \frac{1}{T} \sum_{t=1}^T f_t\left( \parameters^{(t)} \right) -
    \frac{1}{T} \sum_{t=1}^T f_t\left( \parameters^* \right) \le& 2 \bound{\subgrad} \sqrt{
    \frac{\bound{\Psi}}{T} } + \frac{1}{T} \sum_{t=1}^T
    \left( \tilde{f}_t\left(\parameters^*\right) -
    f_t\left(\parameters^*\right) \right) \\
    \label{eq:thm:stochastic-mirror:before-azuma} \le & 2 \bound{\subgrad} \sqrt{
    \frac{\bound{\Psi}}{T} } + \frac{1}{T} \sum_{t=1}^T
    \inner{\stochasticsubgrad^{(t)} - \subgrad
    f_t\left(\parameters^{(t)}\right)}{\parameters^* - \parameters^{(t)}}
  \end{align}
  where the last step follows from the convexity of the $f_t$s. Consider the
  second term on the RHS. Observe that, since the $\stochasticsubgrad^{(t)}$s are
  stochastic subgradients, each of the terms in the sum is zero in expectation
  (conditioned on the past), and the partial sums therefore form a martingale.
  Furthermore, by H\"older's inequality:
  \begin{equation*}
    \inner{\stochasticsubgrad^{(t)} - \subgrad
    f_t\left(\parameters^{(t)}\right)}{\parameters^* - \parameters^{(t)}}
    \le \norm{\stochasticsubgrad^{(t)} - \subgrad
    f_t\left(\parameters^{(t)}\right)}_* \norm{\parameters^* -
    \parameters^{(t)}}
    \le 4 \bound{\stochasticsubgrad} \sqrt{2 \bound{\Psi}}
  \end{equation*}
  the last step because $\norm{\parameters^* - \parameters^{(t)}} \le
  \norm{\parameters^* - \parameters^{(1)}} + \norm{\parameters^{(t)} -
  \parameters^{(1)}} \le 2 \sup_{\parameters \in \Parameters} \sqrt{2
  \bregman{\Psi}\left(\parameters \mid \parameters^{(1)}\right)} \le 2 \sqrt{2
  \bound{\Psi}}$, using the fact that $\bregman{\Psi}$ is $1$-strongly convex
  \wrt $\norm{\cdot}$, and the definition of $\parameters^{(1)}$.
  Hence, by the Hoeffding-Azuma inequality:
  \begin{equation*}
    \probability\left\{ \frac{1}{T} \sum_{t=1}^T \inner{\stochasticsubgrad^{(t)} -
    \subgrad f_t\left(\parameters^{(t)}\right)}{\parameters^* -
    \parameters^{(t)}} \ge \epsilon \right\}
    \le \exp\left( -\frac{T \epsilon^2}{64 \bound{\Psi}
    \bound{\stochasticsubgrad}^2} \right)
  \end{equation*}
  equivalently:
  \begin{equation*}
    \probability\left\{ \frac{1}{T} \sum_{t=1}^T \inner{\stochasticsubgrad^{(t)} -
    \subgrad f_t\left(\parameters^{(t)}\right)}{\parameters^* -
    \parameters^{(t)}} \ge 8 \bound{\stochasticsubgrad} \sqrt{\frac{ \bound{\Psi}
    \ln \frac{1}{\delta} }{T}} \right\}
    \le \delta
  \end{equation*}
  substituting this into \eqref{thm:stochastic-mirror:before-azuma}, and
  applying the inequality $\sqrt{a} + \sqrt{b} \le \sqrt{2a + 2b}$, yields the
  claimed result.
\end{prf}

\begin{cor}{stochastic-sgd}
  \textbf{(Stochastic Gradient Descent)}
  Let $f_1,f_2,\ldots : \Parameters \rightarrow \R$ be a sequence of convex
  functions that we wish to minimize on a compact convex set $\Parameters$.

  Define the step size $\eta = \bound{\Parameters} / \bound{\stochasticsubgrad}
  \sqrt{2 T}$, where $\bound{\Parameters} \ge \max_{\parameters \in \Parameters}
  \norm{\parameters}_2$, and $\bound{\stochasticsubgrad} \ge
  \norm{\stochasticsubgrad^{(t)}}_2$ is a uniform upper bound on the norms of the
  stochastic subgradients.
  Suppose that we perform $T$ iterations of the following \emph{stochastic}
  update, starting from $\parameters^{(1)} = \argmin_{\parameters \in
  \Parameters} \norm{\parameters}_2$:
  \begin{equation*}
    \parameters^{(t+1)} = \Pi_{\Parameters}\left( \parameters^{(t)} - \eta
    \stochasticsubgrad^{(t)} \right)
  \end{equation*}
  where $\expectation\left[ \stochasticsubgrad^{(t)} \mid \parameters^{(t)} \right]
  \in \partial f_t(\parameters^{(t)})$, \ie $\stochasticsubgrad^{(t)}$ is a
  stochastic subgradient of $f_t$ at $\parameters^{(t)}$, and
  $\Pi_{\Parameters}$ projects its argument onto $\Parameters$ \wrt the
  Euclidean norm.  Then, with probability $1-\delta$ over the draws of the
  stochastic subgradients:
  \begin{equation*}
    \frac{1}{T} \sum_{t=1}^T f_t\left( \parameters^{(t)} \right) - \frac{1}{T}
    \sum_{t=1}^T f_t\left( \parameters^* \right) \le 2 \bound{\Parameters}
    \bound{\subgrad} \sqrt{ \frac{ 1 + 16 \ln\frac{1}{\delta} }{T} }
  \end{equation*}
  where $\parameters^* \in \Parameters$ is an arbitrary reference vector.
\end{cor}
\begin{prf}{stochastic-sgd}
  Follows from taking $\Psi\left(\parameters\right) = \norm{\parameters}_2^2 /
  2$ in \thmref{stochastic-mirror}.
\end{prf}

\begin{cor}{stochastic-matrix-multiplicative}
  Let $\Matrixmultipliers \defeq \left\{ \matrixmultipliers \in
  \R^{\matrixmultipliersize \times \matrixmultipliersize} : \forall i \in
  \indices{\matrixmultipliersize} . \matrixmultipliers_{:, i} \in
  \Delta^{\matrixmultipliersize} \right\}$ be the set of all left-stochastic
  $\matrixmultipliersize \times \matrixmultipliersize$ matrices, and let
  $f_1,f_2,\ldots : \Matrixmultipliers \rightarrow \R$ be a sequence of concave
  functions that we wish to maximize.

  Define the step size $\eta = \sqrt{ \matrixmultipliersize \ln
  \matrixmultipliersize / T \bound{\stochasticsupgrad}^2 }$, where
  $\bound{\stochasticsupgrad} \ge \norm{\stochasticsupgrad^{(t)}}_{\infty, 2}$ is a
  uniform upper bound on the norms of the stochastic supergradients, and
  $\norm{\cdot}_{\infty, 2} \defeq \sqrt{ \sum_{i=1}^{\matrixmultipliersize}
  \norm{\matrixmultipliers_{:,i}}_{\infty}^2 }$ is the $L_{\infty,2}$ matrix
  norm.
  Suppose that we perform $T$ iterations of the following stochastic update
  starting from the matrix $\matrixmultipliers^{(1)}$ with all elements equal
  to $1/\matrixmultipliersize$:
  \begin{align*}
    \tilde{\matrixmultipliers}^{(t+1)} =& \matrixmultipliers^{(t)}
    \elementwiseproduct \elementwiseexp\left( \eta \stochasticsupgrad^{(t)}
    \right) \\
    \matrixmultipliers_{:,i}^{(t+1)} =&
    \tilde{\matrixmultipliers}_{:,i}^{(t+1)} /
    \norm{\tilde{\matrixmultipliers}_{:,i}^{(t+1)}}_1
  \end{align*}
  where $\expectation\left[ - \stochasticsupgrad^{(t)} \mid
  \matrixmultipliers^{(t)}\right] \in \partial
  \left(-f_t(\matrixmultipliers^{(t)})\right)$, \ie $\stochasticsupgrad^{(t)}$
  is a stochastic supergradient of $f_t$ at $\matrixmultipliers^{(t)}$, and the
  multiplication and exponentiation in the first step are performed
  element-wise. Then with probability $1-\delta$ over the draws of the
  stochastic supergradients:
  \begin{equation*}
    \frac{1}{T} \sum_{t=1}^T f_t\left( \matrixmultipliers^* \right) -
    \frac{1}{T} \sum_{t=1}^T f_t\left( \matrixmultipliers^{(t)} \right) \le 2
    \bound{\stochasticsupgrad} \sqrt{ \frac{ 2 \left( \matrixmultipliersize \ln
    \matrixmultipliersize \right) \left( 1 + 16 \ln\frac{1}{\delta}\right) }{T}
    }
  \end{equation*}
  where $\matrixmultipliers^* \in \Matrixmultipliers$ is an arbitrary reference
  matrix.
\end{cor}
\begin{prf}{stochastic-matrix-multiplicative}
  The same reasoning as was used to prove \corref{matrix-multiplicative} from
  \thmref{mirror} applies here (but starting from \thmref{stochastic-mirror}).
\end{prf}

\begin{lem}{stochastic-internal-regret}
  Let $\Multipliers \defeq \Delta^{\matrixmultipliersize}$ be the
  $\matrixmultipliersize$-dimensional simplex, define $$\Matrixmultipliers
  \defeq \left\{ \matrixmultipliers \in \R^{\matrixmultipliersize \times
  \matrixmultipliersize} : \forall i \in \indices{\matrixmultipliersize} .
  \matrixmultipliers_{:, i} \in \Delta^{\matrixmultipliersize} \right\}$$ as the
  set of all left-stochastic $\matrixmultipliersize \times
  \matrixmultipliersize$ matrices, and take $f_1,f_2,\ldots : \Multipliers
  \rightarrow \R$ to be a sequence of concave functions that we wish to
  maximize.

  Define the step size $\eta = \sqrt{ \matrixmultipliersize \ln
  \matrixmultipliersize / T \bound{\stochasticsupgrad}^2 }$, where
  $\bound{\stochasticsupgrad} \ge \norm{\stochasticsupgrad^{(t)}}_{\infty}$ is a
  uniform upper bound on the $\infty$-norms of the stochastic supergradients.
  Suppose that we perform $T$ iterations of the following update, starting from
  the matrix $\matrixmultipliers^{(1)}$ with all elements equal to
  $1/\matrixmultipliersize$:
  \begin{align*}
    \multipliers^{(t)} =& \fix \matrixmultipliers^{(t)} \\
    \deltamatrix^{(t)} =& \stochasticsupgrad^{(t)} \left(
    \multipliers^{(t)} \right) ^T \\
    \tilde{\matrixmultipliers}^{(t+1)} =& \matrixmultipliers^{(t)}
    \elementwiseproduct \elementwiseexp\left( \eta \deltamatrix^{(t)} \right)
    \\
    \matrixmultipliers_{:,i}^{(t+1)} =&
    \tilde{\matrixmultipliers}_{:,i}^{(t+1)} /
    \norm{\tilde{\matrixmultipliers}_{:,i}^{(t+1)}}_1
  \end{align*}
  where $\fix \matrixmultipliers$ is a stationary distribution of
  $\matrixmultipliers$ (\ie a $\multipliers \in \Multipliers$ such that
  $\matrixmultipliers \multipliers = \multipliers$---such always exists, since
  $\matrixmultipliers$ is left-stochastic), $\expectation\left[ -
  \stochasticsupgrad^{(t)} \mid \multipliers^{(t)} \right] \in \partial
  \left(-f_t(\multipliers^{(t)})\right)$, \ie $\stochasticsupgrad^{(t)}$ is a
  stochastic supergradient of $f_t$ at $\multipliers^{(t)}$, and the
  multiplication and exponentiation of the third step are performed
  element-wise. Then with probability $1-\delta$ over the draws of the
  stochastic supergradients:
  \begin{equation*}
    \frac{1}{T} \sum_{t=1}^T f_t\left( \matrixmultipliers^* \multipliers^{(t)}
    \right) - \frac{1}{T} \sum_{t=1}^T f_t\left( \multipliers^{(t)} \right) \le
    2 \bound{\stochasticsupgrad} \sqrt{ \frac{ 2 \left( \matrixmultipliersize
    \ln \matrixmultipliersize \right) \left( 1 + 16 \ln\frac{1}{\delta}\right)
    }{T} }
  \end{equation*}
  where $\matrixmultipliers^* \in \Matrixmultipliers$ is an arbitrary
  left-stochastic reference matrix.
\end{lem}
\begin{prf}{stochastic-internal-regret}
  The same reasoning as was used to prove \lemref{internal-regret} from
  \corref{matrix-multiplicative} applies here (but starting from
  \corref{stochastic-matrix-multiplicative}).
\end{prf}

\subsection{Two-Player Convergence Rates}\label{app:convergence:two-player}

\begin{lem}{oracle-lagrangian}
  \ifshowproofs
  \textbf{(\algref{oracle-lagrangian})}
  \fi
  Suppose that $\Multipliers$ and $\Radius$ are as in
  \thmref{lagrangian-suboptimality-and-feasibility}, and define
  $\bound{\stochasticgrad} \ge \max_{t \in \indices{T}}
  \norm{\stochasticgrad_{\multipliers}^{(t)}}_2$.
  If we run \algref{oracle-lagrangian} with the step size $\eta_{\multipliers}
  \defeq \Radius / \bound{\stochasticgrad} \sqrt{2T}$, then the result
  satisfies \thmref{lagrangian-suboptimality-and-feasibility}
  for:
  \begin{equation*}
    \epsilon = \approximation + \Radius \bound{\stochasticgrad}
    \sqrt{\frac{2}{T}}
  \end{equation*}
  where $\approximation$ is the error associated with the oracle $\oracle$.
\end{lem}
\begin{prf}{oracle-lagrangian}
  Applying \corref{sgd} to the optimization over $\multipliers$ gives:
  \begin{equation*}
    \frac{1}{T} \sum_{t=1}^T \lagrangian\left( \parameters^{(t)},
    \multipliers^* \right) - \frac{1}{T} \sum_{t=1}^T \lagrangian\left(
    \parameters^{(t)}, \multipliers^{(t)} \right) \le \bound{\Multipliers}
    \bound{\stochasticgrad} \sqrt{\frac{2}{T}}
  \end{equation*}
  By the definition of $\oracle$ (\defref{oracle}):
  \begin{equation*}
    \frac{1}{T} \sum_{t=1}^T \lagrangian\left( \parameters^{(t)},
    \multipliers^* \right) - \inf_{\parameters^* \in \Parameters} \frac{1}{T}
    \sum_{t=1}^T \lagrangian\left( \parameters^*, \multipliers^{(t)} \right)
    \le \approximation + \bound{\Multipliers} \bound{\stochasticgrad}
    \sqrt{\frac{2}{T}}
  \end{equation*}
  Using the linearity of $\lagrangian$ in $\multipliers$, the fact that
  $\bound{\Multipliers} = \Radius$, and the definitions of $\bar{\parameters}$
  and $\bar{\multipliers}$, yields the claimed result.
\end{prf}

\begin{algorithm*}[t]

\begin{pseudocode}
\codename $\mbox{StochasticLagrangian}\left( \Radius \in \R_+, \lagrangian : \Parameters \times \Multipliers \rightarrow \R, T \in \N, \eta_{\parameters}, \eta_{\multipliers} \in \R_+ \right)$: \\
\codeline Initialize $\parameters^{(1)} = 0$, $\multipliers^{(1)} = 0$ \codecomment{Assumes $0 \in \Parameters$} \\
\codeline For $t \in \indices{T}$: \\
\codeline \> Let $\stochasticsubgrad^{(t)}_{\parameters}$ be a stochastic subgradient of $\lagrangian\left(\parameters^{(t)},\multipliers^{(t)}\right)$ \wrt $\parameters$ \\
\codeline \> Let $\stochasticgrad^{(t)}_{\multipliers}$ be a stochastic gradient of $\lagrangian\left(\parameters^{(t)},\multipliers^{(t)}\right)$ \wrt $\multipliers$ \\
\codeline \> Update $\parameters^{(t+1)} = \Pi_{\Parameters}\left( \parameters^{(t)} - \eta_{\parameters} \stochasticsubgrad^{(t)}_{\parameters} \right)$ \codecomment{Projected SGD updates \dots} \\
\codeline \> Update $\multipliers^{(t+1)} = \Pi_{\Multipliers}\left( \multipliers^{(t)} + \eta_{\multipliers} \stochasticgrad^{(t)}_{\multipliers} \right)$ \codecomment{\;\;\;\;\dots} \\
\codeline Return $\parameters^{(1)},\dots,\parameters^{(T)}$ and $\multipliers^{(1)},\dots,\multipliers^{(T)}$
\end{pseudocode}

\caption{
  Optimizes the Lagrangian formulation (\eqref{lagrangian}) in the convex
  setting.
  The parameter $\Radius$ is the radius of the Lagrange multiplier space
  $\Multipliers \defeq \left\{ \multipliers \in \R_+^{\numconstraints} :
  \norm{\multipliers}_1 \le \Radius \right\}$, and the functions
  $\Pi_{\Parameters}$ and $\Pi_{\Multipliers}$ project their arguments onto
  $\Parameters$ and $\Multipliers$ (respectively) \wrt the Euclidean norm.
}

\label{alg:stochastic-lagrangian}

\end{algorithm*}

\begin{lem}{stochastic-lagrangian}
  \ifshowproofs
  \textbf{(\algref{stochastic-lagrangian})}
  \fi
  Suppose that $\Parameters$ is a compact convex set, $\Multipliers$ and
  $\Radius$ are as in \thmref{lagrangian-suboptimality-and-feasibility}, and
  that the objective and constraint functions
  $\objective,\constraint{1},\dots,\constraint{\numconstraints}$ are convex.
  Define the three upper bounds $\bound{\Parameters} \ge \max_{\parameters \in
  \Parameters} \norm{\parameters}_2$, $\bound{\stochasticsubgrad} \ge \max_{t
  \in \indices{T}} \norm{\stochasticsubgrad_{\parameters}^{(t)}}_2$, and
  $\bound{\stochasticgrad} \ge \max_{t \in \indices{T}}
  \norm{\stochasticgrad_{\multipliers}^{(t)}}_2$.

  If we run \algref{stochastic-lagrangian} with the step sizes
  $\eta_{\parameters} \defeq \bound{\Parameters} / \bound{\stochasticsubgrad}
  \sqrt{2T}$ and $\eta_{\multipliers} \defeq \Radius /
  \bound{\stochasticgrad} \sqrt{2T}$, then the result satisfies the
  conditions of \thmref{lagrangian-suboptimality-and-feasibility} for:
  \begin{equation*}
    \epsilon = 2 \left( \bound{\Parameters} \bound{\stochasticsubgrad} +
    \Radius \bound{\stochasticgrad} \right) \sqrt{ \frac{ 1 + 16
    \ln\frac{2}{\delta} }{T} }
  \end{equation*}
  with probability $1-\delta$ over the draws of the stochastic
  (sub)gradients.
\end{lem}
\begin{prf}{stochastic-lagrangian}
  Applying \corref{stochastic-sgd} to the two optimizations (over $\parameters$
  and $\multipliers$) gives that with probability $1-2\delta'$ over the draws
  of the stochastic (sub)gradients:
  \begin{align*}
    \frac{1}{T} \sum_{t=1}^T \lagrangian\left( \parameters^{(t)},
    \multipliers^{(t)} \right) - \frac{1}{T} \sum_{t=1}^T \lagrangian\left(
    \parameters^*, \multipliers^{(t)} \right) \le& 2 \bound{\Parameters}
    \bound{\stochasticsubgrad} \sqrt{ \frac{ 1 + 16 \ln\frac{1}{\delta'} }{T} }
    \\
    \frac{1}{T} \sum_{t=1}^T \lagrangian\left( \parameters^{(t)},
    \multipliers^* \right) - \frac{1}{T} \sum_{t=1}^T \lagrangian\left(
    \parameters^{(t)}, \multipliers^{(t)} \right) \le& 2 \bound{\Multipliers}
    \bound{\stochasticgrad} \sqrt{ \frac{ 1 + 16 \ln\frac{1}{\delta'} }{T} }
  \end{align*}
  Adding these inequalities, taking $\delta=2\delta'$, using the linearity of
  $\lagrangian$ in $\multipliers$, the fact that $\bound{\Multipliers} =
  \Radius$, and the definitions of $\bar{\parameters}$ and
  $\bar{\multipliers}$, yields the claimed result.
\end{prf}

\end{document}